\documentclass[preprint,3p,authoryear]{elsarticle}
\usepackage[utf8]{inputenc}

\usepackage{amssymb}

\usepackage{lineno}
\modulolinenumbers[10]

\usepackage{subfigure}
\usepackage{amsfonts}
\usepackage{natbib}
\usepackage{graphicx}
\usepackage[hidelinks]{hyperref}
\usepackage{color}
\usepackage{cancel}
\usepackage{multirow}
\usepackage{blkarray}
\usepackage{subfigure}
\usepackage{ulem}
\usepackage{amssymb}
\usepackage{booktabs}
\usepackage{tcolorbox}

\newtheorem{theorem}{Theorem}[section]
\newtheorem{remark}{Remark}[section]
\newtheorem{lemma}{Lemma}[section]

\journal{Neurocomputing}

\DeclareTextFontCommand{\emph}{} 

\usepackage[acronym]{glossaries} 
\makeglossaries

\newacronym{acc}{Acc}{Accuracy}
\newacronym{ai}{AI}{Artificial Intelligence}
\newacronym{ampca}{AMPCA} {Arithmetic Mean Probability of Correct Assignment}
\newacronym{boost}{BOOST}{Boosting Trees}
\newacronym{cel}{CEL}{Cross-Entropy Loss}
\newacronym{cdf}{CDF}{Cumulative Distribution Function}
\newacronym{dcm}{DCM}{Discrete Choice Models}
\newacronym{dmlp}{DMLP}{Deep Multilayer Perceptron}
\newacronym{dca}{DCA}{Discrete Classification Accuracy}
\newacronym{dnn}{DNN}{Deep Neural Network}
\newacronym[plural=GPs,firstplural=Gaussian Processes]{gp}{GP}{Gaussian Process}
\newacronym{gbdt}{GBDT}{Gradient Boosting Decision Trees}
\newacronym{gd}{GD}{Gradient Descent}
\newacronym{glm}{GLM}{Generalised Linear Model}
\newacronym{gmpca}{GMPCA} {Geometric Mean Probability of Correct Assignment}
\newacronym{hpo}{HPO}{hyperparameter optimisation} 
\newacronym{klr}{KLR}{Kernel Logistic Regression}
\newacronym{krr}{KRR}{Kernel Ridge Regression}
\newacronym{lr}{LR}{Logistic Regression}
\newacronym{map}{MAP}{maximum a posteriori probability}
\newacronym{ml}{ML}{Machine Learning}
\newacronym{mle}{MLE}{Maximum Likelihood Estimation}
\newacronym{mlp}{MLP}{Multilayer Perceptron}
\newacronym{mnl}{MNL}{Multinomial Logit}
\newacronym{mnp}{MNP}{Multinomial Probit}
\newacronym{mxl}{MXL}{Mixed Logit}
\newacronym{nl}{NL}{Nested Logit Model}
\newacronym{nn}{NN}{Neural Network}
\newacronym{pmle}{PMLE}{Penalised Maximum Likelihood Estimation}
\newacronym{rf}{RF}{Random Forests}
\newacronym{rbf}{RBF}{Radial Basis Function}
\newacronym{rkhs}{RKHS}{Reproducing Kernel Hilbert Spaces}
\newacronym{rnll}{RNLL}{Regularised Negative Log-Likelihood}
\newacronym{rum}{RUM}{Random Utility Model}
\newacronym{sgd}{SGD}{Stochastic Gradient Descent}
\newacronym{smo}{SMO}{Sequential Minimal Optimisation}
\newacronym{svd}{SVD}{Singular Value Decomposition}
\newacronym{svm}{SVM}{Support Vector Machine}
\newacronym{vot}{VOT}{Value of Time}
\newacronym{wtp}{WTP}{Willingness to Pay}
\newacronym{xgboost}{XGBoost} {Extreme Gradient Boosting}

\usepackage{calc}
\usepackage{booktabs}
\newlength\maxlength
\newlength\thislength
\newglossarystyle{myacronymstyle}
{%
  {
    \setlength{\maxlength}{0pt}%
    \forglsentries[\currentglossary]{\thislabel}%
    {%
       \settowidth{\thislength}{\glsentryshort{\thislabel}}%
       \ifdim\thislength>\maxlength
         \setlength{\maxlength}{\thislength}%
       \fi
    }%
    \settowidth{\thislength}{\hspace{1.5em} \hspace{1em}}%
    \setlength{\glsdescwidth}{\linewidth-\maxlength-\thislength-8\tabcolsep}%
    \begin{tabular}{l@{\hspace{1.5em} \hspace{1em}}p{\glsdescwidth}}
    \toprule
    \multicolumn{1}{l}{\textbf{Acronym}} &
    \multicolumn{1}{@{}l}{\textbf{Meaning}}\\%
    \midrule
  }%
  {
     \bottomrule
     \end{tabular}%
  }%
  \renewcommand*{\glsgroupheading}[1]{}%
}

\newcommand{\x}{\mathbf{x}}

\newcommand{\R}{\mathbb{R}}

\usepackage[linesnumbered,ruled,vlined]{algorithm2e}
\usepackage{setspace}

\SetCommentSty{mycommfont}
\SetKwInput{KwInput}{Input}                
\SetKwInput{KwOutput}{Output}

\begin{document}

\begin{frontmatter}



\title{Scalable Kernel Logistic Regression with Nyström Approximation: Theoretical Analysis and Application to Discrete Choice Modelling}



\author[inst1]{José Ángel Martín-Baos\corref{cor1}}
\ead{JoseAngel.Martin@uclm.es}
\cortext[cor1]{Corresponding author}
\author[inst1]{Ricardo García-Ródenas}
\author[inst2]{Luis Rodriguez-Benitez}
\author[inst3]{Michel Bierlaire}

\affiliation[inst1]{organization={Department of Mathematics, Escuela Superior de Inform\unexpanded{á}tica},
      addressline={University of Castilla-La Mancha}, 
      city={Ciudad Real},
      postcode={13071}, 
      country={Spain}}
\affiliation[inst2]{organization={Department of Information and System Technologies, Escuela Superior de Inform\unexpanded{á}tica},
      addressline={University of Castilla-La Mancha}, 
      city={Ciudad Real},
      postcode={13071}, 
      country={Spain}}
\affiliation[inst3]{organization={Transport and Mobility Laboratory},
      addressline={Ecole Polytechnique Fédérale de Lausanne}, 
      city={Lausanne},
      postcode={1015}, 
      country={Switzerland}}

\begin{abstract}

The application of kernel-based \gls{ml} techniques to discrete choice modelling using large datasets often faces challenges due to memory requirements and the considerable number of parameters involved in these models. This complexity hampers the efficient training of large-scale models. This paper addresses these problems of scalability by introducing the Nyström approximation for \gls{klr} on large datasets. The study begins by presenting a theoretical analysis in which: i) the set of \gls{klr} solutions is characterised, ii) an upper bound to the solution of \gls{klr} with Nyström approximation is provided, and finally iii) a specialisation of the optimisation algorithms to Nyström \gls{klr} is described. After this, the Nyström \gls{klr} is computationally validated. Four landmark selection methods are tested, including basic uniform sampling, a $k$-means sampling strategy, and two non-uniform methods grounded in leverage scores. The performance of these strategies is evaluated using large-scale transport mode choice datasets and is compared with traditional methods such as \gls{mnl} and contemporary \gls{ml} techniques. The study also assesses the efficiency of various optimisation techniques for the proposed Nyström \gls{klr} model. The performance of gradient descent, Momentum, Adam, and L-BFGS-B optimisation methods is examined on these datasets. Among these strategies, the $k$-means Nyström \gls{klr} approach emerges as a successful solution for applying \gls{klr} to large datasets, particularly when combined with the L-BFGS-B and Adam optimisation methods. The results highlight the ability of this strategy to handle datasets exceeding $200,000$ observations while maintaining robust performance.


\end{abstract}




\end{frontmatter}

\glsresetall 


\printglossary[type=\acronymtype,style=myacronymstyle]

\section{Introduction}
\label{sect:introduction}

\gls{dcm} have long been central to analysing individual decision making, with wide-ranging applications in economics, marketing, and transportation \citep{McFa73,BLL85,Tra09}. \gls{dcm} offer simple explanations of decision processes and statistical tools to test the underlying choice mechanisms. \glspl{rum} are commonly employed in this context, relying on utility functions defined by the modellers as a function of the attributes of the available alternatives and the characteristics of the individuals making the choices.

In recent years, the incorporation of \gls{ml} techniques has significantly impacted this field. A growing body of literature demonstrates that \gls{ml} models generally outperform traditional \gls{dcm} in terms of predictive capacity \citep{ Omr15,SMM16,HaH17,ZYY19,LBD19,ZYY20,WMH21,MGR21,MLR23}. As a result, substantial research efforts have been devoted to exploring and exploiting the potential of \gls{ml} in this direction. Nonetheless, it is essential to acknowledge that certain \gls{ml} methods, such as ensemble approaches, might exhibit poor extrapolation properties, and encounter numerical challenges in calculating derivatives of the probability function of the models. These limitations can ultimately impact the performance of the econometric indicators extracted from them, including the calculation of the \gls{wtp} or market shares, as identified by \citet{MLR23}. Therefore, while \gls{ml} holds great promise, it is crucial to identify which \gls{ml} methods do not exhibit these limitations, in order to guarantee their successful application in \gls{dcm}.
  
Among \gls{ml} techniques, supervised learning with \gls{rkhs} has gained popularity. Specifically, previous studies have successfully applied \gls{svm} \citep{CoV95}, \gls{gp} \citep{RaW05}, and \gls{klr} \citep{ZhH05} to discrete choice problems. However, one limitation of \gls{svm} is its inability to directly estimate probabilities, making it unsuitable for estimating econometric indicators. On the other hand, \gls{klr} has demonstrated its effectiveness in the realm of transport choice modelling and offers a regularisation solution similar to Gaussian processes \citep{RaW05}.

\gls{klr} has two crucial advantages in the area of model output derivation. Firstly, it operates with latent functions which take on the role of utility functions in traditional \gls{dcm}, but with the notable advantage of not requiring a pre-specification of the utility function. This allows for the automatic capture of non-linear effects, leading to greater flexibility in modelling \citep{MGL20}. Secondly, \gls{klr} excels at properly estimating choice probabilities and their derivatives with respect to the attributes, providing essential information for behavioural analysis of the various models \citep{MGR21}. These advantages collectively enhance the suitability and performance of \gls{klr} in transportation mode choice analysis.

The estimation of parameters in kernel-based \gls{ml} models entails significantly higher computational complexity compared to \gls{dcm} models. The \gls{mnl}, which is the fundamental model in \gls{dcm}, is mathematically equivalent to \gls{klr}, as both require estimating a linear parameter function. The difference is that while \gls{mnl} models typically involve only tens of parameters, \gls{klr} uses $N \cdot I$ parameters, where $N$ represents the number of data samples in the training set, and $I$ denotes the number of available alternatives. In the past, the sample size $N$ was relatively small, mainly sourced from household surveys, enabling rapid model estimation. However, with the advent of different kinds of electronic devices capturing massive amounts of data from passengers and transport services, the sample size $N$ has grown considerably, necessitating methodologies suitable for handling large datasets.

The computational cost of estimating large-scale \gls{klr} models is further influenced by the dense and high-dimensional nature of the so-called {\sl kernel} or {\sl Gram matrix}. To enhance the efficiency of the kernel-based techniques, a widely adopted approach involves generating a low-rank matrix approximation to the kernel matrix. The Nyström method \citep{SZZ15} stands out as a leading technique for achieving this approximation.

This study introduces the Nyström approximation into the estimation procedure of \gls{klr} to address the computational burden of the method. This study examines multiple combinations of optimisation techniques used in training large-scale \gls{ml} models, along with Nyström approximations. The numerical analysis shows that by applying this technique it is possible to effectively estimate \gls{klr} in large-scale transport mode choice problems.

\subsection{Contributions and scope}\label{sect:contributions}
The main contributions of this research paper can be summarised as follows:
\begin{enumerate}

    \item From a theoretical point of view, this study characterises the solution set of the \gls{klr} as a hyperplane. This analysis allows a suitable formulation of the training problem, which has been called the {\sl restricted training problem}, and differs from the common formulation used in the literature, overcoming the sub-optimality present in such formulations. Subsequently, the Nyström approximation is integrated into the \gls{klr}, and an error bound is established in the application of the Nyström \gls{klr}.  

    \item It looks at the use of four Nyström methods to expedite the training process of \gls{klr}, namely a basic uniform method, a $k$-means algorithm-based method, and two non-uniform methods that rely on leverage scores. Their effectiveness is demonstrated in addressing transport mode choice problems within large datasets that would otherwise be impractical to handle without the application of these techniques
    
    \item Experimentation centres on the critical task of transport mode choice, using the LPMC \citep{HEJ18} and NTS \citep{HaH17} datasets, because of their significant size. A comprehensive comparison of the Nyström-based \gls{klr} approaches is carried out, followed by an extensive analysis of their accuracy and computational complexity in comparison to benchmark techniques for transport mode choice, including \gls{mnl}, \gls{rf}, \gls{svm}, \gls{xgboost}, and \gls{nn}.

    \item The computation of the training loss function and its gradient in the Nyström method-\gls{klr} has been specialised, ensuring tailored optimisation. Additionally, the efficiency of the gradient descent, Momentum, Adam, and L-BFGS-B optimisation methods is assessed for the proposed Nyström \gls{klr} model, enabling the most effective approach within this specific problem domain to be identified.
\end{enumerate}

\subsection{Structure}
The remaining sections of this paper are structured as follows: Section~\ref{sect:relatedwork} provides an overview of the previous techniques relevant to our proposed approach. Section~\ref{sect:klr} presents the formulation of the \gls{klr} model and analyses the structure of the solution set. Next, Section~\ref{sect:methodology} details the strategy for approximating the kernel matrix of \gls{klr} using  Nyström method, and discusses the error bounds involved in this approximation. Section \ref{sect:algortihms} formulates the optimisation algorithms applicable to Nyströn-\gls{klr} models. Section~\ref{sect:experiments} then outlines the experiments conducted, and presents a comprehensive analysis of the numerical results obtained from this research. Finally, Section~\ref{sect:conclusion} sets out the key findings and draws conclusions from the study, while also suggesting potential avenues for future research.

\section{Related work}
\label{sect:relatedwork}

The two major methodologies commonly applied when studying individual travel behaviour are traditional \gls{rum} and \gls{ml} methods. The comparison between these approaches has been examined in the literature from two angles: (i) assessment of the performance of the models, and (ii) behavioural interpretation in the context of discrete choice modelling. 

Focusing on the \gls{ml} techniques, many studies indicate that \gls{xgboost} and \gls{rf} show good performance \citep{WMH21} in this domain, as well as in other classification domains such as stock price direction \citep{BVH15}. This performance is supported by comparisons among classifiers, such as those presented in \citet{FCB14} with 17 families of classifiers in 171 datasets. However, ensemble methods present problems when obtaining derivable probability functions, which makes kernel-based methods and neural networks more promising from the point of view of obtaining model outputs \citep{SDA22,MLR23}. 

The canonical classifier of kernel-based methods is \gls{svm}, but one of its major limitations is that it cannot return class probabilities. Thus, when probability outputs are required, alternatives like \gls{klr} are preferred. This method has exhibited potential in other domains. For instance, \citet{OuA18} applied it to action recognition, \citet{LLT16} proved its efficacy in web image annotation, \citet{LHS14} used this approach to analyse party preferences based on the characteristics of individuals, and \citet{BDS19} used it to study factors linked to institutional births in India. Nonetheless, kernel-based methods have found limited application in the field of transport. Notably, in studies by \citet{EGR15} and \citet{EGL18}, a nested \gls{klr} was used to model rail service choices. Additionally, \citet{MSH16} demonstrated the application of Gaussian processes to housing choice data, while \citet{MGR21} employed \gls{klr} to obtain unbiased estimates of econometric indicators like the \gls{wtp}.

One of the drawbacks to the application of kernel-based methods, such as \gls{klr}, is that they typically exhibit a computational complexity $O(N^3)$. The computational cost of these models primarily stems from the dense, large-dimensional kernel matrix used in these techniques. This cost limits their use when exceeding a certain amount of data, which is usually established as around $10,000$ data samples. Therefore, it is necessary to reduce the computational cost of these techniques.

To increase the efficiency of the method, common approaches involve employing approximation techniques that yield a low-rank representation of the kernel matrix. These methods include \emph{subset of regressors} \citep{ZDJ11}, the \emph{Nyström method} \citep{LTL19}, \emph{subset of data} \citep{ZhH05}, and \emph{incomplete Choleski factorisation} \citep{CaT05,LGG15}, among others. 

Nyström's methods have been successfully applied to large-scale machine learning techniques such as  Kernel Regression applications  \citep{AlM15,ZSH22},  numerical solutions of an integral equations system \citep{ARS05}, among other domains. 
This large number of applications motivates a research effort {\sl per se} to obtain error bounds for low-rank approximation of kernel matrices $\mathbf{K}$. Fundamental theoretical results on the error bounds of the Nyström approximation are presented in \citet{ZTK08,KMT12,SZZ15,GiM16}. These studies highlight that the performance of various Nyström approximations depends on intricate factors present in the data structure, such as spectral decay, leverage score structure, eigenvalue gaps in significant parts of the spectrum, etc. 

The analysis of these bounds allows for the development of methods to enhance the accuracy of the Nyström approximation, and therefore several improved variants have been developed, including ensemble Nyström \citep{SZZ15}, modified Nyström \citep{DeL16}, and SS-Nyström \citep{WZQ14}. Another approach involves applying sampling analysis methods in order to identify a subset of landmark points that capture the majority of the dataset's information. A detailed exploration of these sampling methods can be found in \citet{SZZ15}.
 
In the application of Nyström methods to machine learning, it is desired to know how the error of the approximation affects the performance of the algorithm itself. These studies are usually performed in two stages, the first involves bounding the error of the learnt parameters as a function of the Nyström approximation error, and the second involves bounding the learning performance based on the errors of the resulting parameters. In \citet{CMT10} the error bounds for \gls{svm}, \acrfull{krr},  and graph Laplacian-based regularisation algorithms were studied. \citet{ZSH22} examined a non-uniform Nyström approximation for Sparse Kernel Regression (SKR), and introduced two non-uniform Nyström methods with theoretical guarantees for SKR. \citet{WGM19} analysed the application of the so-called rank-restricted Nyström approximation to kernel $k-$means clustering and showed that it satisfied a $1+\varepsilon$ approximation ratio in terms of the kernel $k-$means cost function, related to the guarantee provided by the same algorithm without using the Nyström method.

To the best of our knowledge, no bounds have been established for the Nyström \gls{klr} and this paper analyses the problem. In this study,  theoretical bounds have been established for the estimated parameters and it has been numerically observed that there is a correlation between these bounds and the usual metrics used in the validation of the \gls{klr}.

Regarding Nyström's application to \gls{klr}, \citet{LTL19} proposed combining the standard Nyström method with a fast iterative algorithm, similar to sequential minimal optimisation, to solve binary classification problems. However, the studies conducted by these authors addressed simpler problems compared to those encountered in this research. Specifically, they exclusively dealt with binary classification problems, and the datasets they considered contained fewer than $4,000$ samples.

The standard approach for training the parameters of \gls{klr} is by \gls{pmle}. To solve this model, it is necessary to explore unconstrained optimisation methods. Some of the most commonly used methods in  \gls{ml} include stochastic gradient, nonlinear conjugate gradient, quasi-Newton (in particular, limited memory BFGS, or L-BFGS), Newton, and truncated Newton methods \citep{BCN18}. The canonical algorithm applied in \gls{mnl} is the Newton method \citep{HZM16}, which is implemented as the default method in software packages like \textit{Biogeme} (a widely used Python package for estimating \gls{rum} models). However, Newton's method has limited scalability when dealing with a large number of instances. To overcome this limitation, quasi-Newton methods have been developed, with the limited-memory BFGS algorithm being particularly efficient and effective \citep{Mal02,AnG07}. One technique that has proved successful for \gls{mnl} on large datasets is the truncated Newton method \citep{KKB07}. In particular, \citet{LWK08} proposed the TRON algorithm, which is a truncated Newton method that uses a $L_2$ regularised \gls{mnl} to solve problems with large numbers of both instances and features. \citet{ZaW17} showed that for sparse binary classification problems, TRON is faster than L-BFGS for $L_2$ regularised \gls{mnl}. Building on these results, \citet{MTA11,MaT11} apply a truncated Newton method to \gls{klr}, and show its effectiveness for small to medium-sized datasets. An alternative approach is represented by the work of \citet{KKB07}, who introduced a specialised interior-point method to solve an $L_1$ regularised \gls{mnl}, which exhibits high efficiency across problem sizes.  The efficiency of this method is only marginally slower when compared to conjugate gradients and L-BFGS algorithms applied to an $L_2$ regularised \gls{mnl}.

After this review of the relevant literature, the next section introduces the theoretical basis of multinomial kernel logistic regression.

\section{Multinomial Kernel Logistic Regression}
\label{sect:klr}

In contrast to traditional \gls{rum}, the primary objective of \gls{klr} is not to explain the decision-making process of individual users. Instead, it focuses on developing procedures with minimal classification error. \gls{klr} is often considered as a variant of \gls{svm} \citep{CoV95}. Nonetheless, \gls{klr} not only predicts the classification of an object, such as an individual's choice but also provides probability estimates for membership in each category.

The foundation of \gls{klr} lies in constructing several \emph{latent functions}, which are analogous to the systematic utility functions found in \glspl{rum}. However, unlike \gls{rum}, \gls{klr} treats these latent functions as black boxes, where the modeller does not need to explicitly state the relationship between the feature vector and the utilities. Rather, \gls{klr} operates within the framework of \gls{rkhs}, seeking these latent functions, which are denoted by $f_i: \mathcal X \mapsto \mathbb{R}$ for each alternative $i=1,\ldots,I$, within this function space. The \gls{rkhs} is univocally generated by the  positive-definite real-valued \emph{kernel} on $k:{\cal X} \times {\cal X} \mapsto \mathbb{R}$, and its associated \gls{rkhs} space is denoted by ${\mathcal H}_k$. The family of functions $\{k(\mathbf x, \mathbf  x')\}_{\mathbf{x}' \in \mathcal X}$ constitutes a basis for this vector space. Thus, any element from ${\mathcal H}_k$ can be represented as a linear combination of basis elements, in particular  for  $f_i \in {\mathcal H}_k$.

The utilisation of \gls{rkhs} ensures a flexible and powerful approach to modelling the complex relationships between the feature vectors and latent functions, making \gls{klr} a highly effective classification tool in various applications. Given these latent functions, the classification criterion used for a given sample $\mathbf{x}$ is:
\begin{equation}
\widetilde f(\mathbf{x})= \underset{i\in \{1,\ldots,I\}}{\hbox{\rm arg maximise }} f_i(\mathbf{x}).    
\end{equation}
\noindent where {\rm arg maximise} denotes the alternative $i$ at which the function output values $f_i(\mathbf{x})$ is maximised.
The latent functions $f_i(\mathbf{x})$ serve as an index of how well the sample $ \mathbf{x}$ fits alternative $i$, and the classifier $\widetilde{f}(\mathbf{x})$ assigns the alternative for which this index is highest.

Let $\{\mathbf{x}_{n}, y_n\}_{n=1}^N$ be $N$ pairs of points in ${\cal X} \times {\cal Y}$, where ${\cal X}\subseteq {\mathbb R}^M$ is the input space and $\cal Y$ is the response space, which, in our context, serves as a label space. The kernel-based learning problem can be stated as the following optimisation problem:
\begin{equation}
    \label{eq:04_ajuste_klr_a}
    \underset{\widetilde f}{\hbox{Minimise }} \frac{1}{N} \sum_{n=1}^N  \ell \left (y_{n}, \widetilde f(\mathbf{x}_{n}) \right ) + R ( \widetilde f ),
\end{equation}
\noindent where  $\ell: {\cal Y} \times {\cal Y} \mapsto \mathbb{R}$ is a loss function that  quantifies the discrepancies between the predicted, $\widetilde f(\mathbf{x}_n)$, and the observed labels, $y_n$, and $R (\widetilde f)$ is a regularisation term to enhance the generalisation capacity of the model.

As stated, this optimisation problem can be infinite-dimensional in nature, since it takes place over the  \gls{rkhs}. The representer theorem, \citet{KiG71}, states that a minimiser of the regularised empirical risk  problem, Equation~(\ref{eq:04_ajuste_klr_a}), admits a representation of the form: 
\begin{equation}
    \label{eq:04_representacion}
    f_i( \mathbf x) = \sum_{n=1}^N  \alpha_{ni}k( \mathbf{x}_{n}, \mathbf x)
\end{equation} 
\noindent if and only if there exists a non-decreasing function $h:[0,\infty) \mapsto \mathbb{R}$ for which $R(\widetilde f)=h \left (\| \widetilde f\| \right )$. It can be seen how $f_i( \mathbf x)$ depends on the parameter vector $\boldsymbol{\alpha}_{\cdot i}$ and in turn, the classifier $\widetilde f( \mathbf x)$ depends on all latent functions. It will therefore also depend on this set of parameter vectors which will be gathered in the form of a matrix by writing $\boldsymbol{\alpha}_{\cdot \cdot}=(\alpha_{ni} ) \in \R^{N \times I}$. As a straightforward consequence, the problem represented in Equation~(\ref{eq:04_ajuste_klr_a}) is reduced to an finite-dimensional optimisation problem where the goal is to find a matrix $\boldsymbol{\alpha}_{\cdot \cdot}$ that solves:
\begin{equation}
    \label{eq:04_ajuste_klr_b}
    \underset{\boldsymbol{\alpha_{\cdot \cdot }} \in \R^{N \times  I} }{\hbox{Minimise }} \frac{1}{N} \sum_{n=1}^N  \ell(y_{n}, \widetilde f(\mathbf{x}_{n})) + R(\widetilde f),
\end{equation}

One of the most commonly used regularisation terms, $R(\widetilde f)$, is the \emph{Ridge} penalisation, often referred to as \emph{Tikhonov regularisation}. This regularisation term includes a convex and differentiable quadratic component in the objective function, $\dfrac{\lambda}{2} \sum_{i=1}^I \boldsymbol{\alpha}_{\cdot i}^\top \mathbf{K} \boldsymbol{\alpha}_{\cdot i}$, where $\mathbf{K}$ is the  so-called \emph{kernel} or \emph{Gram matrix}, defined as follows: $\mathbf{K}_{n\,n'}=k(\mathbf{x}_{n},\mathbf{x}_{n'}) \hbox{ for } n,n'=1,\ldots,N;$ and $\boldsymbol{\alpha}_{\cdot i}$ is the $i-$th  column of the matrix $\boldsymbol{\alpha}_{\cdot \cdot}$, i.e. $\boldsymbol{\alpha}_{\cdot \cdot}= (\boldsymbol{\alpha}_{\cdot 1}, \ldots,\boldsymbol{\alpha}_{\cdot I} ) \in \R^{N \times I}$. 

Since $k(\cdot, \cdot)$ is a kernel function, the Gram matrix $\mathbf{K}$ is guaranteed to be symmetric and positive semi-definite, ensuring the convexity of $R(\widetilde f)$. This regularisation term tends to smooth out the model's parameters and reduce their variance, thereby helping to prevent overfitting and enhancing the generalisation capacity of the model \citep{HTF01}. This, in turn, allows the model to perform better on unseen data. 

Tikhonov regularisation is defined as the sum of the norm of the latent functions in \gls{rkhs}: 
\begin{equation}
  R( \widetilde f)= \dfrac{\lambda}{2} \sum_{i=1}^I  \|f_i(\mathbf{x})\|^2_{{\cal H}_K}=\dfrac{\lambda}{2} \sum_{i=1}^I\sum_{n=1}^{N}\sum_{n'=1}^{N}k(\x_{n},\x_{n'}) \alpha_{ni} \alpha_{n'i}= \dfrac{\lambda}{2} \sum_{i=1}^I\boldsymbol{\alpha}_{\cdot i}^\top \mathbf{K} \boldsymbol{\alpha}_{\cdot i},
\end{equation}
\noindent where the regularisation parameter $\lambda$ controls the trade-off between the model's goodness of fit and its complexity. In order to facilitate the later computation of the gradient, the regularisation parameter is divided by $2$ in this expression.

Given the fundamental role of these matrices in subsequent computations, it is crucial to establish the following notation to efficiently represent and manipulate the Gram matrices. Specifically, $\mathbf{k}_n$ denotes the $n$-th column of $\mathbf{K}$, which represents the transformed attribute vector for each decision maker $n=1,\ldots,N$, i.e. the feature vector.

In the case of \gls{klr}, the negative value of the log-likelihood (see \citet{HTF01, ZhH05, OuA18}), is typically used as the loss function, which is commonly referred to as minimising the \gls{rnll} or, equivalently, maximising the \gls{pmle}.  In \gls{klr}, the softmax function is used to provide estimates of the posterior probability of the alternatives. Hence,
\begin{equation}
    \label{eq:04_proba_KLR}
    p_{ni} \left ( \boldsymbol{\alpha}_{\cdot \cdot} \right ) 
    = \frac{ \exp \left (f_{i}(\mathbf{x}_{n}) \right ) }{\sum_{j=1}^{I}\exp \left (f_{j}(\mathbf{x}_{n})\right  )}=
\frac{ \exp \left (  \mathbf{k}_n^\top  \boldsymbol{\alpha}_{\cdot i} \right) }{\sum_{j=1}^{I}\exp\left ( \mathbf{k}_n^\top \boldsymbol{\alpha}_{\cdot j} \right)}.   
\end{equation}

Henceforth, the matrix $\mathbf{y}_{\cdot \cdot}$ is defined to indicate whether individual $n$ has selected alternative $i$ using the one-hot encoding of $y_n$. Specifically, $y_{ni}=1$ if $y_n=i$ and $y_{ni}=0$ otherwise. Consequently, the \gls{rnll} problem used for training the \gls{klr} can be formulated as follows:
\begin{equation}
  \label{eq:04_estimation_KLR_RNLL}
    \underset{\boldsymbol{\alpha}_{\cdot \cdot }}{\hbox{Minimise }} {\cal L}(\boldsymbol{\alpha}_{\cdot \cdot})  =  \underset{\boldsymbol{\alpha}_{\cdot \cdot }}{\hbox{Minimise }} \left( -\frac{1}{N} \sum _{n=1}^N \sum_{i=1}^I y_{ni} \log p_{ni}(\boldsymbol{\alpha}_{\cdot \cdot })+ \frac{\lambda}{2} \sum_{i=1}^I \boldsymbol{\alpha}_{\cdot i} ^\top \mathbf{K}\boldsymbol{\alpha}_{\cdot i} \right).
\end{equation}

\subsection{Preliminaries}

This section introduces the concepts and mathematical notation for analysing both the solution of the \gls{klr} model and the derivation of error bounds for the Nyström \gls{klr} method.

{\bf Vectors and matrices.}
Given a  matrix $\mathbf{A}_{\cdot \cdot } \in \mathbb{R}^{N \times I}$, 
$\mathbf{A}_{n \cdot }$ represents the $n-th$ row of $\mathbf{A}_{\cdot \cdot }$ and $\mathbf{A}_{\cdot i}$ the $i-th$ column.  In some non-confusing contexts it is possible to simplify the notation by eliminating the subscript  $\cdot \,\cdot$. For the special case of the kernel matrix the notation is simplified by using $\mathbf{K}$ for the matrix and $\mathbf{k}_i$ for the $i-th$ column. $\mathbb{I}_R$ is the $R \times R$ identity matrix, $\mathbf{0}$ is a vector or matrix with only zeros, sized accordingly, and $\mathbf{1}$ is a similarly sized vector containing exclusively ones. 

{\bf \gls{svd}}. Let $\mathbf{A} \in \mathbb{R}^{N \times R }$ and let $R$ be the rank of the matrix $\mathbf{A}$, i. e. $R=\hbox{rank}(\mathbf{A})$. A {\bf \gls{svd}} of $\mathbf{A}$ can be written as 
\begin{equation}
\mathbf{A}=\mathbf{U} \boldsymbol{\Sigma} \mathbf{V}^\top = \sum_{j=1}^R \sigma_j \mathbf{U}_{\cdot j} \mathbf{V}_{j\cdot }
\end{equation}
\noindent where $\mathbf{U}, \boldsymbol \Sigma, \mathbf{V}$  are an $N\times R$ column-orthogonal matrix,  an $R \times R$ diagonal matrix with nonnegative entries and an $R \times R$ column-orthogonal matrix, respectively. By convention, take $\sigma_1 \ge \sigma_2\ge \cdots \sigma_R>0$.

In the case of the kernel matrix $\mathbf{K}\in \mathbb{R}^{N \times N}$, due to the fact that $\mathbf{K}$ is a symmetric and semidefinite matrix, $\mathbf{U}=\mathbf{V}$ holds, and it can be expressed by using the \gls{svd}:
\begin{equation}
\label{eq:SVD}
\mathbf{K}= \mathbf{U} \Sigma \mathbf{U}^\top= \sum_{j=1}^R \sigma_j \mathbf{U}_{\cdot j}\mathbf{U}_{\cdot j} ^\top,
\end{equation}

{\bf The Moore-Penrose inverse} of $\mathbf{A}$ is defined by $\mathbf{A}^\dag =\mathbf{U} \boldsymbol{\Sigma}^{-1} \mathbf{V}^\top$, where $\boldsymbol{\Sigma}^{-1}$ is the diagonal matrix formed by the values $\frac{1}{\sigma_1},\cdots,\frac{1}{\sigma_R}$.

{\bf Matrix Norms.} The real-valued matrix norm involved in this paper is the Schatten $p-$norm, which is defined as:
\begin{align*}
   \| \mathbf{A}\|_p:= \left (\sum_{i=1}^{\min\{N,I\}} \sigma_i^p \right )^{1/p}
    \end{align*}
\noindent where $\sigma_i$ are the  eigenvalues of the matrix $\sqrt{\mathbf{A}^\top \mathbf{A}}$. In the case $\mathbf{a}\in \mathbb{R}^{N \times 1}$, the Schatten $p-$ norm $\|\mathbf{a}\|_p$ coincides with the well-known $l_p$ norm $\|\mathbf{a}\|_p= \left (\sum_{i=1}^N |a_i|^p\right )^{\frac{1}{p}}$, with $1\le p < \infty.$ (see \cite{HCH19}).
    
The case $p = 2$ yields the {\sl Frobenius norm},  $p = 1$ yields the {\sl nuclear norm} and $p = \infty$ yields the {\sl spectral norm}, which is the operator norm stemming from the vector 2-norm:
\begin{align*}
\| \mathbf{A}\|_\infty := \underset{\|\mathbf{a}\|_2=1}{\max} \frac{\|\mathbf{A}  \mathbf{a}\|_2}{\|\mathbf{a}\|_2}= \sqrt{\sigma_1(\mathbf{A}^\top\mathbf{A})}
\end{align*}

The Schatten $p-$norm was selected because it is  {\sl sub-multiplicative} and also {\sl unitarily invariant}.

To measure the error in the Taylor approximation of the softmax function,  the matrix norm induced by the vector 2- and $\infty$- norm is used:  
\begin{align*}
    \| \mathbf{A}_{\cdot \cdot }\|_{2, \infty} :=\underset{1\le n \le N}{\max} \|\mathbf{A}_{n\cdot}\|_2
\end{align*}
\noindent 


{\bf Linear span}. The span of $\mathbf{A}$ is  the set of all linear combinations of the columns of $\mathbf{A}$ and is denoted by $L(\mathbf{A}).$ The orthogonal complement of $L(\mathbf{A})$ is defined by
\begin{equation}
    L(\mathbf{A})^\perp := \left \{ \mathbf{v}\in \mathbb{R}^{N \times 1}:  \langle \mathbf{A}_{\cdot j},\mathbf{v}\rangle= \mathbf{A}_{\cdot j}^\top  \mathbf{v}=0; \hbox{ for all }  j=1,\cdots,I \right \}.
\end{equation}
The linear space can be expressed as the direct sum $\mathbb{R}^{N \times 1}=L(\mathbf{A}) \oplus L(\mathbf{A})^\perp.$ Moreover, if the linear space  $L(\widehat{\mathbf{A}}) \subseteq L(\mathbf{A})$, then $L(\mathbf{A})^\perp  \subseteq L(\widehat{\mathbf{A}})^\perp$.

{\bf Orthogonal projections}. Let  $ \left \{\mathbf{U}_{\cdot 1}, \cdots,\mathbf{U}_{\cdot R} \right  \}$ be  an orthonormal basis of  the  linear space $L(\mathbf{A})$. The orthogonal projection of a vector $\mathbf{v}$ onto the linear space $L(\mathbf{A})$ is:
\begin{equation}\label{eq:P}
\mathcal{P}^\perp_{L(\mathbf{A})} (\mathbf{a}):=\sum_{j=1}^R \langle \mathbf{a}, \mathbf{U}_{\cdot j}  \rangle \mathbf{U}_{\cdot j} \in L(\mathbf{A})
\end{equation}
\noindent and this expression is independent of the orthonormal basis. 

{\bf Projection matrix.} The projection matrix in the space generated by the columns of $\mathbf{A}$, $L(\mathbf{A})$, is given by $\boldsymbol{\Pi}^\perp_{L(\mathbf{A})}=\mathbf{A}(\mathbf{A}^\top \mathbf{A})^\dag \mathbf{A}^\top$. Then, Equation~(\ref{eq:P}) can be rewritten  as:
\begin{equation}
\mathcal{P}^\perp_{L(\mathbf{A})} (\mathbf{a})= \boldsymbol{\Pi}^\perp_{L(\mathbf{A})} \mathbf{a}
\end{equation}

{\bf Matrix notation.} For notational convenience, express  $\mathbf{B}_{\cdot \cdot} \in L(\mathbf{A})$ if and only if $\mathbf{B}_{\cdot j} \in L(\mathbf{A}) $ for all $j$. Let $\mathcal{P}^\perp_{L(\mathbf{A})} (\mathbf{B}_{\cdot \cdot})$  represent the matrix obtained by applying  $\mathcal{P}^\perp_{L(\mathbf{A})} (\mathbf{B}_{\cdot j})$ to each column individually. Using the projection matrix, $\mathcal{P}^\perp_{L(\mathbf{A})} (\mathbf{B}_{\cdot \cdot})= \boldsymbol{\Pi}^\perp_{L(\mathbf{A})} \mathbf{B}_{\cdot \cdot}.$ is satisfied.

\subsection{Analysis of the solution set of the KLR}

This section characterises the solution set of the \gls{klr} training problem by showing that it is a hyperplane in which the function  $\cal L$ remains constant. This result is deemed significant from two distinct perspectives. On the one hand, it would allow the so-called {\sl restricted training problem} to be formulated, meaning the optimal parameters can be obtained by means of an optimisation problem of dimension $\hbox{rank}(\mathbf{K})\cdot I\le N \cdot I$, suitably correcting the overspecification of the model (infinite solutions), which slows down many of the optimisation algorithms usually employed. On the other hand, it provides a theoretical basis for establishing error bounds in the estimation of optimal parameters when employing a Nyström approximation for \gls{klr}.

\begin{theorem}[Characterisation of the solution set of the  \gls{klr} training problem]
\label{th:soluciones}
The \gls{klr} general solution is expressed as:
\begin{equation}
\boldsymbol{\alpha}^*_{\cdot \cdot}=\boldsymbol{\aleph}^*_{\cdot \cdot} + \boldsymbol{\eta}_{\cdot \cdot}, 
\end{equation}
\noindent where $\boldsymbol{\eta}_{\cdot \cdot} \in L(\mathbf{K})^\perp$,  $\boldsymbol{\aleph}^*_{\cdot \cdot}\in L(\mathbf{K})$, and $\boldsymbol{\aleph}^*_{\cdot \cdot}$ is the unique solution to the problem  (\ref{eq:04_estimation_KLR_RNLL}) restricted to the subspace $L(\mathbf{K})$:
\begin{equation}
\label{eq:04_estimation_KLR_RNLL_restricted}
 \underset{\boldsymbol{\alpha}_{\cdot \cdot } \in L(\mathbf{K}) }{\hbox{Minimise }} {\cal L}(\boldsymbol{\alpha}_{\cdot \cdot}).
\end{equation}
\noindent The problem represented by (\ref{eq:04_estimation_KLR_RNLL_restricted}) will be known as the restricted training problem.

Furthermore, any optimal solution $\boldsymbol{\alpha}^*_{\cdot \cdot}$ satisfies the optimality conditions:
\begin{equation}
\label{eq:KKT}
\mathbf{K} \left [ N \lambda \boldsymbol{\alpha}^*_{\cdot i}   + \mathbf{p}_{\cdot i}(\boldsymbol{\alpha}^*_{\cdot \cdot}) - \boldsymbol{y}_{\cdot i} \right ] =\mathbf{0}, \hspace{1cm} i=1,\cdots I.
\end{equation}

\end{theorem}

The proof of this theorem is given in  \ref{App:overspecification}.

\begin{remark}
\label{remk:soluciones}
The training problem formulation commonly found in the literature  assumes that the model's parameters are overspecified (see \cite{Karsmakers2007}, \cite{OuA18}, \cite{ZhH05}). To expedite convergence of \gls{klr}, these parameters are generally set to zero for one alternative (typically the last one), $\boldsymbol{\alpha}_{\cdot I}=\mathbf{0}$. However, as illustrated in Theorem \ref{th:soluciones}, this approach is not valid  in general because $\boldsymbol{\alpha}^*_{\cdot \cdot} \in \boldsymbol{\Theta}=\left \{\boldsymbol{\alpha }_{\cdot \cdot} \in \mathbb{R}^{N \times I} :  {\alpha}_{\cdot I}=\mathbf{0} \right \}$ is not equivalent to $\boldsymbol{\alpha}^*_{\cdot \cdot}\in L(\mathbf{K})$ and it leads to a suboptimal solution.
Consequently, in this study, the training problem formulation given by Equation~(\ref{eq:04_estimation_KLR_RNLL}) is adopted instead of the restriction mainly used in the literature $\boldsymbol{\alpha}_{\cdot \cdot}\in \boldsymbol{\Theta}$.
\end{remark}

\section{Nyström-based Multinomial Kernel Logistic Regression} 
\label{sect:methodology}

Firstly, Section \ref{sec:rev_Nystrom} provides a comprehensive overview of the Nyström-based approximations.  After that, Section \ref{sec:Column} introduces the column selection methods tested in the computational experimentation. Finally, Section \ref{sec:cotas} contains the second most important theoretical contribution, concerning the error bounds when using the Nyström approximation  when applying \gls{klr}.

\subsection{Matrix sketching and the Nyström Method}
\label{sec:rev_Nystrom}

The Nyström method exploits the assumption that certain individuals (or observations) in the kernel matrix contain redundant information (i.e. rank$(\mathbf{K})<N$), and thus it is convenient to avoid duplicating this information. Then, the goal is to select the most representative columns  with the so-called {\sl sketching matrix}. Formally, given the kernel matrix $\mathbf{K}\in \mathbb{R}^{N \times N}$, it is called $\mathbf{C}=\mathbf{K}\mathbf{P} \in \mathbb{R}^{N \times C}$ (typically $C \ll N$) as a {\sl sketch} of $\mathbf{K}$ and $\mathbf{P} \in \mathbf{R}^{N \times C}$ is a {\sl sketching matrix}. The columns of $\mathbf{C}$ typically comprise a rescaled subset of the columns of $\mathbf{K}$; the former type of sketching is called {\sl column selection}.  An alternative approach involves considering random linear combinations of the columns of matrix $\mathbf{K}$, commonly referred to as a {\sl random projection.}

Given a kernel  matrix $\mathbf{K} \in \mathbb{R}^{N \times N}$ and  a sketching matrix $\mathbf{P} \in  \mathbb{R}^{N \times C}$, the Nyström method approximates $\mathbf{K}$  as follows:
\begin{equation}
\label{eq:Nystrom}
    \widehat{\mathbf{K}}= \mathbf{C} \mathbf{W}^\dag  \mathbf{C}^\top,
\end{equation}
\noindent where $\mathbf{W}=\mathbf{P}^\top \mathbf{K} \mathbf{P}$  and $^\dag$ denotes the Moore-Penrose pseudoinverse. The first thing to note about Equation~(\ref{eq:Nystrom}) is that the columns of the matrix $\hat{\mathbf{K}}$ are a linear combination of the columns of $\mathbf{C}$, so that $L(\mathbf{\widehat{K}}) \subseteq L(\mathbf{C})$. Furthermore, $\hbox{rank}(\mathbf{\widehat{K}}) = \hbox{rank}(\mathbf{C})$\footnote{First calculate a bound for the rank of the pseudoinverse of $\mathbf{W}$,
$
 \hbox{rank}(\mathbf{W}^\dag)=
 \hbox{rank}(\mathbf{W})=
 \hbox{rank}((\mathbf{K}^{1/2}\mathbf{P})^\top (\mathbf{K}^{1/2}\mathbf{P}) )\ge 
  \hbox{rank}((\mathbf{K}^{1/2}\mathbf{P})^\top \mathbf{K} (\mathbf{K}^{1/2}\mathbf{P}) )=
   \hbox{rank}(\mathbf{C}^\top \mathbf{C})=
    \hbox{rank}(\mathbf{C}).
 $
 
It follows that:
$
\hbox{rank} (\widehat{\mathbf{K}})=
\hbox{rank} (\mathbf{D}^\top\mathbf{D} )=
\hbox{rank}(\mathbf{D}^\top)
$
\noindent where $\mathbf{D}^\top =\mathbf{C} (\mathbf{W}^\dag)^{1/2}$. On the other hand, $\mathbf{W}^\dag$ is symmetric and positive semidefinite, therefore 
 $
  \hbox{rank}(\mathbf{W}^\dag)=
   \hbox{rank}( (\mathbf{W}^\dag) ^{1/2}   (\mathbf{W}^\dag)^{1/2}) = 
    \hbox{rank}( (\mathbf{W}^\dag) ^{1/2} ).
  $
It then follows that $ \hbox{rank}( (\mathbf{W}^\dag) ^{1/2}) \ge \hbox{rank} (\mathbf{C})$. If it is observed that $\mathbf{D}^\top$ is formed by more than rank$(\mathbf{C})$ independent linear combinations of all the columns of $\mathbf{C}$, it must therefore have a rank greater than or equal to rank$(\mathbf{C})$.  Since the maximum rank it can have is $\hbox{rank}(\mathbf{C})$ it follows that rank$(\widehat{\mathbf{K}}) = \hbox{rank}(\mathbf{C})$.}
 and from there it follows that
$
L(\widehat{\mathbf{K}})=L(\mathbf{C}).
$

Then, the Nyström matrix can be written:
\begin{equation}
\label{eq:projectionKernel}
\widehat{\mathbf{K}}= 
\mathbf{K}^{1/2}
\left \{
\mathbf{K}^{1/2} \mathbf{P} 
\left [[\mathbf{K}^{1/2} \mathbf{P}]^\top  \mathbf{K}^{1/2} \mathbf{P}
\right  ]^\dag   
[\mathbf{K}^{1/2} \mathbf{P}] ^\top 
\right \}\mathbf{K}^{1/2}.
\end{equation}

Identifying the projection matrix in Equation~(\ref{eq:projectionKernel}), the residual is expressed as:
\begin{equation}
\label{eq:residuo}
    \mathbf{K}- \widehat{\mathbf{K}}=
    \mathbf{K}^{1/2}
\left \{
\mathbb{I}_N-
\boldsymbol{\Pi}^\perp_{L(\mathbf{K}^{1/2} \mathbf{P})}
\right \}\mathbf{K}^{1/2}.
\end{equation}

The Nyström approximation $\widehat{\mathbf{K}}$ is the best possible $C-$ approximation to $\mathbf{K}$ given the information $\mathbf{C}$.  In more detail, for all symmetric matrices $\widehat{\mathbf{K}}$  spanned by the columns of $\mathbf{C}$ with a positive semi-definite residual matrix $\mathbf{K} - \widehat{\mathbf{K}}$, the Nyström approximation has the smallest error as measured by either any unitary  invariant norm, such as the Schatten$-p$ norm. Interested readers can find a proof of this statement in \citet{Epp22}.

\subsection{Column selection methods}
\label{sec:Column}

The previous theoretical result indicates that once the $\mathbf{C}$ matrix is given, the Nyström approximation is optimal. It is for this reason that the effort is expended on finding optimal strategies to select the sketching matrix $\mathbf{P}$. The sketching matrix for best $C-$rank approximation is given by   ${\mathbf P} = (\mathbf{U}_{\cdot 1},\ldots,\mathbf{U}_{\cdot C}) {\rm´diag}(\frac{1}{\sigma_1},\ldots,\frac{1}{\sigma_C})$ where $\mathbf{U}_{\cdot j}$ are the first elements of the orthonormal basis used in the \gls{svd} decomposition given in Equation~(\ref{eq:SVD}), assuming that ${\rm rank (\mathbf C)}=C$. This involves computing the $C$ eigenvectors corresponding to the largest eigenvalues $\sigma_1, \cdots, \sigma_C$. Regrettably, achieving high-precision calculation of these eigenvectors is computationally demanding. However, a viable alternative is to resort to column selection methods.

Columns selection obtains $\mathbf{C} \in \mathbb{R}^{N \times C}$ using a randomly sampled and rescaled subset of the columns of $\mathbf{K} \in \mathbb{R}^{N \times N}$.  The columns of the sketch are selected identically and independently as follows: Let $q_1,\cdots q_N \in (0,1)$, with $q_1+\cdots+q_N=1$, be the sampling probabilities associated with the columns of $\mathbf{K}$, each column of $\mathbf{C}$ is randomly sampled from the columns of $\mathbf{K}$ according to the sampling probabilities and rescaled by $\frac{1}{\sqrt{C q_i}}$, where $q_i$ is the sampling probability of column $i$. Three column selection techniques are detailed below.

\begin{enumerate}
\item {\sl Uniform sampling} is column sampling with $q_1=\cdots=q_N=\frac{1}{N}$. 
The simplest method involves randomly drawing $C$ samples from the matrix $\mathbf{K}$ and has been named as the \textit{Nyström KLR} method.

\item {\sl Leverage score sampling} \citep{CKR22} takes  $q_i=\frac{l_i}{\sum_{n=1}^N l_n}$ where  the ridge-leverage score  $l_i$ for any $i$ is expressed by:
\begin{equation}
    l_i= [\mathbf{K}[\mathbf{K}+\mu \mathbb{I}_N]^{-1}]_{ii}, \hbox{ where } \mu >0.
\end{equation}
The main limitation of this method is that it requires inverting the matrix  $(\mathbf{K}+\mu \mathbb{I}_N)$, resulting in a complexity of ${\cal O}(N^3)$. It is therefore necessary to approximate these scores. Two approaches have been successfully applied in the literature, both with a theoretical computational complexity of ${\cal O}(N \cdot S^2)$. The first one is based on the divide-and-conquer ridge-leverage \citep{CKR22}, which, instead of computing the leverage scores using all the data, divides the dataset into $Q$ disjoint small subsets of size $S$, and then computes the leverage scores for each subset. This method can be applied to the Nyström \gls{klr} and is named as \emph{DAC ridge-leverage Nyström KLR}. The second method is based on a recursive implementation of the ridge-leverage scores \citep{MuM17}. At each step, the ridge-leverage scores are recursively improved by approximating the kernel matrix using landmark points sampled from the previously computed scores. The implementation of this method with Nyström \gls{klr} will be named as \emph{RLS Nyström KLR}. In these methods, column selection corresponds to a sketching matrix $\mathbf{P}\in \mathbb{R}^{N \times C}$ that has exactly one non-zero entry in each column, whose position  corresponds to the index of the column selected from $\mathbf{K}$. 

\item {\sl k-means sampling}. 
\citet{OgG17, KMT12,ZTK08} study the effectiveness of kernel $k-$means samples as landmarks in the Nyström method for low-rank approximation of kernel matrices. In this method the $N$ columns of the kernel matrix $\mathbf{K}$ are grouped into $C$ groups and each of these groups is replaced by its mean ({\sl centroid}). 
 Among clustering methods with a fixed number of clusters, the $k-$means algorithm stands out as a widely used method. However, this technique can also be computationally expensive when dealing with large datasets.
For instance, the complexity of the kernel $k-$means clustering as applied to problem at hand ($k=C$) is ${\cal O}(N^3+N^2\cdot M+T\cdot N^2 \cdot C)$, where $M$ is the dimension of the instance space and $T$ the number of iterations completed. This complexity makes it unaffordable in the context of large data sets. A viable solution was proposed by \citet{WGM19} through the introduction of Nyström methods that would obtain a bound on computational cost ${\cal O}(N\cdot C^2+N\cdot M \cdot C+T\cdot N\cdot S\cdot C)$ with $S\ge C$. Alternatively, in this study, $k-$means are applied in the original space ${\cal X}$, leading to a computational cost ${\cal O}(T\cdot N\cdot C^2)$, and in addition, the mini-batch $k$-means method is used to reduce this  computational burden, as proposed by \citet{Scu10}. \citet{OgG17} gives a theoretical justification for why  selecting centroids obtained using $k-$means clustering in the instance space generates a good estimate of the kernel $k-$means centroids. For the latter centroids,  \citet{OgG17} show that under certain hypothesis they span the same subspace as the top $(C-1)$  singular vectors of a factorisation of the kernel matrix and, thus, define a good low-rank approximation of the kernel matrix. Once the cluster analysis has been performed, the most direct way to choose the sketching matrix $\mathbf{P}$ to weight the corresponding columns of $\mathbf{K}$, would be the same way the centroids were calculated in the original space ${\cal X}$. However, a novel way to do this is introduced, where the original data are  enlarged with the centroids of the $C$ groups. In such cases, it can be considered that the samples in the dataset comprise both the original individuals, denoted by the letter $X$, and the  created individuals associated with the centroids, denoted by $Y$. Then, an extended kernel matrix could be written in block form:\[
   \mathbf{K}^*=\left(\begin{array}{@{}c|c@{}}
    \begin{matrix}
    \mathbf{K}_{YY}
    \end{matrix}
    &\begin{matrix}
    \mathbf{K}_{YX}
    \end{matrix} \\
    \hline
    \begin{matrix}
    \mathbf{K}_{XY}
    \end{matrix} &
    \begin{matrix}
    \mathbf{K}
    \end{matrix}
    \end{array}\right).
\]

The original kernel matrix  $\mathbf{K}$ can be approximated using the Nyström approach on the extended kernel matrix $\mathbf{K}^*$ for $\mathbf{P}=[\mathbb{I}_C|\mathbf{0}]^\top $ as follows:
\begin{equation}
    \widehat{\mathbf{K}}= \mathbf{K}_{XY}  (\mathbf{K}_{YY} ) ^\dag  \mathbf{K}_{YX}.
\end{equation}

\end{enumerate}

\subsection{Nyström KLR error bounds}
\label{sec:cotas}

Now, the error incurred by using Nyström \gls{klr} instead of the  \gls{klr} is scrutinised by examining the norm of the difference between their optimal parameters: $\|\boldsymbol{\alpha}^*_{\cdot i}- \widehat{\boldsymbol{\alpha}}^*_{\cdot i}\|_p$. As a consequence of Theorem \ref{th:soluciones}, it may occur that $\|\boldsymbol{\alpha}^*_{\cdot i}- \widehat{\boldsymbol{\alpha}}^*_{\cdot i}\|_p \rightarrow \infty$ while $\|\mathbf{p}_{\cdot i}(\boldsymbol{\alpha}^*_{\cdot \cdot})- \mathbf{p}_{\cdot i}(\widehat{\boldsymbol{\alpha}}^*_{\cdot \cdot})\|_p \rightarrow 0$.  That is, both methods could give the same predictions with completely different parameters.
This forces the analysis to focus on comparing the respective solutions of the restricted training problems, i. e. by bounding $\|\boldsymbol{\aleph}^*_{\cdot i}- \widehat{\boldsymbol{\aleph}}^*_{\cdot i}\|_p$.

\begin{theorem}[Error bounds in Nyström \gls{klr}]
\label{th:cota1}
Let $\widehat{\mathbf{K}}$ be the Nyström approximation for the Kernel matrix $\mathbf{K}$ generated using the sketching matrix $\mathbf{P}$ and $\mathbf{C}=\mathbf{K}\mathbf{P}$. Let $\boldsymbol{\aleph}^*_{\cdot i}$ and $\widehat{\boldsymbol{\aleph}^*_{\cdot i}}$ be the optimal solution for \gls{klr}  and the optimal solution for Nyströnm-\gls{klr} for the kernel matrix $\widehat{\mathbf{K}}$,  respectively. 

Thus, $\widehat{\boldsymbol{\aleph}^*_{\cdot i}}={\mathcal P}^\perp_{L(\widehat{\mathbf{K}})} (\boldsymbol{\aleph}^*_{\cdot i})$ and an error bound is given by the following expression:

\begin{equation}
\| \boldsymbol{\aleph}^*_{\cdot i}- \widehat{\boldsymbol{\aleph}^*_{\cdot i}}  \|_p \le  \| \boldsymbol{\aleph}^{*}_{\cdot i} \|_p  \left \| \left [\mathbf{K}^2- \widehat{\mathbf{K}^2} \right ]^{1/2}\mathbf{K}^{\dag}\right \|_p ,
\end{equation} 
and also:
\begin{equation*}
      \|\boldsymbol{\aleph}^*_{\cdot i}\|_p\le
      \frac{2}{N\lambda}\| \mathbf{1} \|_p,
\end{equation*}
\noindent where $\mathbf{1}\in \mathbb{R}^N$ is a vector of 1s.
\end{theorem}

The proof of this theorem is given in the \ref{sec:demostracion_cotas}. The norm $\left \|\left [\mathbf{K}^2- \widehat{\mathbf{K}^2} \right ]^{1/2} \right \|_p$ depends on spectral decay of $\mathbf{K}$ and Theorem \ref{th:cota1} quantifies  the performance of the Nyström \gls{klr} method based on this value.

\section{Training Algorithms for the Nyström-based  Kernel Logistic Regression} 
\label{sect:algortihms}
First, Section~\ref{sect:gradient_klr} begins by elaborating on the process of calculating the gradient and the loss function of the Nyström \gls{klr} method and delving into the computational complexity it entails.  Lastly, Section~\ref{sect:line-search-methods} introduces a detailed discussion on the general theory of optimisation algorithms that can be effectively applied to efficiently estimate Nyström \gls{klr}.

\subsection{Calculating the gradient and loss function in Nyström-based KLR}
\label{sect:gradient_klr}

A comprehensive review of optimisation methods for large-scale \gls{ml} was conducted by \citet{BCN18}, emphasising the crucial role played by numerical optimisation algorithms, particularly the \gls{sgd} method, within the realm of \gls{ml}. The application of \gls{sgd} to specific classes of \gls{ml} models relies on the premise that the objective function used in training possesses a separable structure. This allows the objective function to be decomposed into the sum of individual loss functions for each sample $n=1,\ldots,N$ in the dataset, as depicted below:
\begin{equation}
    \mathcal{L}(\boldsymbol{\alpha}_{\cdot \cdot }) = \frac{1}{N}\sum_{n=1}^N \mathcal{L}_n(\boldsymbol{\alpha}_{\cdot \cdot }).
\end{equation}
Accordingly, the gradient of the objective function can also be expressed in a separable form concerning the individual samples:
\begin{equation}
    \nabla \mathcal{L}(\boldsymbol{\alpha}_{\cdot \cdot }) = \frac{1}{N}\sum_{n=1}^N \nabla \mathcal{L}_n(\boldsymbol{\alpha}_{\cdot \cdot }).
\end{equation}

The computational cost associated with model training is driven by the number of parameters and the dataset size. To mitigate the impact of the latter factor, a common approach employed in the literature is to randomly select a subset of samples at each iteration to approximate the gradient. Each of these sets is known as a \emph{mini-batch}. This approximation is feasible due to the separable nature of the gradient across the dataset. This characteristic is exploited by the mini-batch \gls{sgd}, which provides an efficient way of optimising models on large-scale datasets. 

The penalty term of the \gls{rnll}, as seen in Equation~(\ref{eq:04_estimation_KLR_RNLL}), disrupts the separable structure of the problem. For this reason, instead of the expression found in the existing literature, a separable expression was derived for the gradient of \gls{klr}. This expression is as follows:
\begin{equation}
    \label{eq:05_GKLR-gradient}
\nabla_{\boldsymbol{\alpha}} \mathcal{L}= \frac{1}{N}  \sum_{n=1}^N \left(N  \lambda\boldsymbol{\alpha}_{n\cdot } +(\mathbf{p}_{n\cdot}- \mathbf{y}_{n\cdot} ) \right ) \otimes \mathbf{k}_n ,
\end{equation}
\noindent where 
 $\otimes$ denotes the Kronecker product, used to obtain a compact notation in the gradient. This matrix operator is defined as follows. Let $\mathbf{A}\in \mathbb{R}^{c\times d}$ and $\mathbf{B}\in \mathbb{R}^{e\times f}$, then:
\begin{equation}
    \mathbf{A} \otimes \mathbf{B} = 
    \begin{bmatrix}
        A_{11} \mathbf{B} & & A_{1d} \mathbf{B} \\
        & \ddots & \\
        A_{c1}\mathbf{B}& & A_{cd} \mathbf{B}
    \end{bmatrix} \in \mathbb{R}^{c \cdot e\times d \cdot f}.
\end{equation} 
A detailed explanation of the derivation of this gradient can be found in \citet{MartinBaos23}.

From a computational point of view, it is convenient to work with matrices rather than vectors in the calculation of the gradient and the evaluation of the objective function, 
which is why Equation~(\ref{eq:05_GKLR-gradient}) is reshaped into a matrix:
\begin{equation}
 \label{eq:05_GKLR-gradient-matrix}
\nabla \mathcal{L}_{\cdot\cdot}(\boldsymbol{\alpha}_{\cdot \cdot }) = \frac{1}{N}\mathbf{K} \left [ N \lambda \boldsymbol{\alpha}_{\cdot \cdot}   + \mathbf{p}_{\cdot \cdot}(\boldsymbol{\alpha}_{\cdot \cdot}) - \boldsymbol{y}_{\cdot \cdot} \right ] \in \mathbb{R}^{N \times I}.
\end{equation}

Generally, when the sample size $N$ reaches tens of thousands, the computational time becomes prohibitively expensive, posing a significant challenge in handling such large datasets with \gls{klr} models. Efficiently computing the matrix of latent functions $\mathbf{f}_{\cdot \cdot }$ is crucial for obtaining $\mathbf{p}_{\cdot \cdot}$, as stated in Equation~(\ref{eq:04_proba_KLR}), and therefore evaluating the objective function and its gradient during the training of \gls{klr}. This is the motivation behind the use of the Nyström approximation, and below is given an outline of its application to the calculation of the gradient and loss function.

\begin{align}
\widehat{\mathbf{f}}_{\cdot \cdot }(\boldsymbol{\alpha}_{\cdot \cdot})=&\widehat{\mathbf{K}} \boldsymbol{\alpha}_{\cdot \cdot}
= \,\mathbf{C} \left [\mathbf{W}^\dag [\mathbf{C}^\top \boldsymbol{\alpha}_{\cdot\cdot}] \right ].
\label{eq:04_estimation_KLR_Nystrom1}\\
\widehat{\mathbf{p}}_{\cdot \cdot}(\boldsymbol{\alpha}_{\cdot \cdot})=&
\left [  \left  (
\exp (\widehat{\mathbf{f}}_{\cdot \cdot }(\boldsymbol{\alpha}_{\cdot \cdot}) ) \mathbf{1}
\right ) _{\circ -1} \right ]^\top 
\exp (\widehat{\mathbf{f}}_{\cdot \cdot }(\boldsymbol{\alpha}_{\cdot \cdot}) ).
\\
\nabla \mathcal{L}_{\cdot\cdot}(\boldsymbol{\alpha}_{\cdot \cdot }) = &\frac{1}{N}
\, \mathbf{C} \left [\mathbf{W}^\dag \left [\mathbf{C}^\top
\left [ N \lambda \boldsymbol{\alpha}_{\cdot \cdot}  + \widehat{\mathbf{p}}_{\cdot \cdot}(\boldsymbol{\alpha}_{\cdot \cdot}) - \boldsymbol{y}_{\cdot \cdot} \right ] \right] \right ].
\label{eq:04_estimation_KLR_Nystrom2}
\\
\mathcal{L} (\boldsymbol{\alpha}_{\cdot \cdot }) = & 
\mathbf{1}^\top [\log(\widehat{\mathbf{p}}_{\cdot \cdot}(\boldsymbol{\alpha}_{\cdot \cdot}))^\top \mathbf{y}_{\cdot \cdot }] \mathbf{1}
+\frac{\lambda}{2} \hbox{Trace} (\boldsymbol{\alpha}^\top_{\cdot \cdot } \widehat{\mathbf{f}}_{\cdot \cdot }(\boldsymbol{\alpha}_{\cdot \cdot})).
\end{align}
\noindent

where $\mathbf{1} \in \mathbb{R}^{I \times 1}$, $\hbox{Trace}(\mathbf{A})$ denotes the trace of the matrix $\mathbf{A}$, the operation $\mathbf{A}_{\circ -1}$ in the previous expressions represent the Hadamard power operation. Furthermore, both this operation and the functions $\log(\mathbf{A})$ and $\exp(\mathbf{A})$ are applied element-wise on the matrix $\mathbf{A}$.

The Nyström method offers two significant advantages. The first pertains to the spatial complexity of the model, as it reduces the memory required for storing the kernel matrix from ${\cal O}(N^2)$ to $ {\cal O}(N \cdot  C)$. The second advantage is related to the temporal complexity. The computation of the objective function and its gradient involves successive kernel matrix multiplications, as seen in Equations~(\ref{eq:04_estimation_KLR_Nystrom1}) and (\ref{eq:04_estimation_KLR_Nystrom2}). The Nyström method enables the computational cost for each kernel matrix multiplication to be reduced from $ {\cal O}(N^2 \cdot I)$ to $ {\cal O}(N \cdot   C \cdot I)$ (typically $C \ll N$).

\subsection{Line search methods}
\label{sect:line-search-methods}

In practical optimisation problems, the choice of the optimisation algorithm depends on a range of factors, such as the problem's dimensions and complexity, the desired level of accuracy, and the available computational resources. This study uses {\sl line search methods} to estimate the parameter matrix $\boldsymbol{\alpha}_{\cdot \cdot}$. These techniques update the parameter matrix iteratively through the following formula:
\begin{eqnarray} 
    \boldsymbol{\alpha}^{(t+1)}_{\cdot \cdot} = \boldsymbol{\alpha}^{(t)}_{\cdot \cdot } - \delta_t g(\boldsymbol{\alpha}^{(t)}_{\cdot \cdot }).
\end{eqnarray}

Here, $\delta_t$ represents the {\sl learning rate}, conforming to the condition $\delta_t>0$, and $g(\boldsymbol{\alpha}^{(t)}_{\cdot \cdot})$ is the search direction. This search direction can take different forms, with three common options being:
\begin{itemize}
    \item[(i)] When $g(\boldsymbol{\alpha}^{(t)}_{\cdot \cdot})=\nabla \mathcal{L}_{\cdot \cdot}(\boldsymbol{\alpha}^{(t)}_{\cdot \cdot})$, it leads to the \gls{gd} method.
    \item[(ii)] When $g(\boldsymbol{\alpha}^{(t)}_{\cdot \cdot})=[\nabla ^2\mathcal{L}(\boldsymbol{\alpha}^{(t)}_{\cdot  \cdot })]^{-1} \nabla \mathcal{L}_{\cdot \cdot}(\boldsymbol{\alpha}^{(t)}_{\cdot \cdot})$, where $\nabla ^2\mathcal{L}(\boldsymbol{\alpha}^{(t)}_{\cdot \cdot})$ is a proper representation of the Hessian matrix, it leads to Newton's method.  Under the assumption that the Hessian matrix maintains Lipschitz continuity and the function $f$ exhibits strong convexity at $\boldsymbol{\alpha}^{(t)}_{\cdot \cdot}$, Newton's method shows quadratic convergence towards the optimal parameters. 
    \item[(iii)] When $g(\boldsymbol{\alpha}^{(t)}_{\cdot \cdot })=\mathbf{H}^{(t)}\nabla \mathcal{L}_{\cdot \cdot}(\boldsymbol{\alpha}^{(t)}_{\cdot \cdot})$, where $\mathbf{H}^{(t)}$ is a positive definite matrix approximating the Hessian matrix, it results in quasi-Newton methods.
\end{itemize}

Convergence rate is a crucial aspect of optimisation algorithms, as it determines how quickly the algorithm can reach an optimal solution. Among popular optimisation methods, the \gls{gd} method exhibits a linear convergence rate. On the other hand, quasi-Newton methods demonstrate superlinear convergence, while Newton's method shows quadratic convergence. Algorithms with higher convergence rates require fewer iterations to achieve a certain level of accuracy. However, it is important to consider that these benefits come at the cost of increased computational complexity, especially in the case of Newton's method, which requires computing the Hessian matrix at each iteration.

This study employs a selection of algorithms with a variety of convergence rates that have demonstrated effectiveness in solving optimisation problems involving extensive datasets. Specifically, the quasi-Newton method, L-BFGS-B \citep{BLN95}, was chosen for its superlinear convergence rate, which offers efficient convergence. Additionally, the study explored the \gls{gd} method as a representative of an algorithm with a linear convergence approach, and two widely-used variants of the \gls{gd} method: the Momentum \gls{gd} and the Adam method.

Momentum \gls{gd} is a variant of the standard \gls{gd} algorithm that helps to accelerate convergence by adding a momentum term to the update rule. The update at each iteration is based not only on the current gradient but also on the accumulated gradient of past iterations. The momentum term acts as a moving average of past gradients, which smooths the search trajectory and helps the algorithm avoid getting stuck in local minima. This smoothing effect is particularly beneficial in the presence of noise or sparse gradients. The momentum term is controlled by a hyperparameter, typically denoted as $\beta\in[0,1]$. A $\beta$ value close to one gives more weight to the past gradients, resulting in a more stable search trajectory, while a $\beta$ value close to zero puts more emphasis on the current gradient, allowing the algorithm to respond quickly to changes in the optimisation landscape. A variation of this technique is called the Nesterov Accelerated Gradient, which incorporates a look-ahead step that anticipates the momentum's effect on the gradient direction, resulting in improved convergence properties and faster optimisation. The Momentum \gls{gd} algorithm has been shown to be effective in accelerating convergence in various optimisation problems, especially in deep learning applications. The search direction used for this algorithm is stated as follows:
\begin{eqnarray} 
    g(\boldsymbol{\alpha}^{(t)}_{\cdot \cdot}) = \beta g(\boldsymbol{\alpha}^{(t-1)}_{\cdot \cdot}) + (1-\beta)\nabla \mathcal{L}_{\cdot \cdot}(\boldsymbol{\alpha}^{(t)}_{\cdot \cdot}).
\end{eqnarray}

Adam \citep{KiB14} is another popular optimisation algorithm that combines the advantages of two other optimisation methods, Momentum and RMSprop. Adam uses a momentum term that estimates the first-order and second-order moments of the gradients to update the parameters, resulting in a more efficient optimisation process. This is particularly effective in non-convex optimisation problems with large and sparse datasets, where other optimisation algorithms may struggle. Despite its effectiveness, the performance of Adam is heavily dependent on the choice of hyperparameters, such as the learning rate, or the momentum parameters $\beta_1\in[0,1]$ and $\beta_2\in [0,1]$. Consequently, precise calibration of these hyperparameters is crucial to attaining optimal performance across varying optimisation scenarios. This algorithm is stated as follows:
\begin{align} 
    \mathbf{M}^{(t)} &= \beta_1 \mathbf{M}^{(t-1)}+ (1-\beta_1)\nabla \mathcal{L}_{\cdot \cdot }(\boldsymbol{\alpha}^{(t)}_{\cdot \cdot })\\
    \mathbf{V}^{(t)} &= \beta_2 \mathbf{V}^{(t-1)}+ (1-\beta_2)\nabla \mathcal{L}_{\cdot \cdot}(\boldsymbol{\alpha}^{(t)}_{\cdot \cdot })_{\circ 2}\\
    g(\boldsymbol{\alpha}^{(t)}_{\cdot \cdot }) &= \dfrac{\sqrt{1-\beta^t_2}}{1-\beta^t_1}  \frac{\mathbf{M}^{(t)}}{\mathbf{V}^{(t)}_{\circ \frac{1}{2}} + \varepsilon},
\end{align}
where the operations $\mathbf{A}_{\circ 2}$ and $\mathbf{A}_{\circ \frac{1}{2}}$ in the previous formulae represent the Hadamard power and root operations, respectively, which are performed element-wise on the matrix $\mathbf{A}$. $\beta^t_1$ and $\beta^t_2$ denote the $t$-th power of $\beta_1$ and $\beta_2$, respectively. Additionally, $\varepsilon\ge 0$.

The learning rate, often referred to as the \emph{step size}, plays a pivotal role in optimisation algorithms, as it determines the speed at which the optimiser descends the error curve. Selecting an appropriate learning rate is critical for achieving timely convergence in optimisation processes. In methods that use a high-quality search direction, such as quasi-Newton methods, determining the optimal step size typically involves solving a one-dimensional optimisation problem, or using the Armijo rule to ensure sufficient descent. However, in the context of \gls{ml} problems, the Armijo rule may introduce computational overheads and is thus less commonly used. Consequently, a frequently employed approach is to use a constant learning rate $\delta_t = \delta$ at all iterations. However, this approach may lead to convergence issues if the learning rate is not chosen appropriately.

To overcome this limitation, a more dynamic approach can be adopted, generating a sequence of decreasing learning rates using various methods, such as the \emph{time-based}, \emph{exponential}, or \emph{step} learning rate decay methods. These approaches provide more control over the step size, and have the potential to expedite convergence. This study employs a time-based learning rate decay strategy, where the learning rate at iteration $t$ is determined by  $\delta_t = \frac{\delta_0}{1+\gamma t}$, where $\delta_0$ is the initial learning rate, and $\gamma$ is a decay rate.

Now that the theoretical basis of this proposal has been established, the next section introduces the experimental results.

\section{Experimental results}
\label{sect:experiments}

The numerical results obtained from the experiments are now presented. Firstly, Section~\ref{Sect:Nystrom_Setup} describes the two large datasets selected, LPMC and NTS. These  datasets were widely used in previous comparisons in the literature. Additionally, this section introduces  \gls{ml} classifiers that have demonstrated good performance in previous work, and which were replicated in the experimentation here. To ensure a concise methodology, the section finishes by outlining the steps taken in this study. Subsequently, two experiments were conducted and are presented in Sections~\ref{Sect:Comp_Nystrom_KLR} and \ref{Sect:Comp_Nystrom_KLR_vs_ML}: the first compared the four Nyström methods implemented for \gls{klr}, while the second compared these methods with other representative \gls{ml} and \gls{mnl} models. Finally, Section~\ref{Sect:Comp_optimisation_Nystrom_KLR} assesses several optimisation algorithms applied to the Nyström \gls{klr} method, aiming to determine which is the most efficient.

\subsection{Experimental setup} 
\label{Sect:Nystrom_Setup}

All numerical tests were conducted using Python 3 on a Linux computer with Ubuntu 20.04 LTS. The system specifications included a $3.8$ GHz 12-core AMD Ryzen 3900xt processor and $32$ GB of RAM. The Python programming language was chosen for its flexibility and rich ecosystem of libraries that implement \gls{ml} and \gls{mnl} methods. Moreover, Python is extensively used in the majority of research studies in the literature, providing a solid foundation for the comparison and reproducibility of results.

\subsubsection{Datasets}
\label{Sect:Datasets}

The numerical experiments in this study used two datasets, LPMC and NTS. These datasets were selected for their large sample sizes and their use in state-of-the-art comparisons of \gls{ml} methods for travel mode choice, which make them well-suited for this research. 

The LPMC dataset \citep{HEJ18} used in this study comprises London Travel Demand Survey data, enriched with additional variables sourced from a directions API. It encompasses a substantial sample size of $81,096$ entries, collected from $17,616$ participants, and consists of $31$ variables, of which $20$ were specifically selected for this study. Categorical variables in the dataset were encoded using the one-hot encoding technique. The choice variable in the LPMC dataset has four distinct categories: walk ($17.6\%$), bike ($2.98\%$), public transport ($35.28\%$), and car ($44.16\%$).

Conversely, the NTS dataset \citep{HaH17} is a ML-focused dataset containing Dutch National Travel Survey results from 2010 to 2012, with a total of $230,608$ surveys. However, in order to speed up the comparative analysis, only $50\%$ of the samples were used in this experiment. The NTS dataset includes sixteen variables, all of which were considered. Categorical variables in the dataset were also encoded using one-hot encoding. Like the LPMC dataset, the choice variable in NTS also has four categories: walk ($16.29\%$), bike ($24.41\%$), public transport ($4.03\%$), and car ($55.26\%$).

\subsubsection{Proposed classifiers}
\label{Sect:Proposed_classifiers}

The \gls{mnl} model is usually considered the baseline model for discrete choice analysis. This study employs the open-source Python package Biogeme \citep{Bie03} to estimate the \gls{mnl} models using a \gls{mle} approach. For the NTS dataset, linear utilities specified over all attributes were considered, using distinct parameters for each alternative. However, for the LPMC dataset, a utility function with individual-specific attributes and alternative-varying attributes was defined. In this case, all features were selected as individual-specific, except for certain features which were chosen per alternative:
\begin{itemize}
\item[-] Walk: \textit{distance} and \textit{dur\_walking}.
\item[-] Bike: \textit{distance} and \textit{dur\_cycling}.
\item[-] Public transport: \textit{dur\_pt\_access}, \textit{dur\_pt\_rail}, \textit{dur\_pt\_bus}, \textit{dur\_pt\_int\_waiting}, \textit{dur\_pt\_int\_walking}, \\
\textit{pt\_n\_interchanges}, and \textit{cost\_transit}.
\item[-] Car: \textit{dur\_driving} and \textit{cost\_driving\_total}.
\end{itemize}

Now, \gls{ml} methods included in this research are briefly described:
\begin{itemize}
  \item \acrfull{svm}. This is a binary classifier that employs a \textit{kernel} to transform the data into a high-dimension space, aiming to identify the optimal linear decision surface (hyperplane) that separates data into two classes. This paper incorporates two variants of this model: one with a linear kernel, a model frequently included in prior studies comparing \gls{mnl} and \gls{ml} methods, and another \gls{svm} equipped with a radial basis function kernel. The latter model presents estimation problems due to the large size of the datasets used. This is addressed by the use of a strategy similar to that used for \gls{klr}, applying the Nyström method. This is done using the same configuration as in the base Nyström \gls{klr} model for each dataset.
  A \textit{one-vs-one} strategy is applied to extend \gls{svm} to multi-class classification. The one-vs-one classification strategy trains a binary \gls{svm} classifier for each unique pair of classes and predicts the class that receives the most votes. Note that \gls{svm} does not directly provide probability estimates for each class. Hence, the probabilities were approximated using Platt's method, which consists of training the \gls{svm} using $5-$fold CV, and subsequently applying a logistic regression to the \gls{svm} scores.

  \item \acrfull{rf}. This classifier works by training several decision trees in parallel, using a random subset of attributes in order to build an uncorrelated forest. In a decision tree, nodes represent binary decision rules, while leaves signify distinct classes. The predictions of the \gls{rf} are determined by aggregating the votes of these individual decision trees. The ensemble of trees yields more accurate classifications compared to the average performance of an individual decision tree. 
  
  \item \acrfull{xgboost}. Gradient Boosting is a tree-based ensemble method first proposed by \citet{Fri01}. This method is grounded in iteratively growing low-depth decision trees based on the idea of additive training. Each individual decision tree tries to predict the total residual error of the previous decision tree. \gls{xgboost} is open-source and implements a scalable and efficient version of the Gradient Boosting algorithms. 
  
  \item \acrfull{nn}. Neural models are extensively employed for function approximation and classification tasks. In this study, we adopt a \gls{mlp} topology due to its simplicity and its capability to capture intricate non-linear relationships between input variables and the predicted class. This architecture begins with an input layer, comprising one neuron for each attribute in the dataset. Subsequently, a hidden layer follows, which contains a variable number of neurons interconnected with the input layer's neurons. Finally, the output layer of the model consists of neurons corresponding to the number of classes in the problem. These output neurons provide probabilities indicating the probability of a given instance belonging to each class.
\end{itemize}

Next, \gls{rf}, \gls{svm}, and \gls{nn} methods were implemented using the \textit{scikit-learn} Python package, and the \textit{XGBoost} package \citep{Chen16} was used for \gls{xgboost}. With regard to  the \gls{klr} model, it was implemented using the \textit{GKLR} package\footnote{The \textit{GKLR} project is available at \url{https://github.com/JoseAngelMartinB/gklr} and can be installed from the PIP repository at \url{https://pypi.org/project/gklr/} by executing the command \textit{pip install gklr} in your terminal or command prompt.} developed by the authors.

\subsubsection{Model validation}
\label{Sect:Model_validation}

The datasets used in this study were partitioned into training and testing subsets. For the NTS dataset, a random split was made, with $70\%$ of the data used for training and $30\%$ for testing. In line with the methodology employed by \citet{HEJ18}, the LPMC dataset was divided into a training set comprising the first two years of data (approximately $70\%$ of the dataset) and a test set comprising the last year of data. Using the training set as reference, all variables in both training and test sets were normalised, centering them to zero mean and unit standard deviation. It should be noted that no measures were taken to address the class imbalance in the datasets. 

In their study, \citet{Hil21} highlighted that employing trip-wise sampling with panel data introduces a bias in model performance assessments, especially when dealing with flexible non-linear \gls{ml} classifiers. To adhere to the crucial principle of validating \gls{ml} classifiers on unseen out-of-sample data, it is essential to ensure that responses from the same individuals (or households) are entirely included in either the training or test set. This requirement is easily met in the LPMC dataset; however, the NTS dataset lacks such information. To address this issue, the approach proposed by \citet{Hil21} was followed, which involves identifying households based on a unique combination of geographic area (inferred from the variables \textit{diversity} and \textit{green}) and socio-economic information about the individuals (comprising variables such as \textit{age}, \textit{gender}, \textit{ethnicity}, \textit{education}, \textit{income}, \textit{cars}, \textit{bikes}, and \textit{driving license ownership}). 

After preparing the training and test datasets, the subsequent crucial step involves training the models using the selected methods. However, it is critical first to undertake a hyperparameter tuning phase to identify the optimal set of hyperparameters for the \gls{ml} models. This process also includes fine-tuning the parameters of the optimisation algorithms. For instance, in the case of \gls{klr}, parameters such as the learning rate decay method, the initial learning rate, and the decay rate were adjusted.

This hyperparameter optimisation stage was performed using the \textit{hyperopt} Python package, employing the \gls{cel} metric to evaluate the quality of the hyperparameters. To ensure a robust estimation of this metric, a 5-fold cross-validation approach on the training data was used. $1,000$ function evaluations were carried out to adjust the hyperparameters. The details regarding the search space and the optimal hyperparameter values for each model are provided in \ref{App:HPO} to facilitate transparency and reproducibility of this research.

The next step is training the models and assessing their performance. The study proposes two performance metrics: \gls{dca} and \gls{gmpca}. Within the \gls{ml} community, the most frequently-used benchmark index is \gls{dca} (also known as accuracy), which quantifies the proportion of observations from the model that have been correctly predicted. Although this is not a recommended index for comparing discrete choice models, it is used in this study to enable comparisons with results from earlier studies in the literature. Accuracy is defined as
\begin{equation}
  \gls{dca} = \dfrac{1}{N_{\hbox{\tiny test}}} \sum_{n}^{N_{\hbox{\tiny test}}} \mathbb{I} (y_n=\widehat y_n),
\end{equation}
\noindent where $N_{\hbox{\tiny test}}$ is the number of samples in the test set.

According to the findings presented in \citet{Hil19}, it should be emphasised that performance measures should include probability-based indices, due to their crucial role in the behavioural analysis of models. Traditionally, log-likelihood is the metric applied in \gls{rum} studies, in order to identify the model that best explains a given dataset. Nevertheless, this metric has two inherent difficulties: i) its sensitivity to dataset size, and ii) its challenging interpretability. To address these shortcomings, \citet{Hil19} advocate for the adoption of the \gls{gmpca} as an alternative indicator. The \gls{gmpca} is a probability-based index that measures the geometric mean of the probabilities of each individual, to choose the correct alternative according to the model, hence, this index can be interpreted as a normalised likelihood function. The advantage of the \gls{gmpca} is that it has a clear physical interpretation as a robust measure of the average correctness of the model. In light of this recommendation, the \gls{gmpca} index was incorporated into the analysis, and is expressed as follows: 
\begin{equation}
  \gls{gmpca}= \left ( \prod_{n}^{N_{\hbox{\tiny test}}} p_{n, i_n } \right ) ^{\frac{1}{N_{\hbox{\tiny test}}}},
\end{equation}
\noindent where $i_n$ is the alternative chosen by individual $n$, and $N_{\hbox{\tiny test}}$ is the number of samples in the test set.

\subsection{Comparison of Nyström methods for KLR on large datasets}
\label{Sect:Comp_Nystrom_KLR}
    
Theorem \ref{th:cota1} shows that the convergence of the parameters of the Nyström \gls{klr} to the parameters of the \gls{klr} depends on the norm $\| (\mathbf{K}^2 -\widehat{\mathbf{K}}^2)^{1/2}\|_p$. More specifically, if
  ${\cal E} (C)$ is the Schatten $p-$norm of the error in the parameters of the Nyström \gls{klr}, which depends on the number of landmark points $C$, it is proved that  ${\cal E}(C)=O(\| (\mathbf{K}^2 -\widehat{\mathbf{K}}^2)^{1/2}\|_p)$.

The behaviour of this bound will be analysed numerically for two transport mode choice problems. The theorem is satisfied for any Nyström approximation, in particular for the best $C-$rank approximation, which is given by the expression $\widehat{\mathbf{K}}_C= \sum_{j=1}^C \sigma_j  \mathbf{U}_{\cdot j} \mathbf{U}_{\cdot j}^\top$, obtaining:
 \begin{eqnarray}
 (\mathbf{K}^2 -\widehat{\mathbf{K}}_C^2)^{1/2} &= \left (\sum_{j=1}^R \sigma^2_j  \mathbf{U}_{\cdot j} \mathbf{U}_{\cdot j}^\top - \sum_{j=1}^C \sigma^2_j  \mathbf{U}_{\cdot j} \mathbf{U}_{\cdot j}^\top \right ) ^{1/2}= \left (
 \sum_{j=C+1}^R \sigma^2_j  \mathbf{U}_{\cdot j} \mathbf{U}_{\cdot j}^\top \right )^{1/2}\\
 & =
 \sum_{j=C+1}^R \sigma_j  \mathbf{U}_{\cdot j} \mathbf{U}_{\cdot j}^\top \Rightarrow
 \| (\mathbf{K}^2 -\widehat{\mathbf{K}}_C^2)^{1/2}\|_p= \left ( \sum_{j=C+1} ^R \sigma^p_j \right )^{1/p}.
\end{eqnarray}

Considering the spectral norm ($p=\infty$), it follows that  $ \| (\mathbf{K}^2 -\widehat{\mathbf{K}}_C^2)^{1/2}\|_\infty = \sigma_{C+1}$.  This leads to the fact that the error bound, as a function of the number of landmark points $C$, has an expression of the form ${\cal E}(C) \le \delta \sigma_{C+1}$.

Figure \ref{fig:nystrom_singular_values} shows the decay of the  spectrum of $\mathbf{K}$, $\sigma_{C+1}$ versus $C$, in the two datasets. Therefore,  
the error bounds of the parameters will be proportional to the functions shown in Figure \ref{fig:nystrom_singular_values}. Fitting the regression model   $\sigma_{C+1}= \dfrac{a}{C^b}$, the fitted model  $\sigma_{C+1}= \dfrac{e^{6.58}}{C^{1.36}}$ is obtained for the LPMC problem, and $\sigma_{C+1}= \dfrac{e^{6.51}}{C^{1.30}}$ for the NTS dataset. These results indicate a rapid spectral decay ($b>1$) in these application examples, suggesting that a large number of landmarks is not required for an effective Nyström approximation. Additionally, it is observed that $\|(\mathbf{K}^2 -\widehat{\mathbf{K}}_C^2)^{1/2}\|_p=\|\mathbf{K} -\widehat{\mathbf{K}}_C\|_p$, and thus by selecting more than $500$ landmarks, the error of Nyström matrix approximations with respect to the kernel matrix $\mathbf{K}$ is of the order of $\sigma_{500}\approx 10^{-1}$.

\begin{figure}[h]
	\centering
    \includegraphics[width=\textwidth]{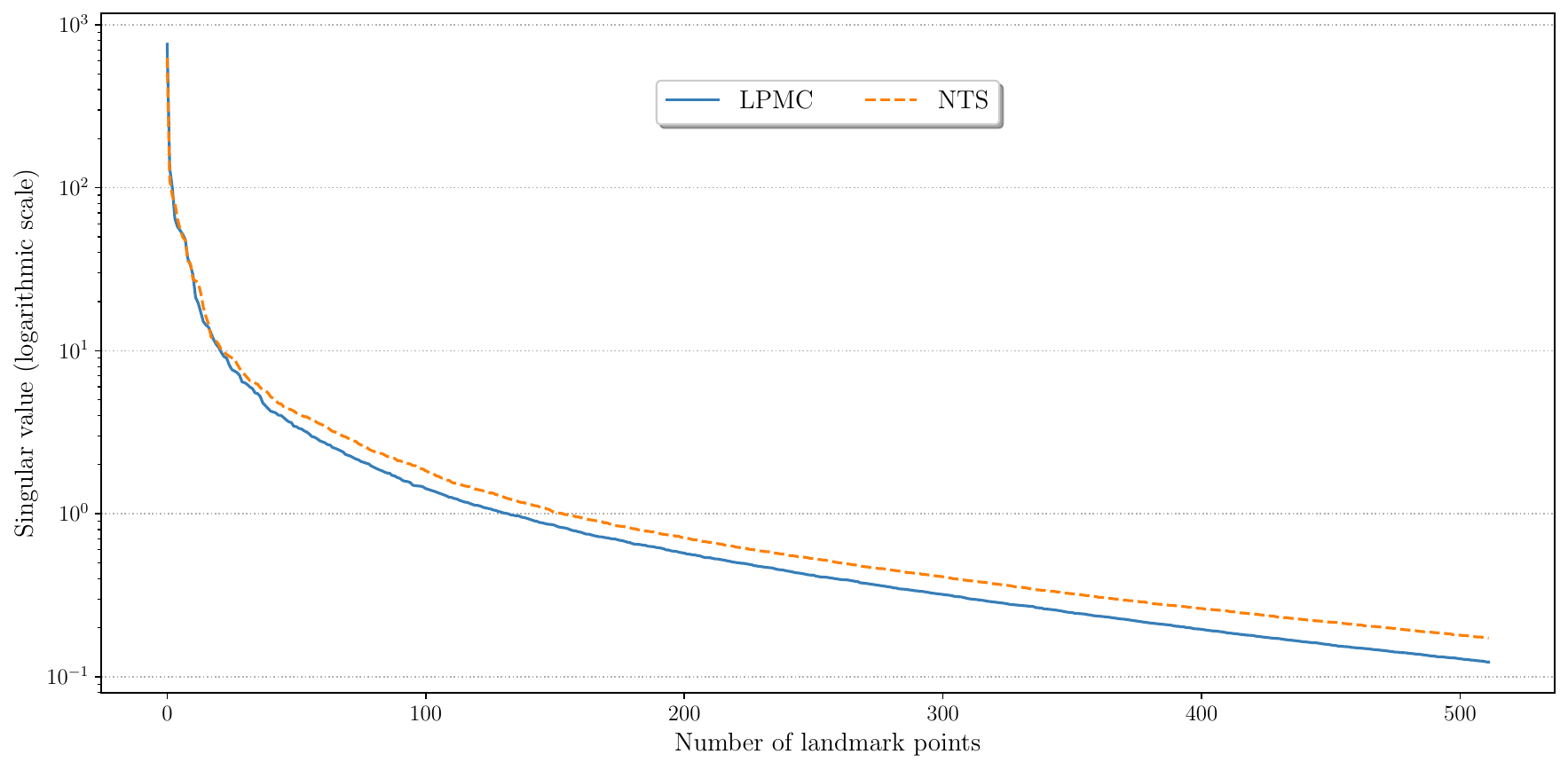}
    \caption{Singular values of the kernel matrix $\mathbf{K}$ for the LPMC and NTS datasets}
    \label{fig:nystrom_singular_values}
\end{figure}

Next, \gls{gmpca}, \gls{dca}, and the computational cost of various implementations of \gls{klr}, are evaluated,  using the column selection methods presented in Section~\ref{sec:Column}. For each implementation of \gls{klr}, a \gls{rbf} kernel was selected, and the parameters of these models were estimated using the L-BFGS-B algorithm. Specifically, the comparison focuses on determining the number of landmark points $C$ required in each Nyström approximation to achieve satisfactory results. The question that arises is whether the new metrics exhibit the same pattern as that registered in this experiment. Intuition indicates that they should be similar and that therefore these metrics should  be seen to give small improvements as they go beyond $300$ landmarks. 
          
The results for the LPMC and NTS datasets are illustrated in Figures~\ref{fig:LPMC_KLR_comp} and \ref{fig:NTS_KLR_comp}, respectively. The standard Nyström method using uniform sampling is denoted by the blue line, the Nyström method using the $k$-means sampling strategy is represented by the orange line, the DAC ridge-leverage Nyström by the green line, and the RLS-Nyström by the red line.

\begin{figure}[p]
	\centering
    \includegraphics[width=\textwidth]{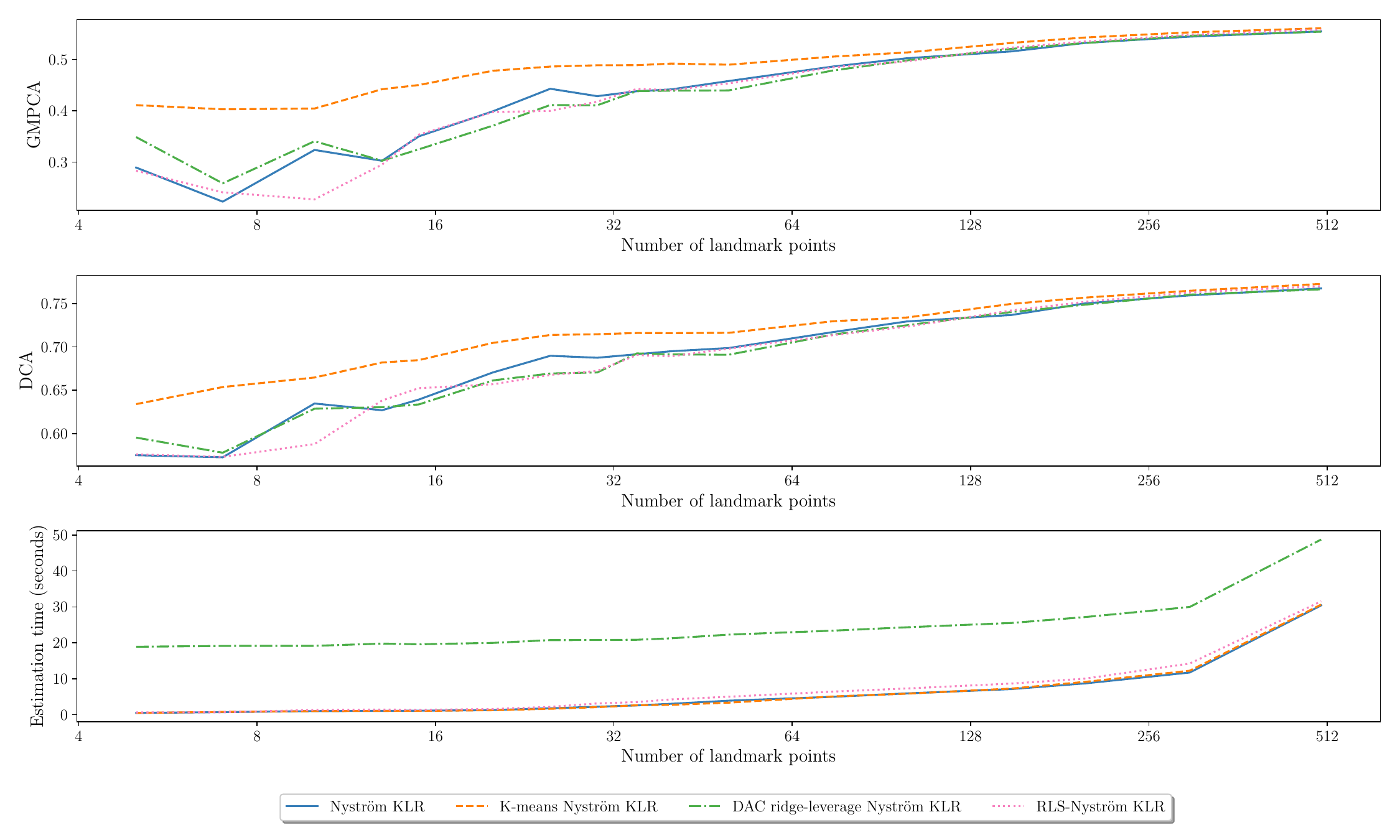}
    \caption{Comparison of Nyström sampling techniques on the LPMC dataset}
    \label{fig:LPMC_KLR_comp}
\end{figure}

\begin{figure}[p]
	\centering
    \includegraphics[width=\textwidth]{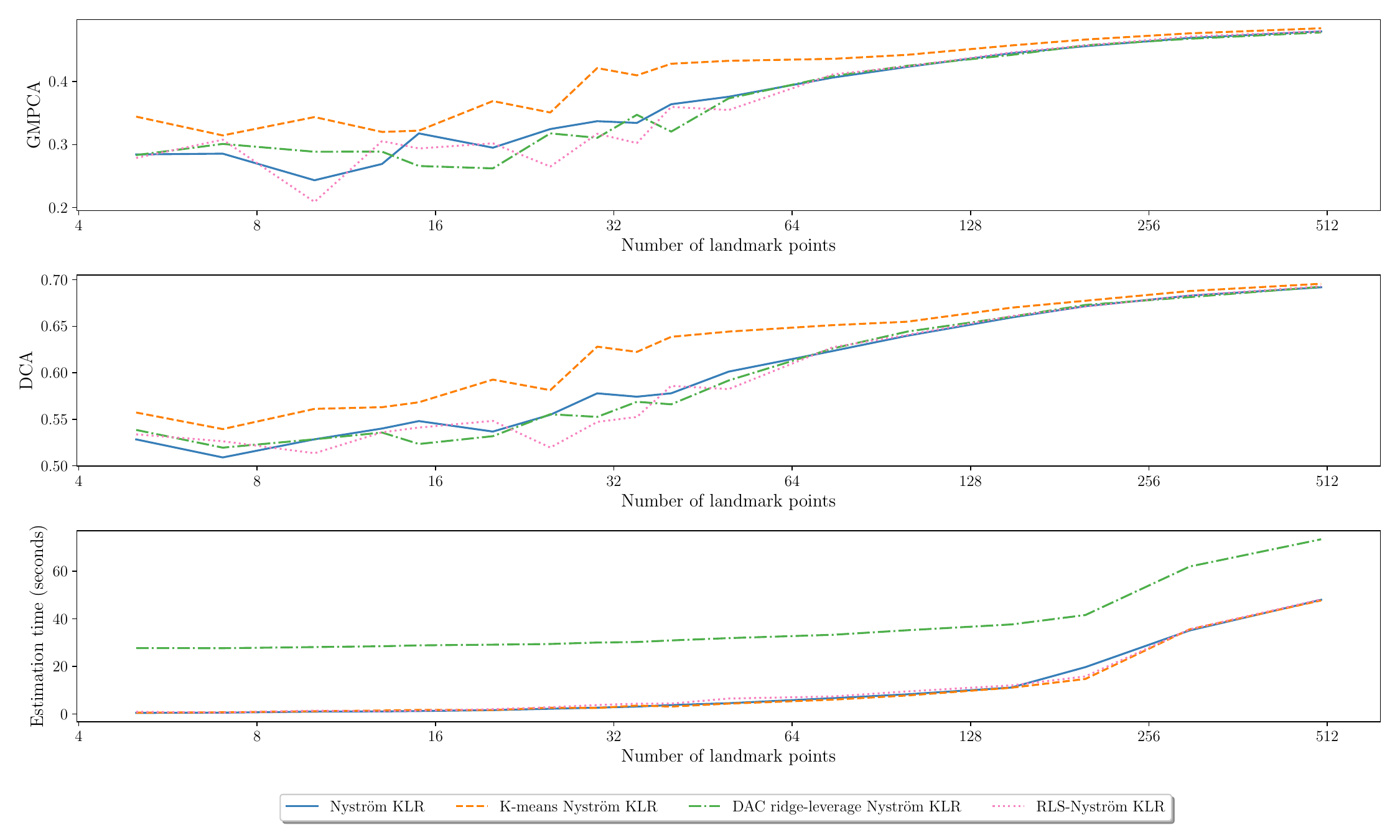}
    \caption{Comparison of Nyström sampling techniques on the NTS dataset}
    \label{fig:NTS_KLR_comp}
\end{figure}

In both datasets, as the number of landmark points in the Nyström method increases, all the methods eventually converge to a similar level of accuracy. When $C$ grows, the standard Nyström \gls{klr} converges to the best rank $C-$approximation $\widehat{\mathbf{K}}_C$ (\cite{KMT12}, Theorem 2)  and in turn converges to the kernel matrix $\mathbf{K}$ (\cite{SZZ15}, Theorem 6). Therefore, since Figures~\ref{fig:LPMC_KLR_comp} and \ref{fig:NTS_KLR_comp}  shows that all Nyström schemes (including random sampling) converge to the same level of accuracy, it is this performance that would be achieved by directly applying \gls{klr}. However, it is important to highlight that using \gls{klr} directly, without Nyström, on these datasets would not be feasible due to the size of the kernel matrix. To illustrate this, storing the training matrices ($70\%$ of the observations) of the LPMC and NTS datasets would require approximately $22$ GB and $194$ GB of RAM memory, respectively. Hence, the adoption of the techniques described here becomes indispensable for handling such large datasets.

Regarding computational costs, all methods give similar performance, except for the DAC ridge-leverage Nyström, which requires more time to compute the leverage indices. The RLS-Nyström method also takes slightly longer than the $k$-means Nyström method, although the difference is minimal.

The analysis of the results suggests that the $k$-means Nyström \gls{klr} method exhibits faster convergence on both datasets, especially when the number of landmark points is small. This can be evidenced by the higher values of the \gls{gmpca} and \gls{dca} indices when using less than $100$ landmark points. Additionally, the results of the NTS dataset indicate that, even with a larger dataset, the $k$-means Nyström method achieves convergence using only a few landmark points. However, as the number of landmark points grows larger, there is little discernible difference between the various techniques, except for a slight overall improvement in \gls{gmpca} and \gls{dca}. Nevertheless, it should be noted that the estimation time of the \gls{klr} model and the size of the kernel matrix will significantly increase when using a large number of landmark points.

To conclude this experiment, it should be noted that the metrics \gls{gmpca} and \gls{dca} behave similarly to the error bound of the parameters ${\cal \varepsilon}(C)$. Looking at Figures~\ref{fig:LPMC_KLR_comp} and \ref{fig:NTS_KLR_comp}, which employ logarithmic scaling on the X-axis, the performance improvement is negligible above $300$ landmark points.

\subsection{Comparison of the Nyström KLR method with other ML algorithms on large datasets}
\label{Sect:Comp_Nystrom_KLR_vs_ML}

Having studied the behaviour of the different Nyström-based \gls{klr} methods, in this experiment, these methods were compared with several state-of-the-art \gls{ml} methods and a \gls{mnl} model. The purpose of this comparison was to demonstrate that the Nyström \gls{klr} method can serve as a viable alternative to other methods for large datasets. For this experiment, a baseline model was established using a \gls{mnl} with linear utility functions, which were estimated using Biogeme.

The results of this experiment are set out in Table~\ref{tab:ML-comp-test-table-Nystrom}. For each implementation of \gls{klr}, a \gls{rbf} kernel was employed, and the model parameters were estimated using the L-BFGS-B algorithm. As seen in Figure~\ref{fig:LPMC_KLR_comp} for the LPMC dataset, the Nyström \gls{klr} with $500$ landmark points returns asymptotic values for the \gls{dca} and \gls{gmpca}, so there is hardly any difference in performance between the models. Hence, $500$ landmark points were used for the Nyström \gls{klr} models. For the NTS dataset, a much larger dataset, Figure~\ref{fig:NTS_KLR_comp} shows that for $500$ landmark points the \gls{dca} and \gls{gmpca} indices still have a small positive slope, therefore $1,000$ landmark points were chosen for the Nyström methods. If the number of landmark points is reduced, it can be seen that the performance of the $k$-means Nyström method would decrease slightly, while the other Nyström methods would experience a more severe decrease.

It is important to note that the estimated size of the training kernel matrices in these models is approximately $0.2$ GB and $1.2$ GB for the LPMC and NTS datasets, respectively. This represents a reduction of $110$ and $160$ times their size, respectively, compared to traditional \gls{klr} models, which makes the application of this model feasible for such large datasets.

Computational time, also referred to as training time, is an important factor to consider. In terms of the time required to estimate the Nyström \gls{klr} models, although it remains substantial, it is significantly lower than that of the linear \gls{svm} model, especially for very large datasets like NTS. The time required is also comparable to that of the Nyström \gls{svm} model. Moreover, with the exception of the DAC ridge-leverage variant, all Nyström \gls{klr} models require less time to estimate than the \gls{mnl} models, which are commonly used by practitioners for large datasets.

Finally, analysing the performance of the models, it can be observed that the ensemble-based methods (such as \gls{xgboost} or \gls{rf}) and neural networks are the best performers in terms of \gls{dca} and \gls{gmpca}. The best-performing model with respect to each index in both datasets is marked in bold, which in this case always corresponds to the \gls{xgboost} model. This finding is in line with other recent studies in the literature \citep{WMH21}.

Regarding the \gls{klr} models, the best \gls{klr} implementation for each dataset is marked with an asterisk `*'. It can be observed that for the LPMC dataset, which is the smallest one, the results of all the Nyström \gls{klr} models are quite close to those of the top-performing \gls{ml} algorithms, which suggests they are promising models. However, by employing sampling strategies, the Nyström-based \gls{klr} models demonstrate improved performance compared to the base \gls{mnl} model. It is noteworthy that, in the NTS dataset, the Nyström \gls{klr} methods yield lower scores in terms of \gls{dca}. Nevertheless, they still provide better approximations to the probability of each alternative, as measured through the \gls{gmpca}, compared to the \gls{mnl} models. 

There are two plausible explanations for the weaker results in the NTS dataset. The first relates to the significant class imbalance observed in the `public transport' and `walk' alternatives. Table~\ref{tab:ML-class-comp-test-table} reports the precision scores for each alternative in the test set. It is notable that for the minority classes in the LPMC and NTS datasets, precision scores are generally low across all methods. Significantly, KLR, along with MNL, SVM, and XGBoost, achieve the highest precision scores for the smallest class in both datasets. Addressing this class imbalance through techniques such as resampling or using different evaluation metrics tailored for imbalanced datasets might yield improved results. Future research should address this problem.

Another possible explanation for the performance difference is that the LPMC dataset focuses on discrete choice modelling, whereas the NTS dataset is more oriented towards \gls{ml} applications, and not all variables are equally important. For instance, the NTS dataset contains a significant number of environmental variables, which might have a reduced impact on the classification process. This issue was highlighted in \citet{MLR23}, where it was calculated the SHAP values for each variable in both datasets, which allowed the variations in variable importance to be analysed. In the LPMC dataset, a wide range of variables significantly influences the decision-maker's choices. Conversely, in the NTS dataset, despite having fewer variables, only a few of them, primarily the `distance', have a notable impact on the classification. The remaining variables are less significant and can impact classifier performance. This accounts for the generally lower effectiveness of all classifiers, particularly the \gls{klr}. This highlights the sensitivity of the \gls{klr} method to non-significant features, suggesting the importance of conducting an attribute analysis prior to applying these techniques.

\begin{table}[ht]
\centering
\caption{Evaluation of DCA and GMPCA indices in the test set and estimation time by model for LPMC and NTS datasets}
\label{tab:ML-comp-test-table-Nystrom}
\resizebox{\textwidth}{!}{
\begin{tabular}{rllllll}
\toprule
{} & \multicolumn{3}{c}{LPMC} & \multicolumn{3}{c}{NTS} \\
\cmidrule(lr){2-4} \cmidrule(lr){5-7} 
{} &    DCA &  GMPCA & Estimation time (s) &    DCA &  GMPCA & Estimation time (s) \\
\midrule
MNL                            &  72.54 &  48.85 &              623.43 &  65.42 &  43.83 &              855.61 \\
SVM                            &  73.61 &  49.92 &              193.48 &  63.22 &  43.34 &              765.59 \\
LinearSVM                      &  72.13 &  48.92 &              691.21 &  64.64 &  43.72 &             3,963.52 \\
RF                             &  73.58 &  50.14 &                2.67 &  68.19 &  46.84 &                1.87 \\
XGBoost                        &  \textbf{74.71} &  \textbf{51.85} &               82.04 &  \textbf{68.72} &  \textbf{48.05} &              138.72 \\
NN                             &  73.87 &  50.72 &                5.25 &  68.40 &  47.12 &                7.51 \\
Nyström KLR                    &  73.45 &  50.41* &              303.39 &  64.98 &  44.53* &              776.46 \\
k-means Nyström KLR            &  73.46 &  50.35 &              309.40 &  65.09* &  44.50 &              719.25 \\
DAC ridge-leverage Nyström KLR &  73.49 &  50.33 &              507.37 &  64.91 &  44.41 &             1,010.26 \\
RLS-Nyström KLR                &  73.62* &  50.43 &              324.85 &  64.81 &  44.52 &              727.24 \\
\bottomrule
\multicolumn{7}{l}{\textit{Note:} For Nyström KLR models, $500$ landmark points were used for the LPMC dataset and $1,000$ for the NTS dataset.}
\end{tabular}
}
\end{table}

\begin{table}
\centering
\caption{Evaluation of precision score for each alternative in the test sets for LPMC and NTS datasets}
\label{tab:ML-class-comp-test-table}
\resizebox{\textwidth}{!}{
\begin{tabular}{rllllllll}
\toprule
{} & \multicolumn{4}{c}{LPMC} & \multicolumn{4}{c}{NTS} \\
\cmidrule(lr){2-5} \cmidrule(lr){6-9} 
{} &  Walk &  Bike & Public Transport &   Car &  Walk &  Bike & Public Transport &   Car \\
\midrule
MNL                            &  0.66 &  0.50 &             0.77 &  0.72 &  0.56 &  0.51 &             0.62 &  0.72 \\
SVM                            &  0.66 &  0.38 &             0.78 &  0.74 &  0.55 &  0.53 &             0.70 &  0.66 \\
LinearSVM                      &  0.67 &  0.00 &             0.78 &  0.70 &  0.54 &  0.49 &             0.00 &  0.71 \\
RF                             &  0.70 &  0.00 &             0.75 &  0.74 &  0.64 &  0.54 &             0.66 &  0.74 \\
XGBoost                        &  0.71 &  0.31 &             0.76 &  0.75 &  0.63 &  0.54 &             0.65 &  0.75 \\
NN                             &  0.70 &  0.00 &             0.76 &  0.74 &  0.60 &  0.54 &             0.62 &  0.75 \\
Nyström KLR                    &  0.67 &  0.23 &             0.76 &  0.74 &  0.60 &  0.51 &             0.71 &  0.70 \\
k-means Nyström KLR            &  0.67 &  0.27 &             0.76 &  0.74 &  0.61 &  0.51 &             0.71 &  0.70 \\
DAC ridge-leverage Nyström KLR &  0.67 &  0.24 &             0.76 &  0.74 &  0.61 &  0.51 &             0.73 &  0.70 \\
RLS-Nyström KLR                &  0.67 &  0.18 &             0.76 &  0.74 &  0.61 &  0.51 &             0.73 &  0.70 \\
\bottomrule
\multicolumn{7}{l}{\textit{Note:} For Nyström KLR models, $500$ landmark points were used for the LPMC dataset and $1,000$ for the NTS dataset.}
\end{tabular}
}
\end{table}

\subsection{Comparison of optimisation techniques for Nyström KLR}
\label{Sect:Comp_optimisation_Nystrom_KLR}

After overcoming the challenge of storing large kernel matrices, the subsequent step involves determining the most efficient optimisation algorithm to estimate the vector of parameters in \gls{klr}. The estimation process involves solving the optimisation problem in Equation~(\ref{eq:04_estimation_KLR_RNLL}), where the loss function is the negative conditional log-likelihood, leading to a procedure called minimising \gls{rnll} or, equivalently, maximising \gls{pmle}. Typically, the algorithms used for solving this \gls{pmle} problem rely on the gradient of the objective function.
From a theoretical point of view, training \gls{mnl} or \gls{klr} models with \gls{pmle} is fundamentally equivalent, as both are instances of the same \acrfull{glm}. However, the differences appear from a computational point of view due to the number of parameters in the models. In the case of \gls{mnl} (in the context of \gls{rum}), the number of parameters increases on the order of ${\cal O}(M \cdot I)$, where $M$ is the dimension of the original input space, i.e., $\mathbf{X}\subset \mathbb{R}^M$, and $I$ is the number of alternatives in the choice set. On the other hand, \gls{klr} grows more abruptly with ${\cal O}(N \cdot I)$, where $N$ is the number of observations. Consequently, training \gls{klr} models poses a computational challenge, as it requires the estimation of a large number of coefficients.

As discussed earlier, the Nyström \gls{klr} model has the advantage of reducing the size of the kernel matrix, which makes the number of parameters to be estimated in the model increase on the order of ${\cal O}(C \cdot I)$, where $C$ is the number of landmark points. This advantageous feature allows the number of required parameters to no longer depend on the size of the dataset but only on the number of landmark points, which is a significant advantage for large datasets. Thus, the number of parameters becomes more similar to that of the \gls{mnl}, making the Nyström approach a promising method for scalable estimation on large datasets.

To compare the different optimisation techniques, it was necessary to first implement these optimisation methods within the \textit{GKLR} Python package. The experimental setup and datasets of the previous section are maintained, and it was selected a $k$-means Nyström \gls{klr} method for this comparison. Specifically, the LPMC dataset was used, with $500$ landmark points, while the NTS dataset employed $1,000$ landmark points. 

The $k$-means Nyström \gls{klr} model was estimated for both datasets with a \gls{pmle} and using the L-BFGS-B, \gls{gd}, Momentum, and Adam optimisation methods. For the L-BFGS-B method, convergence was ensured by setting a tolerance of $10^{-12}$ and a maximum of $10,000$ iterations. This guarantees that the method terminates only when it can no longer explore lower values of the objective function. Regarding the \gls{gd}, Momentum, and Adam methods, they were set to undergo a total of $5,000$ iterations (epochs) with a time-based learning rate decay strategy. Furthermore, a hyperparameter tuning process was conducted for the learning rate and decay rate through grid search within a predefined search space for all optimisation methods.

Figure~\ref{fig:LPMC_optimization_algo_comparison} illustrates the relationship between the objective function value and the required estimation time for each optimisation method using the LPMC dataset. A similar representation is offered in Figure~\ref{fig:NTS_optimization_algo_comparison} for the NTS dataset, which is notably larger in scale. Furthermore, Table~\ref{tab:Optimization_algo_table} provides a comprehensive overview of the final optimisation results for each method.

Upon analysis, it becomes evident that the L-BFGS-B method is the optimal choice for training the $k$-means Nyström \gls{klr} model under the \gls{pmle} framework. Closely behind are the Adam and Momentum methods, in that order. Notably, the \gls{gd} method lags behind the other algorithms even with adjustments to the learning rate or increased computational time. Additional numerical tests, which have not been reported, indicate that the Adam method can achieve the best results when a sufficient number of iterations are performed with a good hyperparameter set. However, it is essential to note that the Adam method takes approximately $4$ to $6$ times longer (in terms of CPU time) than the L-BFGS-B method to achieve and surpass these optimal values.

\begin{figure}[p]
	\centering
    \includegraphics[width=\textwidth]{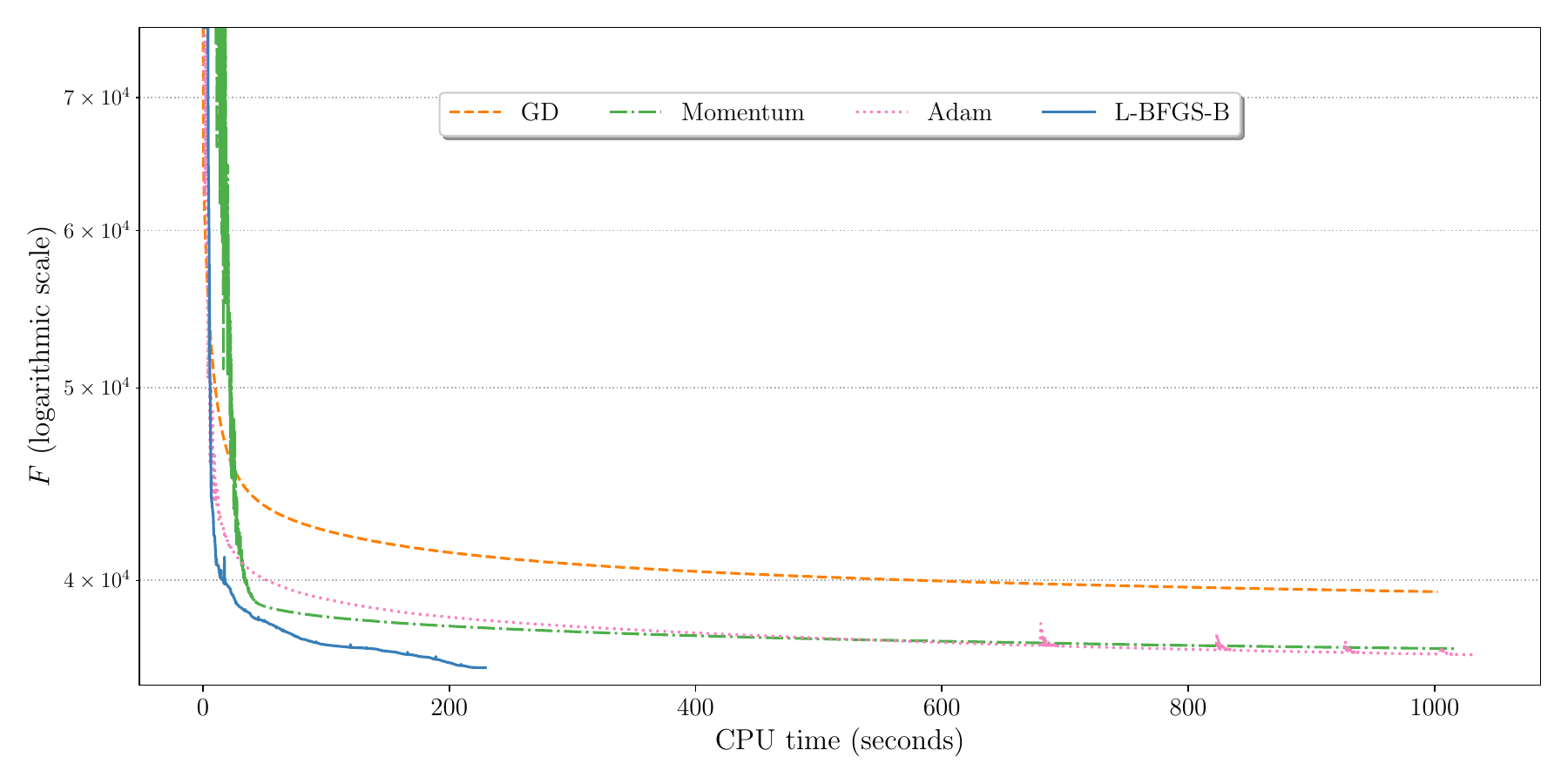}
    \caption{Comparison of optimisation algorithms for KLR on the LPMC dataset}
    \label{fig:LPMC_optimization_algo_comparison}
\end{figure}

\begin{figure}[p]
	\centering
    \includegraphics[width=\textwidth]{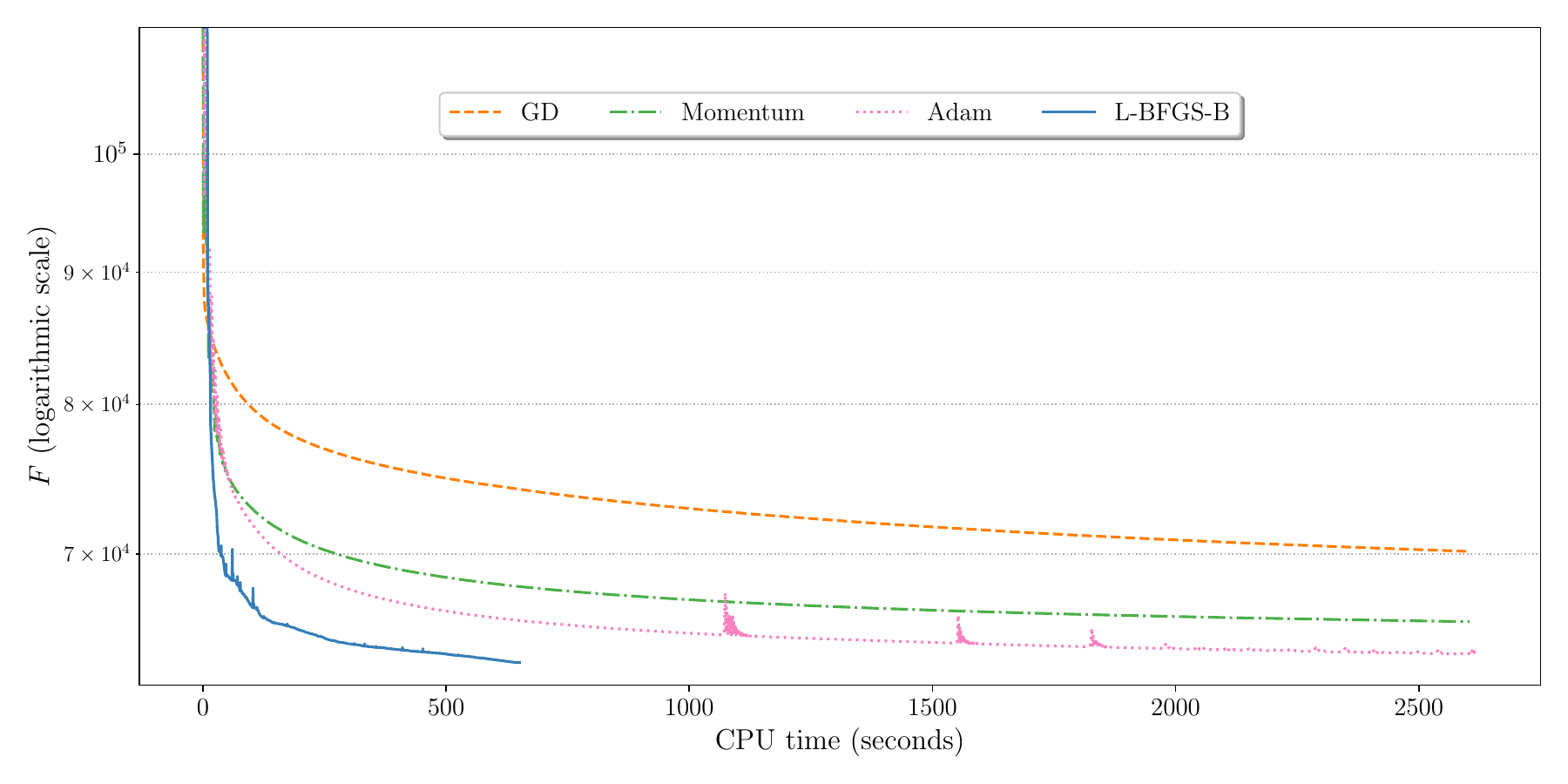}
    \caption{Comparison of optimisation algorithms for KLR on the NTS dataset}
    \label{fig:NTS_optimization_algo_comparison}
\end{figure}

\begin{table}[ht]
\centering
\caption{Optimisation results}
\label{tab:Optimization_algo_table}
\begin{tabular}{rllll}
\toprule
{} & \multicolumn{2}{c}{LPMC} & \multicolumn{2}{c}{NTS} \\
\cmidrule(lr){2-3} \cmidrule(lr){4-5} 
{} & Final Loss & CPU Time (s) & Final Loss & CPU Time (s) \\
\midrule
L-BFGS-B &  \textbf{36,144.43} &       \textbf{229.31} &  \textbf{63,552.09} &       \textbf{650.78} \\
GD       &  39,473.54 &     1,003.81 &  70,173.70 &     2,599.00 \\
Momentum &  36,949.68 &     1,020.09 &  65,908.53 &     2,605.57 \\
Adam     &  36,686.48 &     1,035.54 &  64,020.47 &     2,620.94 \\
\bottomrule
\end{tabular}
\end{table}

An important consideration of this study is that to obtain the graphs presented, it was necessary to report the value of the objective function at every iteration of the estimation methods. However, obtaining the value of the objective function using the \gls{klr} method with large datasets is computationally expensive. Consequently, all estimation methods display CPU times that surpass those that would have been the case if the objective function's value were not reported, or if it were computed only at intervals of $\Delta_t$ iterations. Since this calculation is performed in all methods, it does not pose a problem when comparing them, which is the purpose of this experiment.

Another significant finding of this study is the consequence of Theorem~\ref{th:soluciones}, particularly Remark~\ref{remk:soluciones}. The training problem of \gls{klr} is overspecified in terms of the parameters. To expedite convergence, the traditional approach is to set the parameters for one alternative to zero, i.e., $\boldsymbol{\alpha}^*_{\cdot \cdot} \in \boldsymbol{\Theta}=\left \{\boldsymbol{\alpha }_{\cdot \cdot} \in \mathbb{R}^{N \times I}:  {\alpha}_{\cdot I}=\mathbf{0} \right \}$. However, in this study, the training problem formulation given by Equation~(\ref{eq:04_estimation_KLR_RNLL}) is adopted instead of the commonly used restriction $\boldsymbol{\alpha}_{\cdot \cdot}\in \boldsymbol{\Theta}$. For this reason, if the constraint ${\alpha}_{\cdot I}=\mathbf{0}$ is imposed, a suboptimal solution is obtained.

To demonstrate this, we repeated the previous experiment while enforcing ${\alpha}_{\cdot I}=\mathbf{0}$. Table~\ref{tab:06_Optimization_algo_table_zeros} reports the final loss values from training, the CPU estimation time, and the GMPCA index for the test set. Comparing the loss values with those in Table~\ref{tab:Optimization_algo_table}, it becomes clear that this approach results in a suboptimal solution, as none of the optimisation methods reach the previous performance levels. Additionally, when comparing the GMPCA index on the test set with those obtained for the k-means Nyström \gls{klr} model in Table~\ref{tab:ML-comp-test-table-Nystrom}, we can observe that the performance on the test dataset is also inferior, indicating a reduction in the model’s generalisation capacity.

\begin{table}[h]
\centering
\caption{Optimisation results with $\boldsymbol{\alpha}^*_{\cdot \cdot} \in \boldsymbol{\Theta}=\left \{\boldsymbol{\alpha }_{\cdot \cdot} \in \mathbb{R}^{N \times I} :  {\alpha}_{\cdot I}=\mathbf{0} \right \}$}
\label{tab:06_Optimization_algo_table_zeros}
\begin{tabular}{rllllll}
\toprule
{} & \multicolumn{3}{c}{LPMC} & \multicolumn{3}{c}{NTS} \\
\cmidrule(lr){2-4} \cmidrule(lr){5-7} 
{} & Final Loss &  GMPCA & CPU Time (s) & Final Loss &  GMPCA & CPU Time (s) \\
\midrule
L-BFGS-B &  36,820.31 &  49.97 &       164.10 &  64,291.27 &  44.27 &       493.37 \\
GD       &  40,040.16 &  47.36 &     1,068.42 &  71,838.21 &  40.92 &     2,648.15 \\
Momentum &  37,477.28 &  49.52 &     1,027.60 &  66,550.53 &  43.42 &     2,889.25 \\
Adam     &  37,082.93 &  49.83 &     1,043.43 &  64,404.81 &  44.26 &     2,905.86 \\
\bottomrule
\end{tabular}
\end{table}

\section{Conclusion and future work}
\label{sect:conclusion}

Recent research indicates that \gls{ml} models generally outperform traditional \gls{dcm} in terms of predictive accuracy. However, some \gls{ml} methods, such as ensemble techniques (\gls{xgboost} or \gls{rf}), suffer from limited extrapolation capabilities and encounter numerical issues when calculating derivatives of the model’s probability function. As detailed in \citet{MLR23}, these problems impact the performance of econometric indicators derived from these models, including the calculation of \gls{wtp} or market shares. Additionally, in \citet{MGR21,MartinBaos23}, it was shown that \gls{klr} consistently return accurate choice probabilities and derivatives with respect to the attributes, providing essential information for behavioural analysis. For these reasons, despite not being the top-performing predictive technique, \gls{klr} emerges as a compelling method from a behavioural perspective.

This study has addressed the challenges of applying the \gls{klr} technique to discrete choice modelling with large datasets. Despite inherent complexities posed by memory demands and the large number of parameters in these models, the proposed solution effectively addressed scalability obstacles. The contributions of this research can be summarised into four key points.

Firstly, a theoretical study was conducted to characterise the solution set of the \gls{klr} training problem, identifying that the common practice in the literature of constraining the parameters of the last alternative to zero is suboptimal with respect to the so-called {\sl restricted training problem}. Furthermore, the parameters of \gls{klr} with the Nyström approximation was rigorously analysed, characterising the solution of Nyström \gls{klr}, and establishing the first error bound on learnt parameters.  A pending issue for future developments is the establishment of the bounds for the performance metrics of the Nyström-\gls{klr} model as a function of the error bounds on the parameters.

Secondly, the study assessed the feasibility of using the Nyström approach within \gls{klr} to analyse large datasets. This technique significantly reduced kernel matrix size, while upholding the performance of the model. In the datasets used for this comparison, performance superior to that of traditional discrete choice models was achieved, such as \gls{mnl}, while reducing the size of the kernel matrices by a factor of between $110$ and $160$, compared to the original \gls{klr} model. A comprehensive comparison of four Nyström methods demonstrated their potential for efficiently handling datasets exceeding $200,000$ observations. These four methods encompass distinct sampling approaches, including a uniform, $k$-means, and two non-uniform strategies based on leverage scores. They were tested on a multi-categorical classification problem within the domain of transport mode choice, using large datasets. The $k$-means Nyström strategy emerged as particularly effective, particularly when a modest number of landmark points is used to approximate the kernel matrix. However, it was found that when the number of landmark points increases, all strategies converge to the same trend. In \citet{KMT12}, more specifically in Theorem 2, it is shown that  standard Nyström method converge  to the best rank-$C$ approximation of Nyström when the number of  landmarks $C$ grows and therefore the graph shows that they all converge to the best rank $C-$ Nyström approximation.

Thirdly, the performance of Nyström-based strategies was benchmarked against traditional methods commonly applied in the transport field, such as \gls{mnl} and state-of-the-art \gls{ml} techniques. The results concluded that while effective, the Nyström \gls{klr} models did not surpass the superior performance of methods like \gls{xgboost}. However, as indicated previously, the top-performing methods, like ensemble approaches, often struggle in estimating the derivatives of probabilities, which is crucial for extracting econometric indicators. The \gls{klr} approach allows the derivatives of the probabilities to be estimated accurately, and this is what motivates further research in the direction of obtaining econometric information using these models. 

Fourthly, it was found that the results of the Nyström \gls{klr} model depend to a large extent on the correct estimation of their parameters and hyperparameters. Therefore, this research looks at a number of optimisation techniques for Nyström \gls{klr} model estimation and finds that the L-BFGS-B method emerges as the most effective choice, closely followed by the Adam method (which is a commonly-used technique for training \glspl{dnn}).

Promising avenues for further work include exploring the capabilities of the mini-batch \gls{sgd} for the estimation of the Nyström \gls{klr} model. This method offers significant advantages, including accelerated convergence and enhanced generalisation performance, as evidenced in the existing literature \citep{BCN18}. However, integrating mini-batch \gls{sgd} into the current \textit{GKLR} package (\url{https://github.com/JoseAngelMartinB/gklr}) proves challenging for large datasets as the gradient computation requires calculating the Kronecker product for a subset of the kernel matrix (i.e., the mini-batch). In the current implementation of the package, this operation dominates the computational cost. Consequently, employing mini-batches entails multiple iterations of the Kronecker product, thereby escalating computational demands proportionally. Addressing this limitation requires optimising the estimation module of the \textit{GKLR} package to take advantage of the parallel processing capabilities of the CPU to overcome this barrier.

Lastly, datasets applied in discrete choice modelling, such as survey datasets, might exhibit imbalances due to unequal representation of alternatives in the population. Although not reported, the numerical tests also show that the \gls{ml} methods achieve lower precision scores on the less represented classes. Addressing these imbalances is another avenue for future research. Exploring data balancing techniques and investigating their impact on model performance holds promise for further refinement. In essence, this study contributes valuable insights into the effectiveness and limitations of Nyström \gls{klr} in handling large datasets, paving the way for continued advancements in the field.

\section*{Acknowledgements}
This work was supported by grant PID2020-112967GB-C32 funded by MICIU/AEI/10.13039/501100011033 and by {\sl ERDF A way of making Europe}. Additional funding was provided by grant 2022-GRIN-34249 from the University of Castilla-La Mancha and ERDF. We would also like to express our gratitude to the Transport and Mobility Laboratory at EPFL for their support of this work during the research stay of Mart\'in-Baos.

\section*{Declaration of Generative AI and AI-assisted technologies in the writing process}
During the preparation of this work, the authors used OpenAI GPT-3.5 and Grammarly to aid in generating and refining natural language text for different sections of the manuscript. After using this tool, the authors reviewed and edited the content as needed and take full responsibility for the content of the publication.


\appendix

\section{Analysis of the KLR model overspecification under RNLL}

\label{App:overspecification}

This appendix is intended to show the sources of overspecification of \gls{rnll}.

\begin{lemma}
\label{App:lema1}
The function $\mathbf{h}_{\cdot i} (\boldsymbol{\upsilon}_{\cdot \cdot})= \mathbf{p}_{\cdot i} (\boldsymbol{\alpha}_{\cdot \cdot}+\boldsymbol{\upsilon}_{\cdot \cdot})$ remains constant over the vector space $\boldsymbol{\upsilon}_{\cdot \cdot}\in V+L(\mathbf{K})^\perp$ where:
$$
V= \left \{ \mathbf{1}^\top \otimes \boldsymbol{\beta}:  \boldsymbol{\beta} \in \mathbb{R}^{N \times 1}\right \}
$$
\end{lemma}

\begin{proof}
Let $\boldsymbol{\alpha}_{\cdot \cdot } \in \mathbb{R} ^{N \times I}$ and $\boldsymbol{\alpha}^*_{\cdot \cdot }=\boldsymbol{\alpha}_{\cdot \cdot }+\mathbf{1}^\top \otimes \boldsymbol{\beta}+ \boldsymbol{\alpha}^\perp_{\cdot \cdot } $ where $\boldsymbol{\alpha}^\perp_{\cdot \cdot } \in L(\mathbf{K})^\perp$, then,
\begin{align}
    \nonumber
    p_{ni} \left ( \boldsymbol{\alpha}^*_{\cdot \cdot } \right ) 
    &= \frac{ \exp \left (  \mathbf{k}_n^\top  (\boldsymbol{\alpha}_{\cdot i}+\boldsymbol{\beta}+\boldsymbol{\alpha}^\perp_{\cdot i}) \right) }{\sum_{j=1}^{I}\exp\left ( \mathbf{k}_n^\top (\boldsymbol{\alpha}_{\cdot j}+\boldsymbol{\beta} + \boldsymbol{\alpha}^\perp_{\cdot j}) \right)}
     = \frac{ \exp \left (  \mathbf{k}_n^\top  \boldsymbol{\alpha}_{\cdot i} \right) \exp \left (  \mathbf{k}_n^\top  \boldsymbol{\beta} \right) \exp(0)}{\sum_{j=1}^{I}\exp\left ( \mathbf{k}_n^\top \boldsymbol{\alpha}_{\cdot j}  \right) \exp \left (  \mathbf{k}_n^\top  \boldsymbol{\beta} \right) \exp(0)}
     \\
     \label{eq:04_eq1_plus_b}
    &= \frac{ \exp \left (  \mathbf{k}_n^\top  \boldsymbol{\alpha}_{\cdot i} \right) }{\sum_{j=1}^{I}\exp\left ( \mathbf{k}_n^\top \boldsymbol{\alpha}_{\cdot j} \right)}= p_{ni} \left ( \boldsymbol{\alpha}_{\cdot \cdot} \right ) 
\end{align}
\end{proof}
\begin{remark}
Applying the \gls{mle} method to \gls{klr} (by setting $\lambda=0$), it will not be possible to univocally determine the function $f_i(\x)$, since there are infinite functions $f^*_i(\x)$ leading to the same probability functions $\mathbf{p}_{\cdot i}(\boldsymbol{\alpha}_{\cdot \cdot })$ and, therefore, the log-likelihood value remains constant. This leads to the overspecification of the \gls{mle} problem, resulting in an infinite number of optimal solutions. To address this, a potential solution involves setting the parameter value for the last alternative, $I$, to $\mathbf{0}$. This choice is equivalent to setting $\boldsymbol{\beta}=-\boldsymbol{\alpha}_I$ yields another optimal \gls{mle} solution, where the parameter vector of the last alternative is $\mathbf{0}$. 
Setting $\boldsymbol{\alpha}_{\cdot I}=\mathbf{0}$ does not compromise the quality of the solution of the training problem when $\lambda=0$ but it does affect the quality of the solution when $\lambda >0$.
\end{remark}

The proof to Theorem \ref{th:soluciones} is given as follows.

\begin{proof}
Let $\boldsymbol{\alpha}^*_{\cdot \cdot}$ be a solution to the training problem specified in Equation~(\ref{eq:04_estimation_KLR_RNLL}). Then, this solution can be written as:
\begin{equation}
\boldsymbol{\alpha}^*_{\cdot \cdot}= {\mathcal P}^\perp_{L(\mathbf{K})}
\left ( \boldsymbol{\alpha}^*_{\cdot \cdot} \right ) +
{\mathcal P}^\perp_{L(\mathbf{K})^\perp}
\left ( \boldsymbol{\alpha}^*_{\cdot \cdot} \right )=
\boldsymbol{\aleph}^*_{\cdot \cdot} + \boldsymbol{\eta}_{\cdot \cdot}
\end{equation} 
As $\boldsymbol{\eta}_{\cdot \cdot}=(\mathbf{1}^\top \otimes \mathbf{0})+\boldsymbol{\eta}_{\cdot \cdot} \in V+ L(\mathbf{K})^\perp$, applying the Lemma \ref{App:lema1} gives:
\begin{eqnarray}
\label{eq:termino1}
\mathbf{p}_{\cdot i}\left (\boldsymbol{\alpha}^*_{\cdot \cdot} \right )=
\mathbf{p}_{\cdot i }\left (\boldsymbol{\aleph}^*_{\cdot \cdot} \right )
\end{eqnarray}

On the other hand, since $\boldsymbol{\eta}_{\cdot \cdot} \in L(\mathbf{K})^\perp$:
\begin{equation}
\label{eq:ortogonalidad}
\mathbf{K} \, \boldsymbol{\eta}_{\cdot \cdot}=\mathbf{0} \Rightarrow \mathbf{K} \, \boldsymbol{\alpha}^*_{\cdot \cdot}=
\mathbf{K} \, \boldsymbol{\aleph}^*_{\cdot \cdot}
\end{equation}
\noindent which ensures that the regularisation term satisfies:
\begin{equation}
\label{eq:termino2}
\frac{\lambda}{2} \sum_{i=1}^I \boldsymbol{\alpha}_{\cdot i} ^\top \mathbf{K}\boldsymbol{\alpha}_{\cdot i}=
\frac{\lambda}{2} \sum_{i=1}^I \boldsymbol{\aleph}_{\cdot i} ^\top \mathbf{K}\boldsymbol{\aleph}_{\cdot i}
\end{equation}
From Equations (\ref{eq:termino1}) and (\ref{eq:termino2}) it can be obtained that:
\begin{equation}
\label{eq:termino3}
\mathcal{L} (\boldsymbol{\alpha}^*_{\cdot \cdot})=\mathcal{L} (\boldsymbol{\aleph}^*_{\cdot \cdot})
\end{equation}
Furthermore, Equation~(\ref{eq:termino3}) guarantees that $\boldsymbol{\aleph}^*_{\cdot \cdot}$ is a minimiser of: 
\begin{equation}
\label{eq:training_breve}
\min_{\boldsymbol{\alpha}_{\cdot \cdot} \in \mathbb{R}^{N \times I}} \mathcal{L}(\boldsymbol{\alpha}_{\cdot \cdot})
\end{equation}
Because $\mathbb{R}^{N \times 1}= $$L(\mathbf{K}) \oplus L(\mathbf{K})^\perp$, it holds that:
\begin{equation}
\label{eq:termino4}
\min_{\boldsymbol{\alpha}_{\cdot \cdot} \in L(\mathbf{K})} \mathcal{L}(\boldsymbol{\alpha}_{\cdot \cdot})
\ge 
\min_{\boldsymbol{\alpha}_{\cdot \cdot} \in \mathbb{R}^{N \times I}} \mathcal{L}(\boldsymbol{\alpha}_{\cdot \cdot})
\end{equation}

From the fact that  $\boldsymbol{\aleph}^*_{\cdot \cdot}\in L(\mathbf{K})$, and from the optimality of $\boldsymbol{\aleph}^*_{\cdot \cdot}$ for the problem  (\ref{eq:training_breve}), Equation~(\ref{eq:termino4})  leads to the optimality of $\boldsymbol{\aleph}^*_{\cdot \cdot}$ for the restricted problem $\min_{\boldsymbol{\alpha}_{\cdot \cdot} \in L(\mathbf{K})} \mathcal{L}(\boldsymbol{\alpha}_{\cdot \cdot}).$

To prove that the restricted problem has a unique solution, consider the decomposition of \gls{svd} of the kernel matrix:
$$
\mathbf{K}= \mathbf{U} \boldsymbol{\Sigma} \mathbf{U}^\top
$$
where $\mathbf{U}, \boldsymbol{\Sigma} $ are an $N\times R$ column-orthogonal matrix and  an $R\times R$ diagonal matrix with nonnegative entries, respectively.  Assuming that   $\hbox{rank}(\mathbf{K})=R$ then $\mathbf{K}$ has $\sigma_1\ge \cdots \ge\sigma_R>0$ eigenvalues  and $\boldsymbol{\Sigma}=\hbox{diag}(\sigma_1\cdots \sigma_R)$.

As the  column vectors of $\mathbf{U}$ form a basis of the vector space $L(\mathbf{K})$:
\begin{equation}
\boldsymbol{\alpha}_{\cdot\cdot} \in L(\mathbf{K}) \Leftrightarrow \boldsymbol{\alpha}_{\cdot\cdot} = \mathbf{U} \boldsymbol{\Sigma}^{-\frac{1}{2}} \boldsymbol{\beta}_{\cdot \cdot}
\end{equation}
where $\boldsymbol{\Sigma}^{-\frac{1}{2}}=\hbox{diag}(\sigma^{-1/2}_1\cdots \sigma^{-1/2}_R)$ and $\boldsymbol{\beta}_{\cdot \cdot} \in \mathbb{R}^{R \times I }$.

Now, the restricted training problem is expressed in terms of the new variable $\boldsymbol{\beta}_{\cdot \cdot}$ as:
\begin{equation}
\boldsymbol{\alpha}^\top _{\cdot i} \mathbf{K} \boldsymbol{\alpha}_{\cdot i}=
\boldsymbol{\beta}^\top _{\cdot i} \boldsymbol{\Sigma}^{-\frac{1}{2}}  \mathbf{U}^\top \mathbf{K}  \mathbf{U} \boldsymbol{\Sigma}^{-\frac{1}{2}} \boldsymbol{\beta}_{\cdot i}=
\boldsymbol{\beta}^\top _{\cdot i} \boldsymbol{\Sigma}^{-\frac{1}{2}} \mathbf{U}^\top   \mathbf{U} \boldsymbol{\Sigma} \mathbf{U}^\top  \mathbf{U} \boldsymbol{\Sigma}^{-\frac{1}{2}} \boldsymbol{\beta}_{\cdot i}=
\boldsymbol{\beta}^\top _{\cdot i} \mathbb{I}_R \boldsymbol{\beta}_{\cdot i}
\end{equation}
where $\mathbb{I}_R$ is the identity matrix of order $R$. As the identity matrix is positive definite this shows that the penalisation term for the restricted training problem is a strictly convex function.
 
On the other hand, the latent functions can be written as:
\begin{equation}
\mathbf{f}_{\cdot \cdot}= \mathbf{K} \boldsymbol{\alpha}_{\cdot \cdot}=
 \mathbf{K}  \mathbf{U} \boldsymbol{\Sigma}^{-\frac{1}{2}} \boldsymbol{\beta}_{\cdot \cdot}=  \mathbf{U} \boldsymbol{\Sigma} \mathbf{U}^\top \mathbf{U} \boldsymbol{\Sigma}^{-\frac{1}{2}} \boldsymbol{\beta}_{\cdot \cdot}=
  \mathbf{U} \boldsymbol{\Sigma}^{\frac{1}{2}} \boldsymbol{\beta}_{\cdot \cdot}
\end{equation}
Letting $\mathbf{B}= \mathbf{U} \boldsymbol{\Sigma}^{\frac{1}{2}}  \in \mathbb{R}^{ N \times I}$ gives:
\begin{equation}
p_{ni}(\boldsymbol{\alpha}_{\cdot \cdot})=
\frac{\exp(\mathbf{B}_{n \cdot}\boldsymbol{\beta}_{\cdot i })}{\sum_{j=1}^I \exp(\mathbf{B}_{n \cdot}\boldsymbol{\beta}_{\cdot j })} =p_{ni}(\boldsymbol{\beta}_{\cdot \cdot})
\end{equation}

\begin{equation}
 {\mathcal L}(\boldsymbol{\beta}_{\cdot \cdot})=-\frac{1}{N} \sum _{n=1}^N \sum_{i=1}^I y_{ni} \log p_{ni}(\boldsymbol{\beta}_{\cdot \cdot })+ \frac{\lambda}{2} \sum_{i=1}^I \boldsymbol{\beta}_{\cdot i} ^\top \mathbb{I}_R\boldsymbol{\beta}_{\cdot i},
\end{equation}

If $\lambda>0$, it is observed that the objective function of the restricted training problem ${\mathcal L}(\boldsymbol{\beta}_{\cdot \cdot})$ can be seen as the sum of two main components: (i) the negative log-likelihood and (ii) a penalty term. Notably, \citet{DaK93} and \citet{HZM16} demonstrated the convexity of component (i) within the broader context of the \gls{nl} model, which encompasses \gls{mnl}. Furthermore, the term (ii) also exhibits strictly convex characteristics due to the positive definiteness inherent in the identity matrix $\mathbb{I}_R$ and  ${\mathcal L}(\boldsymbol{\beta}_{\cdot \cdot})$ becomes a strictly convex function. As a result, the restricted training problem has a unique solution. Since it has been proved that $\boldsymbol{\aleph}^*_{\cdot \cdot}$ is a minimiser of the restricted training problem, this is the only solution of the problem.

Equation (\ref{eq:KKT}) is derived by applying matrix differentiation rules and can be found in the doctoral thesis of one of the authors \citep{MartinBaos23}. Moreover, a derivation for the case of binary classification problems can also be found in \cite{OuA18}.
\end{proof}

\section{Error bounds for Nyström KLR}
\label{sec:demostracion_cotas}

\begin{lemma}
\label{lemma1}
Let $\mathbf{p}_{\cdot i}(\boldsymbol{\alpha}_{\cdot \cdot})$ be the probabilities vector computed for the kernel matrix $\mathbf{K}$, and let $\widehat{\mathbf{p}}_{\cdot i}(\boldsymbol{\alpha}_{\cdot \cdot})$ be the same probabilities vector found by using the approximation $\widehat{\mathbf{K}}$. Then:
\begin{equation}
\widehat{\mathbf{p}}_{\cdot i}(\boldsymbol{\alpha}_{\cdot \cdot})=\mathbf{p}_{\cdot i}(\boldsymbol{\alpha}_{\cdot \cdot})- \mathbf{p}_{\cdot \cdot} (\boldsymbol{\alpha}_{\cdot \cdot})\left [ \boldsymbol{\alpha}_{\cdot \cdot} ^\top (\widehat{\mathbf{K}}-\mathbf{K})\right ] \mathbf{p}_{\cdot i}(\boldsymbol{\alpha}_{\cdot \cdot})+ \boldsymbol{\varepsilon}
\end{equation}
where $\| \boldsymbol{\varepsilon}\|_2= o(\| (\widehat{\mathbf{K}}-\mathbf{K})^\top \boldsymbol{\alpha}_{\cdot \cdot }\|_{2,\infty})$.
\end{lemma}
\begin{proof}
Consider the softmax function:
\begin{equation}
S_i(\mathbf{f})=\dfrac{\exp(\mathbf{f}_i)}{\sum_{j=1}^I \exp(\mathbf{f}_j)}
\end{equation}

This function is expressed by its first-order Taylor polynomial expanded at $\mathbf{f}$:
\begin{equation}
\label{eq:taylor}
S_i(\widehat{\mathbf{f}})=S_i(\mathbf{f}) -S_i(\mathbf{f}) \sum_{j=1}^I S_j(\mathbf{f}) \left [\widehat{\mathbf{f}}_j-\mathbf{f}_j \right ]+ o(\|\widehat{\mathbf{f}}-\mathbf{f}\|_ 2)
\end{equation} 

Letting $\widehat{\mathbf{f}}_{nj}=\widehat{\mathbf{k}}_n^\top \boldsymbol{\alpha}_{\cdot j}$ and $\mathbf{f}_{nj}=\mathbf{k}_n ^\top \boldsymbol{\alpha}_{\cdot j}$, then $\widehat{\mathbf{f}}_{n \cdot }=  \widehat{\mathbf{k}}_n^\top \boldsymbol{\alpha}_{\cdot \cdot }  $ and $\mathbf{f}_{n \cdot }=  \mathbf{k}_n^\top \boldsymbol{\alpha}_{\cdot \cdot }  $.

The probability that user $n$ selects alternative $i$, computed either exactly or approximately using the feature vector, is given respectively by:
\begin{eqnarray*}
p_{ni}= S_i(\mathbf{f}_{n \cdot })=S_i(\mathbf{k}_n^\top \boldsymbol{\alpha}_{\cdot \cdot })\\
\widehat{p}_{ni}= S_i(\widehat{\mathbf{f}}_{n \cdot })=S_i(\widehat{\mathbf{k}}_n^\top \boldsymbol{\alpha}_{\cdot \cdot } )
\end{eqnarray*}

Applying Taylor's expansion to Equation~(\ref{eq:taylor}) gives:
\begin{equation}
\label{eq:LinealizacionProb}
\widehat{p}_{ni}=p_{ni}-p_{ni}\sum_{j=1}^I p_{nj} [\widehat{\mathbf{k}}_n^\top \boldsymbol{\alpha}_{\cdot j}-\mathbf{k}_n ^\top \boldsymbol{\alpha}_{\cdot j}]+o(\|(\widehat{\mathbf{k}}_n-\mathbf{k}_n)^\top \boldsymbol{\alpha}_{\cdot \cdot }\|_2)
\end{equation}
Rewriting Equation~(\ref{eq:LinealizacionProb}) using matrices:
\begin{equation}
\widehat{p}_{ni}=p_{ni}-p_{ni}\,\mathbf{p}_{n\cdot} \left [ \boldsymbol{\alpha}_{\cdot \cdot} ^\top (\widehat{\mathbf{k}}_n-\mathbf{k}_n)\right ]+o(\|(\widehat{\mathbf{k}}_n-\mathbf{k}_n)^\top \boldsymbol{\alpha}_{\cdot \cdot }\|_2)
\end{equation}
Then, expanding the above expression with respect to $n$:
\begin{equation}
\widehat{\mathbf{p}}_{\cdot i}=\mathbf{p}_{\cdot i}- \mathbf{p}_{\cdot \cdot} \left [ \boldsymbol{\alpha}_{\cdot \cdot} ^\top (\widehat{\mathbf{K}}-\mathbf{K})\right ] \mathbf{p}_{\cdot i}+ \boldsymbol{\varepsilon}
\end{equation}
where for the error vector $\boldsymbol{\varepsilon}$, all components satisfy  $\boldsymbol{\varepsilon}_n=o(\|(\widehat{\mathbf{k}}_n-\mathbf{k}_n)^\top \boldsymbol{\alpha}_{\cdot \cdot }\|_2)$, therefore
$\boldsymbol{\varepsilon}_n=o \left (\underset{n}{\max}{\|(\widehat{\mathbf{k}}_n-\mathbf{k}_n)^\top \boldsymbol{\alpha}_{\cdot \cdot }}\|_2 \right )=o(\|(\widehat{\mathbf{K}}-\mathbf{K})^\top \boldsymbol{\alpha}_{\cdot \cdot }\|_{2,\infty})$, and it follows that $\| \boldsymbol{\varepsilon}\|_2= o(\| (\widehat{\mathbf{K}}-\mathbf{K})^\top \boldsymbol{\alpha}_{\cdot \cdot }\|_{2,\infty})$.
\end{proof}

The proof of Theorem \ref{th:cota1} is given as follows.

\begin{proof}
By hypothesis, $\widehat{\mathbf{K}}$ is a Nyström approximation associated to the sketching matrix $\mathbf{P}$, giving $\mathbf{C}=\mathbf{K}\mathbf{P}.$ It holds that  $L(\mathbf{C})=L(\widehat{\mathbf{K}})$ and so, this linear space can be interchangeably referred to in the proof.

Let $\boldsymbol{\aleph}^*_{\cdot\cdot}$ be a solution, Equation~(\ref{eq:04_estimation_KLR_RNLL}).  It then satisfies the optimality conditions:
\begin{eqnarray}
\label{eq:teo2_KKT_1}
\mathbf{K} \underset{\mathbf{z}_i}{\underbrace{\left [ N \lambda \boldsymbol{\aleph}^*_{\cdot i}   + \mathbf{p}_{\cdot i}(\boldsymbol{\aleph}^*_{\cdot \cdot}) -\boldsymbol{y}_{\cdot i} \right ]} }=\mathbf{0}, \hspace{1cm} i=1,\cdots, I.
\end{eqnarray}

Equation~(\ref{eq:teo2_KKT_1}) shows that $\mathbf{K} \mathbf{z}_i=\mathbf{0} \Rightarrow  \mathbf{z}_i \in L(\mathbf{K})^\perp$. Moreover, by hypothesis, $\widehat{\mathbf{K}}$ is a Nyström approximation and, therefore, $L(\widehat{\mathbf{K}}) \subseteq L(\mathbf{K})$ from which it follows that $L(\mathbf{K})^\perp\subseteq  L(\widehat{\mathbf{K}}) ^\perp $ and therefore $\mathbf{z}_i \in L(\widehat{\mathbf{K}}) ^\perp$. This leads to $\widehat{\mathbf{K}}\mathbf{z}_i=\mathbf{0}$, that is:

\begin{eqnarray}
\label{eq:teo2_KKT_2}
\widehat{\mathbf{K}} \left [ N \lambda \boldsymbol{\aleph}^*_{\cdot i}   + \mathbf{p}_{\cdot i}(\boldsymbol{\aleph}^*_{\cdot \cdot}) - \boldsymbol{y}_{\cdot i} \right ]=\mathbf{0}, \hspace{1cm} i=1,\cdots, I.
\end{eqnarray}

This solution is uniquely expressed as:
\begin{equation}
\boldsymbol{\aleph}^*_{\cdot i}={\mathcal P}^\perp_{L(\widehat{\mathbf{K}})} (\boldsymbol{\aleph}^*_{\cdot i})+{\mathcal P}^\perp_{L(\widehat{\mathbf{K}})^\perp} (\boldsymbol{\aleph}^*_{\cdot i})
\end{equation} 
Letting ${\widehat{\boldsymbol{\chi}}^*_{\cdot i}}={\mathcal P}^\perp_{L(\widehat{\mathbf{K}})} (\boldsymbol{\aleph}^*_{\cdot i})$ 
 and ${\widehat{\boldsymbol{\chi}}^{*\perp}_{\cdot i}}={\mathcal P}^\perp_{L(\widehat{\mathbf{K}})^\perp} (\boldsymbol{\aleph}^*_{\cdot i})$. Now establish:

\begin{equation}
\label{eq:clave}
\widehat{\boldsymbol{\chi}}^{*\top}_{\cdot \cdot } [\mathbf{K}- \widehat{\mathbf{K}}]=\mathbf{0}
\end{equation}

The following relationship (refer to Equation~(\ref{eq:residuo}) for its derivation)
\begin{equation}
\label{eq:proyeccion_nystrom}
\mathbf{K}- \widehat{\mathbf{K}}= \mathbf{K}^{1/2} \left [\mathbb{I}- \boldsymbol{\Pi}^\perp _{L(\mathbf{K}^{1/2} \mathbf{P}) }\right ]\mathbf{K}^{1/2}
\end{equation}
allows Equation~(\ref{eq:clave}) to be written as: 
\begin{equation}
\widehat{\boldsymbol{\chi}}^{*\top}_{\cdot \cdot } [\mathbf{K}- \widehat{\mathbf{K}}]=
\widehat{\boldsymbol{\chi}}^{*\top}_{\cdot \cdot } \mathbf{K}^{1/2} \mathbf{A}
\end{equation}
where $\mathbf{A}= \left [\mathbb{I}- \boldsymbol{\Pi}^\perp _{L(\mathbf{K}^{1/2} \mathbf{P}) }\right ]\mathbf{K}^{1/2} \subseteq L(\mathbf{K}^{1/2} \mathbf{P})^\perp $. On the other hand, $\mathbf{K}^{1/2}=\mathcal{P}^\perp_{L(\mathbf{K}^{1/2} \mathbf{P})}(\mathbf{K}^{1/2})+\mathcal{P}^\perp_{L(\mathbf{K}^{1/2} \mathbf{P})^\perp}(\mathbf{K}^{1/2})$, and so rewriting: 
\begin{eqnarray*}
\widehat{\boldsymbol{\chi}}^{*\top}_{\cdot \cdot } [\mathbf{K}- \widehat{\mathbf{K}}]=
\widehat{\boldsymbol{\chi}}^{*\top}_{\cdot \cdot } 
\left [\mathcal{P}^\perp_{L(\mathbf{K}^{1/2} \mathbf{P})}(\mathbf{K}^{1/2})+\mathcal{P}^\perp_{L(\mathbf{K}^{1/2} \mathbf{P})^\perp}(\mathbf{K}^{1/2})
\right ] \mathbf{A}=\\
\widehat{\boldsymbol{\chi}}^{*\top}_{\cdot \cdot } 
\underset{\mathbf{=0 (*)}}{\underbrace{\mathcal{P}_{L(\mathbf{K}^{1/2} \mathbf{P})}(\mathbf{K}^{1/2}) \mathbf{A}}}+
\underset{\mathbf{=0 (**)}}{\underbrace{
\widehat{\boldsymbol{\chi}}^{*\top}_{\cdot \cdot } 
\mathcal{P}_{L(\mathbf{K}^{1/2} \mathbf{P})^\perp}(\mathbf{K}^{1/2})
}}
 \mathbf{A}=\mathbf{0}
\end{eqnarray*}

The relationship (*) is satisfied by the definition of $\mathbf{A}$. Now consider the relationship $(**)$. Given $\widehat{\boldsymbol{\chi}}^{*}_{\cdot \cdot } \in L(\mathbf{KP}) \Rightarrow  
\widehat{\boldsymbol{\chi}}^{*}_{\cdot \cdot }= \mathbf{KPD}
$ with $\mathbf{D} \in \mathbb{R}^{C \times I}$, then:
\begin{equation}
    \widehat{\boldsymbol{\chi}}^{*\top}_{\cdot \cdot }= \mathbf{D}^\top \mathbf{P}^\top\mathbf{K}= \mathbf{D}^\top \mathbf{P}^\top\mathbf{K}^{1/2}\mathbf{K}^{1/2}=
     \mathbf{D}^\top \mathbf{P}^\top\mathbf{K}^{1/2}
     \left [
     \mathcal{P}^\perp_{L(\mathbf{K}^{1/2} \mathbf{P})}(\mathbf{K}^{1/2})+\mathcal{P}^\perp_{L(\mathbf{K}^{1/2} \mathbf{P})^\perp}(\mathbf{K}^{1/2})
    \right  ]
\end{equation}
\noindent from which is obtained:
\begin{align}
\widehat{\boldsymbol{\chi}}^{*\top}_{\cdot \cdot } 
\mathcal{P}_{L(\mathbf{K}^{1/2} \mathbf{P})^\perp}(\mathbf{K}^{1/2})
&=\mathbf{D}^\top \mathbf{P}^\top\mathbf{K}^{1/2}
     \left [
     \mathcal{P}^\perp_{L(\mathbf{K}^{1/2} \mathbf{P})}(\mathbf{K}^{1/2})+\mathcal{P}^\perp_{L(\mathbf{K}^{1/2} \mathbf{P})^\perp}(\mathbf{K}^{1/2})
    \right  ]
    \mathcal{P}^\perp_{L(\mathbf{K}^{1/2} \mathbf{P})^\perp}(\mathbf{K}^{1/2})\\
    &=
\mathbf{D}^\top \mathbf{P}^\top\mathbf{K}^{1/2}
     \left [
    \mathcal{P}^\perp_{L(\mathbf{K}^{1/2} \mathbf{P})^\perp}(\mathbf{K}^{1/2})
    \right ]^2=\mathbf{0}
\end{align}

Using the Taylor expansion at $\boldsymbol{\aleph}^*_{\cdot \cdot}$, as given in Lemma \ref{lemma1}:
\begin{equation}
\label{eq:teo2_KKT_4}
\widehat{\mathbf{p}}_{\cdot i} \left ({\widehat{\boldsymbol{\chi}}^{*}_{\cdot \cdot }} \right )=\mathbf{p}_{\cdot i} \left (\boldsymbol{\aleph}^*_{\cdot \cdot} \right )-
 \mathbf{p}_{\cdot \cdot}\left (\boldsymbol{\aleph}^*_{\cdot \cdot} \right ) 
 \underset{\mathbf {=0}}{\underbrace{
 \left [\widehat{\boldsymbol{\chi}}^{*\top}_{\cdot \cdot}  (\widehat{\mathbf{K}}-\mathbf{K})\right ]}
 }
  \mathbf{p}_{\cdot i}\left (\boldsymbol{\aleph}^*_{\cdot \cdot} \right )+ \boldsymbol{\varepsilon} \Rightarrow  \widehat{\mathbf{p}}_{\cdot i} \left ({\widehat{\boldsymbol{\chi}}^{*}_{\cdot \cdot}} \right )=\mathbf{p}_{\cdot i} \left (\boldsymbol{\aleph}^*_{\cdot \cdot} \right )+\boldsymbol{\varepsilon}
\end{equation}

Next, using  Equation~(\ref{eq:teo2_KKT_4}), Equation~(\ref{eq:teo2_KKT_2}) can be rewritten as:
\begin{eqnarray}
\label{eq:teo2_KKT_5}
\widehat{\mathbf{K}} \left [ N \lambda \left [\widehat{\boldsymbol{\chi}}^*_{\cdot i}  +
\widehat{\boldsymbol{\chi}}^{*\perp}_{\cdot i} \right]
+ \widehat{\mathbf{p}}_{\cdot i}(\widehat{\boldsymbol{\chi}}^*_{\cdot \cdot}) -\boldsymbol{y}_{\cdot i} 
 +\boldsymbol{\varepsilon}
 \right ]=\mathbf{0}, \hspace{1cm} i=1,\cdots I.\end{eqnarray}

From orthogonality: i) $\widehat{\mathbf{K}} 
\widehat{\boldsymbol{\chi}}^{*\perp}_{\cdot i}=\mathbf{0}$. Additionally, the property  ii)  $\| \boldsymbol{\varepsilon}\|_2= o(\| (\widehat{\mathbf{K}}-\mathbf{K})^\top\widehat{\boldsymbol{\chi}}^{*}_{\cdot \cdot}\|_{2,\infty})=0 \Rightarrow \boldsymbol{\varepsilon}=\mathbf{0}$ holds. Incorporating i) and ii)  into Equation~(\ref{eq:teo2_KKT_5}) shows that $\widehat{\boldsymbol{\chi}}^*_{\cdot i}$ satisfies the optimality condition:
\begin{eqnarray}
\label{eq:teo2_KKT_6}
\widehat{\mathbf{K}} \left [ N \lambda \widehat{\boldsymbol{\chi}}^*_{\cdot i}  
 + \widehat{\mathbf{p}}_{\cdot i}(\widehat{\boldsymbol{\chi}}^*_{\cdot \cdot}) - \boldsymbol{y}_{\cdot i} 
 \right ]=\mathbf{0}, \hspace{1cm} i=1,\cdots, I.\end{eqnarray}

Let $\widehat{\boldsymbol{\aleph}}^*_{\cdot i} \in L(\widehat{\boldsymbol{K}})$ with $i=1,\cdots, I$ be the optimal solution for Nyström \gls{klr} posed in the matrix $\widehat{\mathbf{K}}$. As $\widehat{\boldsymbol{\chi}}^*_{\cdot i} \in L(\widehat{\boldsymbol{K}})$ and from the uniqueness of the optimal solution for the restricted training problem, it follows that  $\widehat{\boldsymbol{\aleph}}^*_{\cdot i} =\widehat{\boldsymbol{\chi}}^*_{\cdot i} $ for every $i$. Thus:
\begin{equation}
\|\boldsymbol{\aleph}^*_{\cdot i}- \widehat{\boldsymbol{\aleph}}^*_{\cdot i} \|_p=
\|\boldsymbol{\chi}^*_{\cdot i}+\boldsymbol{\chi}^{*\perp}_{\cdot i}- \widehat{\boldsymbol{\aleph}}^*_{\cdot i} \|_p=
\|\widehat{\boldsymbol{\aleph}}^*_{\cdot i}+\boldsymbol{\chi}^{*\perp}_{\cdot i}- \widehat{\boldsymbol{\aleph}}^*_{\cdot i} \|_p
=
\|{\mathcal P}^\perp_{L(\widehat{\mathbf{K}})^\perp} (\boldsymbol{\aleph}^*_{\cdot i})\|_p
\end{equation}

By definition of the $p-$Schatten norm:
\begin{eqnarray}
\nonumber
\| 
{\mathcal P}^\perp_{L(\widehat{\mathbf{K}})^\perp} (\boldsymbol{\aleph}^*_{\cdot i})
\|_p = \hbox{Trace} \left (
\left [ 
\left [
 [\mathbb{I}- \boldsymbol{\Pi}^\perp _{L(\widehat{\mathbf{K}}) }  ] \boldsymbol{\aleph}^*_{\cdot i} 
\right ] ^\top
\left [ [\mathbb{I}- \boldsymbol{\Pi}^\perp _{L(\widehat{\mathbf{K}}) }  ] \boldsymbol{\aleph}^*_{\cdot i} \right ]
\right ]^{p/2}
\right )^{1/p}
=
\\
\sqrt{
 \boldsymbol{\aleph}^{*\top}_{\cdot i} 
 [\mathbb{I}- \boldsymbol{\Pi}^\perp _{ L(\widehat{\mathbf{K}})  }  ]^2 \boldsymbol{\aleph}^*_{\cdot i} }= \sqrt{
 \boldsymbol{\aleph}^{*\top}_{\cdot i} 
\left [\mathbb{I}- \boldsymbol{\Pi}^\perp _{ L(\widehat{\mathbf{K}}) } \right ] \boldsymbol{\aleph}^*_{\cdot i} }
\label{eq:proyecccion_ortogonal_solucion}
\end{eqnarray}

Next, lets compute the value of the preceding expression (\ref{eq:proyecccion_ortogonal_solucion}). Rewriting  Equation~(\ref{eq:proyeccion_nystrom}) for the  symmetric positive semi-definite $\mathbf{K}^2$ gives:
\begin{equation}
\mathbf{K}^2- \widehat{\mathbf{K}^2}= \mathbf{K} \left [\mathbb{I}- \boldsymbol{\Pi}^\perp _{L(\mathbf{K} \mathbf{P}) }\right ]\mathbf{K}
\end{equation}

Multiplying on the left and on the right of the preceding expression with the matrix $\mathbf{K}^{\dag}=\mathbf{U} \boldsymbol{\Sigma}^{-1} \mathbf{U} ^\perp$ gives: 
\begin{equation}
\label{eq:proyeccion_nystrom1}
\mathbf{K}^{\dag} \left [\mathbf{K}^2- \widehat{\mathbf{K}^2} \right ]\mathbf{K}^{\dag}= \mathbf{U}\mathbf{U}^\perp \left [\mathbb{I}- \boldsymbol{\Pi}^\perp _{L(\mathbf{C}) }\right ]\mathbf{U}\mathbf{U}^\perp,
\end{equation}
\noindent where, to simplify notation, let $\mathbf{C}=\mathbf{KP}$.

Note that ${\boldsymbol \Pi}^\perp_{L(\mathbf{U})}=\mathbf{U}\mathbf{U}^\perp$. As $\boldsymbol{\aleph^*_{\cdot i}} \in L(\mathbf{K})=L(\mathbf{U})$ then  ${\mathcal P}^\perp_{L(\mathbf{U})} (\boldsymbol{\aleph^*_{\cdot i}})=\mathbf{U}\mathbf{U}^\perp\boldsymbol{\aleph^*_{\cdot i}}=\boldsymbol{\aleph^*_{\cdot i}} \in L(\mathbf{K})$.  On the other hand, as $\boldsymbol{\aleph^*_{\cdot i}} \in L(\mathbf{K})$, the orthogonal protection of $\left [\mathbb{I}- \boldsymbol{\Pi}^\perp _{L(\mathbf{C}) }\right ] \boldsymbol{\aleph}^*_{\cdot i} \in L(\mathbf{K})$,  and, considering the above argument, $ \mathbf{U} \mathbf{U} ^\perp \left [\mathbb{I}- \boldsymbol{\Pi}^\perp _{L(\mathbf{C}) } \right ] \boldsymbol{\aleph}^*_{\cdot i}=\left [\mathbb{I}- \boldsymbol{\Pi}^\perp _{L(\mathbf{C}) }\right ] \boldsymbol{\aleph}^*_{\cdot i}$. Joining both relationships and using Equation~(\ref{eq:proyeccion_nystrom1}), it is found that:
\begin{align}
\nonumber
\| 
{\mathcal P}^\perp_{L(\widehat{\mathbf{K}})^\perp} (\boldsymbol{\aleph}^*_{\cdot i})
\|_p &= \sqrt{
\boldsymbol{\aleph}^{*\top}_{\cdot i}\left [\mathbb{I}- \boldsymbol{\Pi}^\perp _{L(\mathbf{C}) } \right ] \boldsymbol{\aleph}^*_{\cdot i}}=\sqrt{
\boldsymbol{\aleph}^{*\top}_{\cdot i}
\mathbf{U} \mathbf{U} ^\perp \left [\mathbb{I}- \boldsymbol{\Pi}^\perp _{L(\mathbf{C}) } \right ]\mathbf{U} \mathbf{U} ^\perp \boldsymbol{\aleph}^*_{\cdot i} }=\sqrt{
\boldsymbol{\aleph}^{*\top}_{\cdot i}\mathbf{K}^{\dag} \left [\mathbf{K}^2- \widehat{\mathbf{K}^2} \right ]\mathbf{K}^{\dag}\boldsymbol{\aleph}^*_{\cdot i}}\\
&= \left \|  \left [\mathbf{K}^2- \widehat{\mathbf{K}^2} \right ]^{1/2}\mathbf{K}^{\dag}\boldsymbol{\aleph}^*_{\cdot i} \right \|_p
\end{align}

Finally, because all Schatten norms are sub-multiplicative:
\begin{eqnarray}
\|\boldsymbol{\aleph}^*_{\cdot i}- \widehat{\boldsymbol{\aleph}}^*_{\cdot i} \|_p=\left \|  \left [\mathbf{K}^2- \widehat{\mathbf{K}^2} \right ]^{1/2}\mathbf{K}^{\dag}\boldsymbol{\aleph}^*_{\cdot i} \right \|_p
\le
\| \boldsymbol{\aleph}^{*}_{\cdot i} \|_p  \left \| \left [\mathbf{K}^2- \widehat{\mathbf{K}^2} \right ]^{1/2}\mathbf{K}^{\dag}\right \|_p
 \end{eqnarray}

The following statement of the theorem 
\begin{equation*}
      \|\boldsymbol{\aleph}^*_{\cdot i}\|_p\le
      \frac{2}{N\lambda}\| \mathbf{1} \|_p
\end{equation*}
 \noindent will now be proved.

From the optimality conditions, stated in Equation~(\ref{eq:teo2_KKT_1}), it can be deduced that $\mathbf{K} \, \mathbf{z}_i=\mathbf{0}$, and therefore, $\mathbf{z}_i \in L(\mathbf{K})^\perp \Rightarrow \mathcal{P}^\perp_{L(\mathbf{K})} (\mathbf{z}_i)=\mathbf{0}$. From which it is obtained:
\begin{equation}
    {\mathcal P}^\perp_{L(\mathbf{K})} \left (\left [ N \lambda \boldsymbol{\aleph}^*_{\cdot i}   + \mathbf{p}_{\cdot i}(\boldsymbol{\aleph}^*_{\cdot \cdot})  - \boldsymbol{y}_{\cdot i} \right ] \right )=
    { N \lambda \boldsymbol{\aleph}^*_{\cdot i}+ \mathcal P}^\perp_{L(\mathbf{K})} \left (\left [   \mathbf{p}_{\cdot i}(\boldsymbol{\aleph}^*_{\cdot \cdot}) - \boldsymbol{y}_{\cdot i} \right ] \right )=\mathbf{0}
\end{equation}
Then:
\begin{equation*}
      \|{\boldsymbol{\aleph}^*_{\cdot i}\|_p= \frac{1}{N\lambda} \| \mathcal P}^\perp_{L(\mathbf{K})} \left (   \mathbf{p}_{\cdot i}(\boldsymbol{\aleph}^*_{\cdot \cdot}) - \boldsymbol{y}_{\cdot i}  \right )\|_p \le 
      \frac{1}{N\lambda} \|   \mathbf{p}_{\cdot i}(\boldsymbol{\aleph}^*_{\cdot \cdot}) - \boldsymbol{y}_{\cdot i} \|_p\le
      \frac{1}{N\lambda} \left [\, \|    \mathbf{p}_{\cdot i}(\boldsymbol{\aleph}^*_{\cdot \cdot}) \|_p+\|\boldsymbol{y}_{\cdot i} \|_p \,\right ]
\end{equation*}
\noindent

Note that the components of these vectors $p_{ni}$ are probabilities, and therefore $0\le p_{ni} \le 1$. Moreover, through the codification of \(y_{ni}\), these elements take values of \(0\) or \(1\), yielding the desired result.
\end{proof}

\section{Hyperparameter Optimisation}
\label{App:HPO}

Table~\ref{tab:hyperparameters-space-Nystrom} provides an overview of the search spaces for every hyperparameter used in the \gls{ml} algorithms employed within this study. Additionally, the table includes the optimal values of those hyperparameters obtained for each dataset.

\begin{table}[!h]
\caption{Hyperparameter space of the ML models}
\label{tab:hyperparameters-space-Nystrom}
\resizebox{\textwidth}{!}{
\begin{tabular}{llllllll}
\toprule
{\bf Technique} &{\bf Name of the hyperparameter } &{\bf Notation} &{\bf Type} &{\bf Search space} &{\bf NTS} & {\bf LPMC}\\
\midrule
  
\multirow{1}{1cm}{LinearSVM} & Soft margin constant &$C$ & Log-uniform distribution &  $[0.1, 10]$ & $6.380$ & $2.704$\\
\hline
\multirow{3}{1cm}{SVM} & Soft margin constant &$C$ & Log-uniform distribution &  $[0.1, 10]$ & $7.999$ & $6.391$\\
& Kernel function & $K$ & Choice & $[\textrm{RBF}]$ & $\textrm{RBF}$ & $\textrm{RBF}$\\
& Parameter of the Gaussian function & $c$ & Log-uniform distribution & $[10^{-3}, 1]$ & $0.008$ & $0.007$\\
& Number of landmark points in Nyström & $C$ & Choice & $[500,1000]$ & $1000$ & $500$\\
\hline
\multirow{5}{1cm}{RF} & Number of individual decision trees & $B$ & Uniform distribution  & $[1, 200]$ & $153$ & $180$\\
& Maximum features for the best split&$m$ & Uniform distribution & $[2, \textrm{Nº features}]$ & $8$ & $16$\\
& Maximum depth of each tree &$d$ & Uniform distribution & $[3,10]$ & $10$ & $10$\\
& Minimum samples to be at a leaf node &$l$ & Uniform distribution & $[1,20]$ & $3$ & $11$\\
& Minimum samples to split an internal node &$s$ & Uniform distribution & $[2,20]$ & $15$ & $14$\\
& Goodness of split metric &$c$ & Choice & $[\textrm{Gini}|\textrm{Entropy}]$ & Entropy & Entropy\\
\hline
\multirow{10}{1cm}{XGBoost} & Maximum tree depth & $d$ & Uniform distribution & $[1,14]$ & $7$ & $7$\\
& Minimum loss for a new split & $\gamma$ & Log-uniform distribution & $[10^{-4},5]$ & $4.970$ & $4.137$\\
& Minimum instance weight in a child  & $w$ & Uniform distribution & $[1,100]$ & $1$ & $32$\\
& Maximum delta step in each tree & $\delta$ & Uniform distribution & $[0,10]$ & $0$ & $4$\\
& Subsample ratio of the training instance & $s$ & Uniform distribution & $[0.5,1]$& $0.823$ & $0.935$\\
& Subsample ratio of columns at each tree & $c_t$ & Uniform distribution & $[0.5,1]$ & $0.553$ & $0.679$\\
& Subsample ratio of columns for each level & $c_l$ & Uniform distribution & $[0.5,1]$& $0.540$ & $0.629$\\
& L1 regularisation term & $\alpha$ & Log-uniform distribution & $[10^{-4},10]$ & $0.028$ & $0.003$\\
& L2 regularisation term & $\lambda$ & Log-uniform distribution & $[10^{-4},10]$ & $0.264$ & $0.5 \cdot 10^{-3}$\\
& Number of boosting rounds & $B$ & Uniform distribution & $[1,6000]$ & $4376$ & $2789$\\
\hline
\multirow{8}{1cm}{NN} & Number of neurons in the hidden layer & $n_1$ & Uniform distribution & $[10, 500]$ & 10 & 51\\
& Activation function of hidden layers & $f$ & Choice & $[\textrm{tanh}]$ & $\textrm{tanh}$ & $\textrm{tanh}$\\
& Solver for weights optimisation & $S$ & Choice & $[\textrm{LBFGS} |\textrm{SGD}|\textrm{Adam}]$ & $\textrm{LBFGS}$ & $\textrm{SGD}$\\
& Initial learning rate &$\eta_{0}$ & Uniform distribution & $[10^{-4}, 1]$ & 0.416 & 0.041\\
& Learning rate schedule & $\eta$ & Choice & $[\textrm{adaptive}]$ & $\textrm{adaptive}$ & $\textrm{adaptive}$\\
& Maximum number of iterations & $t$ & Choice & $[10^6]$ & $10^6$ & $10^6$\\
& Batch size for training & $BS$ & Choice & $[128|256|512|1024]$ & $512$ & $1024$\\
& Tolerance for the optimisation algorithm & $tol$ & Choice & $[10^{-3}]$ & $10^{-3}$ & $10^{-3}$\\
\hline
\multirow{3}{1cm}{KLR} & Kernel function & $K$ & Choice & $[\textrm{RBF}]$ & $\textrm{RBF}$ & $\textrm{RBF}$ \\
& Parameter of the Gaussian function & $c$ & Log-uniform distribution & $[10^{-3}, 1]$ & $0.054$ & $0.037$\\
& Tikhonov penalization parameter &$\lambda$ & Fixed & $10^{-6}$ & $10^{-6}$ & $10^{-6}$\\
\bottomrule
\end{tabular}
}
\end{table}

\bibliographystyle{elsarticle-harv} 
\bibliography{references}

\begin{thebibliography}{69}
\expandafter\ifx\csname natexlab\endcsname\relax\def\natexlab#1{#1}\fi
\providecommand{\url}[1]{\texttt{#1}}
\providecommand{\href}[2]{#2}
\providecommand{\path}[1]{#1}
\providecommand{\DOIprefix}{doi:}
\providecommand{\ArXivprefix}{arXiv:}
\providecommand{\URLprefix}{URL: }
\providecommand{\Pubmedprefix}{pmid:}
\providecommand{\doi}[1]{\href{http://dx.doi.org/#1}{\path{#1}}}
\providecommand{\Pubmed}[1]{\href{pmid:#1}{\path{#1}}}
\providecommand{\bibinfo}[2]{#2}
\ifx\xfnm\relax \def\xfnm[#1]{\unskip,\space#1}\fi
\bibitem[{Alaoui and Mahoney(2015)}]{AlM15}
\bibinfo{author}{Alaoui, A.E.}, \bibinfo{author}{Mahoney, M.W.}, \bibinfo{year}{2015}.
\newblock \bibinfo{title}{{Fast randomized kernel ridge regression with statistical guarantees}}.
\newblock \bibinfo{journal}{Advances in Neural Information Processing Systems} \bibinfo{volume}{2015-Janua}, \bibinfo{pages}{775--783}.
\bibitem[{Allouch et~al.(2021)Allouch, Remogna, Sbibih and Tahrichi}]{ARS05}
\bibinfo{author}{Allouch, C.}, \bibinfo{author}{Remogna, S.}, \bibinfo{author}{Sbibih, D.}, \bibinfo{author}{Tahrichi, M.}, \bibinfo{year}{2021}.
\newblock \bibinfo{title}{{Superconvergent methods based on quasi-interpolating operators for fredholm integral equations of the second kind.}}
\newblock \bibinfo{journal}{Applied Mathematics and Computation} \bibinfo{volume}{404}.
\newblock \DOIprefix\doi{10.1016/j.amc.2021.126227}.
\bibitem[{Andrew and Gao(2007)}]{AnG07}
\bibinfo{author}{Andrew, G.}, \bibinfo{author}{Gao, J.}, \bibinfo{year}{2007}.
\newblock \bibinfo{title}{{Scalable training of L1-regularized log-linear models}}, in: \bibinfo{booktitle}{ACM International Conference Proceeding Series}, \bibinfo{publisher}{ACM Press}, \bibinfo{address}{New York, New York, USA}. pp. \bibinfo{pages}{33--40}.
\newblock \URLprefix \url{http://portal.acm.org/citation.cfm?doid=1273496.1273501}, \DOIprefix\doi{10.1145/1273496.1273501}.
\bibitem[{Ballings et~al.(2015)Ballings, Van Den~Poel, Hespeels and Gryp}]{BVH15}
\bibinfo{author}{Ballings, M.}, \bibinfo{author}{Van Den~Poel, D.}, \bibinfo{author}{Hespeels, N.}, \bibinfo{author}{Gryp, R.}, \bibinfo{year}{2015}.
\newblock \bibinfo{title}{{Evaluating multiple classifiers for stock price direction prediction}}.
\newblock \bibinfo{journal}{Expert Systems with Applications} \bibinfo{volume}{42}, \bibinfo{pages}{7046--7056}.
\bibitem[{Bansal et~al.(2019)Bansal, Daziano and Sunder}]{BDS19}
\bibinfo{author}{Bansal, P.}, \bibinfo{author}{Daziano, R.A.}, \bibinfo{author}{Sunder, N.}, \bibinfo{year}{2019}.
\newblock \bibinfo{title}{{Arriving at a decision: A semi-parametric approach to institutional birth choice in India}}.
\newblock \bibinfo{journal}{Journal of Choice Modelling} \bibinfo{volume}{31}, \bibinfo{pages}{86--103}.
\newblock \URLprefix \url{https://www.sciencedirect.com/science/article/pii/S175553451830054X?via%3Dihub}, \DOIprefix\doi{10.1016/J.JOCM.2019.04.001}.
\bibitem[{Ben-Akiva and Lerman(1985)}]{BLL85}
\bibinfo{author}{Ben-Akiva, M.E.}, \bibinfo{author}{Lerman, S.R.}, \bibinfo{year}{1985}.
\newblock \bibinfo{title}{{Discrete choice analysis: Theory and Application to Travel Demand}}. volume~\bibinfo{volume}{9}.
\newblock \bibinfo{publisher}{MIT press}, \bibinfo{address}{Cambridge, Massachusetts}.
\bibitem[{Bierlaire(2003)}]{Bie03}
\bibinfo{author}{Bierlaire, M.}, \bibinfo{year}{2003}.
\newblock \bibinfo{title}{{BIOGEME: A free package for the estimation of discrete choice models}}, in: \bibinfo{booktitle}{Swiss transport research conference}.
\bibitem[{Bottou et~al.(2018)Bottou, Curtis and Nocedal}]{BCN18}
\bibinfo{author}{Bottou, L.}, \bibinfo{author}{Curtis, F.E.}, \bibinfo{author}{Nocedal, J.}, \bibinfo{year}{2018}.
\newblock \bibinfo{title}{{Optimization Methods for Large-Scale Machine Learning}}.
\newblock \bibinfo{journal}{SIAM Review} \bibinfo{volume}{60}, \bibinfo{pages}{223--311}.
\newblock \URLprefix \url{https://doi.org/10.1137/16M1080173}, \DOIprefix\doi{10.1137/16M1080173}.
\bibitem[{Byrd et~al.(1995)Byrd, Lu, Nocedal and Zhu}]{BLN95}
\bibinfo{author}{Byrd, R.H.}, \bibinfo{author}{Lu, P.}, \bibinfo{author}{Nocedal, J.}, \bibinfo{author}{Zhu, C.}, \bibinfo{year}{1995}.
\newblock \bibinfo{title}{{A Limited Memory Algorithm for Bound Constrained Optimization}}.
\newblock \bibinfo{journal}{SIAM Journal on Scientific Computing} \bibinfo{volume}{16}, \bibinfo{pages}{1190--1208}.
\newblock \URLprefix \url{http://epubs.siam.org/doi/10.1137/0916069}, \DOIprefix\doi{10.1137/0916069}.
\bibitem[{Cawley and Talbot(2005)}]{CaT05}
\bibinfo{author}{Cawley, G.}, \bibinfo{author}{Talbot, N.}, \bibinfo{year}{2005}.
\newblock \bibinfo{title}{{The evidence framework applied to sparse kernel logistic regression}}.
\newblock \bibinfo{journal}{Neurocomputing} \bibinfo{volume}{64}, \bibinfo{pages}{119--135}.
\newblock \DOIprefix\doi{10.1016/j.neucom.2004.11.021}.
\bibitem[{Chen and Guestrin(2016)}]{Chen16}
\bibinfo{author}{Chen, T.}, \bibinfo{author}{Guestrin, C.}, \bibinfo{year}{2016}.
\newblock \bibinfo{title}{{XGBoost: A Scalable Tree Boosting System}}, in: \bibinfo{booktitle}{Proceedings of the 22nd ACM SIGKDD International Conference on Knowledge Discovery and Data Mining}, \bibinfo{publisher}{ACM}, \bibinfo{address}{New York, NY, USA}. pp. \bibinfo{pages}{785--794}.
\newblock \URLprefix \url{http://doi.acm.org/10.1145/2939672.2939785}, \DOIprefix\doi{10.1145/2939672.2939785}.
\bibitem[{Cherfaoui et~al.(2022)Cherfaoui, Kadri and Ralaivola}]{CKR22}
\bibinfo{author}{Cherfaoui, F.}, \bibinfo{author}{Kadri, H.}, \bibinfo{author}{Ralaivola, L.}, \bibinfo{year}{2022}.
\newblock \bibinfo{title}{{Scalable Ridge Leverage Score Sampling for the Nystr{\"{o}}m Method}}, in: \bibinfo{booktitle}{ICASSP 2022 - 2022 IEEE International Conference on Acoustics, Speech and Signal Processing (ICASSP)}, pp. \bibinfo{pages}{4163--4167}.
\newblock \DOIprefix\doi{10.1109/ICASSP43922.2022.9747039}.
\bibitem[{Cortes et~al.(2010)Cortes, Mohri and Talwalkar}]{CMT10}
\bibinfo{author}{Cortes, C.}, \bibinfo{author}{Mohri, M.}, \bibinfo{author}{Talwalkar, A.}, \bibinfo{year}{2010}.
\newblock \bibinfo{title}{{On the impact of kernel approximation on learning accuracy}}.
\newblock \bibinfo{journal}{Journal of Machine Learning Research} \bibinfo{volume}{9}, \bibinfo{pages}{113--120}.
\bibitem[{Cortes and Vapnik(1995)}]{CoV95}
\bibinfo{author}{Cortes, C.}, \bibinfo{author}{Vapnik, V.}, \bibinfo{year}{1995}.
\newblock \bibinfo{title}{{Support-vector networks}}.
\newblock \bibinfo{journal}{Machine Learning} \bibinfo{volume}{20}, \bibinfo{pages}{273--297}.
\newblock \DOIprefix\doi{10.1007/BF00994018}.
\bibitem[{Daganzo and Kusnic(1993)}]{DaK93}
\bibinfo{author}{Daganzo, C.F.}, \bibinfo{author}{Kusnic, M.}, \bibinfo{year}{1993}.
\newblock \bibinfo{title}{{Technical Note-Two Properties of the Nested Logit Model}}.
\newblock \bibinfo{journal}{Transportation Science} \bibinfo{volume}{27}, \bibinfo{pages}{395}.
\newblock \URLprefix \url{http://pubsonline.informs.org400.https//doi.org/10.1287/trsc.27.4.395http://www.informs.org}, \DOIprefix\doi{10.1287/trsc.27.4.395}.
\bibitem[{De~Bonis and Laurita(2016)}]{DeL16}
\bibinfo{author}{De~Bonis, M.C.}, \bibinfo{author}{Laurita, C.}, \bibinfo{year}{2016}.
\newblock \bibinfo{title}{{A modified Nystr{\"{o}}m method for integral equations with Mellin type kernels}}.
\newblock \bibinfo{journal}{Journal of Computational and Applied Mathematics} \bibinfo{volume}{296}, \bibinfo{pages}{512--527}.
\newblock \URLprefix \url{https://www.sciencedirect.com/science/article/pii/S0377042715005099}, \DOIprefix\doi{https://doi.org/10.1016/j.cam.2015.10.010}.
\bibitem[{Epperly(2022)}]{Epp22}
\bibinfo{author}{Epperly, E.N.}, \bibinfo{year}{2022}.
\newblock \bibinfo{title}{{Low-Rank Approximation Toolbox: Nystr{\"{o}}m Approximation}}.
\newblock \URLprefix \url{https://www.ethanepperly.com/index.php/2022/10/11/low-rank-approximation-toolbox-nystrom-approximation/}.
\bibitem[{Espinosa-Aranda et~al.(2018)Espinosa-Aranda, Garc{\'{i}}a-R{\'{o}}denas, L{\'{o}}pez-Garc{\'{i}}a and Angulo}]{EGL18}
\bibinfo{author}{Espinosa-Aranda, J.}, \bibinfo{author}{Garc{\'{i}}a-R{\'{o}}denas, R.}, \bibinfo{author}{L{\'{o}}pez-Garc{\'{i}}a, M.}, \bibinfo{author}{Angulo, E.}, \bibinfo{year}{2018}.
\newblock \bibinfo{title}{{Constrained nested logit model: formulation and estimation}}.
\newblock \bibinfo{journal}{Transportation} \bibinfo{volume}{45}, \bibinfo{pages}{1523--1557}.
\newblock \DOIprefix\doi{10.1007/s11116-017-9774-2}.
\bibitem[{Espinosa-Aranda et~al.(2015)Espinosa-Aranda, Garc{\'{i}}a-R{\'{o}}denas, Ram{\'{i}}rez-Flores, L{\'{o}}pez-Garc{\'{i}}a and Angulo}]{EGR15}
\bibinfo{author}{Espinosa-Aranda, J.}, \bibinfo{author}{Garc{\'{i}}a-R{\'{o}}denas, R.}, \bibinfo{author}{Ram{\'{i}}rez-Flores, M.}, \bibinfo{author}{L{\'{o}}pez-Garc{\'{i}}a, M.}, \bibinfo{author}{Angulo, E.}, \bibinfo{year}{2015}.
\newblock \bibinfo{title}{{High-speed railway scheduling based on user preferences}}.
\newblock \bibinfo{journal}{European Journal of Operational Research} \bibinfo{volume}{246}, \bibinfo{pages}{772--786}.
\newblock \DOIprefix\doi{10.1016/j.ejor.2015.05.052}.
\bibitem[{Fern{\'{a}}ndez-Delgado et~al.(2014)Fern{\'{a}}ndez-Delgado, Cernadas, Barro and Amorim}]{FCB14}
\bibinfo{author}{Fern{\'{a}}ndez-Delgado, M.}, \bibinfo{author}{Cernadas, E.}, \bibinfo{author}{Barro, S.}, \bibinfo{author}{Amorim, D.}, \bibinfo{year}{2014}.
\newblock \bibinfo{title}{{Do we need hundreds of classifiers to solve real world classification problems?}}
\newblock \bibinfo{journal}{Journal of Machine Learning Research} \bibinfo{volume}{15}, \bibinfo{pages}{3133--3181}.
\bibitem[{Friedman(2001)}]{Fri01}
\bibinfo{author}{Friedman, J.H.}, \bibinfo{year}{2001}.
\newblock \bibinfo{title}{{Greedy function approximation: A gradient boosting machine}}.
\newblock \bibinfo{journal}{Annals of Statistics} \bibinfo{volume}{29}, \bibinfo{pages}{1189--1232}.
\newblock \DOIprefix\doi{10.1214/AOS/1013203451}.
\bibitem[{Gittens and Mahoney(2016)}]{GiM16}
\bibinfo{author}{Gittens, A.}, \bibinfo{author}{Mahoney, M.W.}, \bibinfo{year}{2016}.
\newblock \bibinfo{title}{{Revisiting the Nystr{\"{o}}m method for improved large-scale machine learning}}.
\newblock \bibinfo{journal}{Journal of Machine Learning Research} \bibinfo{volume}{17}, \bibinfo{pages}{1--65}.
\bibitem[{Hagenauer and Helbich(2017)}]{HaH17}
\bibinfo{author}{Hagenauer, J.}, \bibinfo{author}{Helbich, M.}, \bibinfo{year}{2017}.
\newblock \bibinfo{title}{{A comparative study of machine learning classifiers for modeling travel mode choice}}.
\newblock \bibinfo{journal}{Expert Systems with Applications} \bibinfo{volume}{78}, \bibinfo{pages}{273--282}.
\newblock \DOIprefix\doi{10.1016/j.eswa.2017.01.057}.
\bibitem[{Hasan et~al.(2016)Hasan, Zhiyu and Mahani}]{HZM16}
\bibinfo{author}{Hasan, A.}, \bibinfo{author}{Zhiyu, W.}, \bibinfo{author}{Mahani, A.S.}, \bibinfo{year}{2016}.
\newblock \bibinfo{title}{{Fast estimation of multinomial logit models: R package mnlogit}}.
\newblock \bibinfo{journal}{Journal of Statistical Software} \bibinfo{volume}{75}.
\newblock \DOIprefix\doi{10.18637/jss.v075.i03}.
\bibitem[{Hastie et~al.(2001)Hastie, Tibshirani and Friedman}]{HTF01}
\bibinfo{author}{Hastie, T.}, \bibinfo{author}{Tibshirani, R.}, \bibinfo{author}{Friedman, J.}, \bibinfo{year}{2001}.
\newblock \bibinfo{title}{{The Elements of Statistical Learning}}.
\newblock Springer Series in Statistics, \bibinfo{publisher}{Springer New York Inc.}, \bibinfo{address}{New York, NY, USA}.
\bibitem[{Hillel(2019)}]{Hil19}
\bibinfo{author}{Hillel, T.}, \bibinfo{year}{2019}.
\newblock \bibinfo{title}{{Understanding travel mode choice: A new approach for city scale simulation}}.
\newblock Ph.D. thesis. University of Cambridge.
\bibitem[{Hillel(2021)}]{Hil21}
\bibinfo{author}{Hillel, T.}, \bibinfo{year}{2021}.
\newblock \bibinfo{title}{{New perspectives on the performance of machine learning classifiers for mode choice prediction: An experimental review}}, in: \bibinfo{booktitle}{21st Swiss Transport Research Conference}, \bibinfo{address}{Monte Verit{\`{a}}, Ascona}.
\newblock \URLprefix \url{http://www.strc.ch/2021/Hillel.pdf}.
\bibitem[{Hillel et~al.(2018)Hillel, Elshafie and Jin}]{HEJ18}
\bibinfo{author}{Hillel, T.}, \bibinfo{author}{Elshafie, M.Z.E.B.}, \bibinfo{author}{Jin, Y.}, \bibinfo{year}{2018}.
\newblock \bibinfo{title}{{Recreating passenger mode choice-sets for transport simulation: A case study of London, UK}}.
\newblock \bibinfo{journal}{Proceedings of the Institution of Civil Engineers - Smart Infrastructure and Construction} \bibinfo{volume}{171}, \bibinfo{pages}{29--42}.
\newblock \URLprefix \url{https://www.icevirtuallibrary.com/doi/10.1680/jsmic.17.00018}, \DOIprefix\doi{10.1680/JSMIC.17.00018}.
\bibitem[{Hunag et~al.(2019)Hunag, Chen and Hu}]{HCH19}
\bibinfo{author}{Hunag, C.H.}, \bibinfo{author}{Chen, J.S.}, \bibinfo{author}{Hu, C.C.}, \bibinfo{year}{2019}.
\newblock \bibinfo{title}{{The Schatten p-norm on Rn}}.
\newblock \bibinfo{journal}{Journal of Nonlinear and Convex Analysis} \bibinfo{volume}{21}, \bibinfo{pages}{21--29}.
\bibitem[{Karsmakers et~al.(2007)Karsmakers, Pelckmans and Suykens}]{Karsmakers2007}
\bibinfo{author}{Karsmakers, P.}, \bibinfo{author}{Pelckmans, K.}, \bibinfo{author}{Suykens, J.A.}, \bibinfo{year}{2007}.
\newblock \bibinfo{title}{{Multi-class kernel logistic regression: A fixed-size implementation}}, in: \bibinfo{booktitle}{IEEE International Conference on Neural Networks - Conference Proceedings}, pp. \bibinfo{pages}{1756--1761}.
\newblock \DOIprefix\doi{10.1109/IJCNN.2007.4371223}.
\bibitem[{Kimeldorf and Wahba(1971)}]{KiG71}
\bibinfo{author}{Kimeldorf, G.}, \bibinfo{author}{Wahba, G.}, \bibinfo{year}{1971}.
\newblock \bibinfo{title}{{Some results on Tchebycheffian spline functions}}.
\newblock \bibinfo{journal}{J. Math. Anal. Appl.} \bibinfo{volume}{33}, \bibinfo{pages}{82--95}.
\newblock \URLprefix \url{https://doi.org/10.1016/0022-247X(71)90184-3}, \DOIprefix\doi{10.1016/0022-247X(71)90184-3}.
\bibitem[{Kingma and Ba(2014)}]{KiB14}
\bibinfo{author}{Kingma, D.P.}, \bibinfo{author}{Ba, J.}, \bibinfo{year}{2014}.
\newblock \bibinfo{title}{{Adam: A Method for Stochastic Optimization}}.
\newblock \URLprefix \url{http://arxiv.org/abs/1412.6980}, \DOIprefix\doi{10.48550/arXiv.1412.6980}.
\bibitem[{Koh et~al.(2007)Koh, Kim, Boyd and Lin}]{KKB07}
\bibinfo{author}{Koh, K.}, \bibinfo{author}{Kim, S.J.}, \bibinfo{author}{Boyd, S.}, \bibinfo{author}{Lin, Y.}, \bibinfo{year}{2007}.
\newblock \bibinfo{title}{{An Interior-Point Method for Large-Scale 1-Regularized Logistic Regression}}.
\newblock \bibinfo{type}{Technical Report}.
\bibitem[{Kumar et~al.(2012)Kumar, Mohri and Talwalkar}]{KMT12}
\bibinfo{author}{Kumar, S.}, \bibinfo{author}{Mohri, M.}, \bibinfo{author}{Talwalkar, A.}, \bibinfo{year}{2012}.
\newblock \bibinfo{title}{{Sampling methods for the Nystr{\"{o}}m method}}.
\newblock \bibinfo{journal}{Journal of Machine Learning Research} \bibinfo{volume}{13}, \bibinfo{pages}{981--1006}.
\bibitem[{Langrock et~al.(2014)Langrock, Heidenreich and Sperlich}]{LHS14}
\bibinfo{author}{Langrock, R.}, \bibinfo{author}{Heidenreich, N.B.}, \bibinfo{author}{Sperlich, S.}, \bibinfo{year}{2014}.
\newblock \bibinfo{title}{{Kernel-based semiparametric multinomial logit modelling of political party preferences}}.
\newblock \bibinfo{journal}{Statistical Methods and Applications} \bibinfo{volume}{23}, \bibinfo{pages}{435--449}.
\newblock \DOIprefix\doi{10.1007/s10260-014-0261-z}.
\bibitem[{Lei et~al.(2019)Lei, Tang, Li and Wu}]{LTL19}
\bibinfo{author}{Lei, D.}, \bibinfo{author}{Tang, J.}, \bibinfo{author}{Li, Z.}, \bibinfo{author}{Wu, Y.}, \bibinfo{year}{2019}.
\newblock \bibinfo{title}{{Using low-rank approximations to speed up kernel logistic regression algorithm}}.
\newblock \bibinfo{journal}{IEEE Access} \bibinfo{volume}{7}, \bibinfo{pages}{84242--84252}.
\newblock \DOIprefix\doi{10.1109/ACCESS.2019.2924542}.
\bibitem[{Lh{\'{e}}ritier et~al.(2019)Lh{\'{e}}ritier, Bocamazo, Delahaye and Acuna-Agost}]{LBD19}
\bibinfo{author}{Lh{\'{e}}ritier, A.}, \bibinfo{author}{Bocamazo, M.}, \bibinfo{author}{Delahaye, T.}, \bibinfo{author}{Acuna-Agost, R.}, \bibinfo{year}{2019}.
\newblock \bibinfo{title}{{Airline itinerary choice modeling using machine learning}}.
\newblock \bibinfo{journal}{Journal of Choice Modelling} \bibinfo{volume}{31}, \bibinfo{pages}{198--209}.
\newblock \DOIprefix\doi{10.1016/j.jocm.2018.02.002}.
\bibitem[{Lin et~al.(2008)Lin, Weng, Keerthi and Smola}]{LWK08}
\bibinfo{author}{Lin, C.J.}, \bibinfo{author}{Weng, R.C.}, \bibinfo{author}{Keerthi, S.S.}, \bibinfo{author}{Smola, A.}, \bibinfo{year}{2008}.
\newblock \bibinfo{title}{{Trust Region Newton Method for Large-Scale Logistic Regression}}.
\newblock \bibinfo{type}{Technical Report}.
\newblock \URLprefix \url{http://www.csie.ntu.edu.tw/}.
\bibitem[{Liu et~al.(2016)Liu, Liu, Tao, Wang and Lu}]{LLT16}
\bibinfo{author}{Liu, W.}, \bibinfo{author}{Liu, H.}, \bibinfo{author}{Tao, D.}, \bibinfo{author}{Wang, Y.}, \bibinfo{author}{Lu, K.}, \bibinfo{year}{2016}.
\newblock \bibinfo{title}{{Manifold regularized kernel logistic regression for web image annotation}}.
\newblock \bibinfo{journal}{Neurocomputing} \bibinfo{volume}{172}, \bibinfo{pages}{3--8}.
\newblock \DOIprefix\doi{10.1016/j.neucom.2014.06.096}.
\bibitem[{L{\'{o}}pez~Garc{\'{i}}a et~al.(2015)L{\'{o}}pez~Garc{\'{i}}a, Garc{\'{i}}a-R{\'{o}}denas and Gonz{\'{a}}lez~G{\'{o}}mez}]{LGG15}
\bibinfo{author}{L{\'{o}}pez~Garc{\'{i}}a, M.L.}, \bibinfo{author}{Garc{\'{i}}a-R{\'{o}}denas, R.}, \bibinfo{author}{Gonz{\'{a}}lez~G{\'{o}}mez, A.}, \bibinfo{year}{2015}.
\newblock \bibinfo{title}{{K-means algorithms for functional data}}.
\newblock \bibinfo{journal}{Neurocomputing} \bibinfo{volume}{151}.
\newblock \DOIprefix\doi{10.1016/j.neucom.2014.09.048}.
\bibitem[{Maalouf and Trafalis(2011a)}]{MaT11}
\bibinfo{author}{Maalouf, M.}, \bibinfo{author}{Trafalis, T.}, \bibinfo{year}{2011}a.
\newblock \bibinfo{title}{{Robust weighted kernel logistic regression in imbalanced and rare events data}}.
\newblock \bibinfo{journal}{Computational Statistics and Data Analysis} \bibinfo{volume}{55}, \bibinfo{pages}{168--183}.
\newblock \DOIprefix\doi{10.1016/j.csda.2010.06.014}.
\bibitem[{Maalouf and Trafalis(2011b)}]{MTA11}
\bibinfo{author}{Maalouf, M.}, \bibinfo{author}{Trafalis, T.B.}, \bibinfo{year}{2011}b.
\newblock \bibinfo{title}{{Kernel logistic regression using truncated Newton method}}.
\newblock \bibinfo{journal}{Comput Manag Sci} \bibinfo{volume}{8}, \bibinfo{pages}{415--428}.
\newblock \DOIprefix\doi{10.1007/s10287-010-0128-1}.
\bibitem[{Malouf(2002)}]{Mal02}
\bibinfo{author}{Malouf, R.}, \bibinfo{year}{2002}.
\newblock \bibinfo{title}{{A comparison of algorithms for maximum entropy parameter estimation}}, in: \bibinfo{booktitle}{In Proceedings of the 6th conference on Natural language learning}, \bibinfo{publisher}{Association for Computational Linguistics}. pp. \bibinfo{pages}{1--7}.
\bibitem[{Mann et~al.(2018)Mann, Spaiser, Hedman and Sumpter}]{MSH16}
\bibinfo{author}{Mann, R.P.}, \bibinfo{author}{Spaiser, V.}, \bibinfo{author}{Hedman, L.}, \bibinfo{author}{Sumpter, D.J.}, \bibinfo{year}{2018}.
\newblock \bibinfo{title}{{Choice modelling with Gaussian processes in the social sciences: A case study of neighbourhood choice in stockholm}}.
\newblock \bibinfo{journal}{PLoS ONE} \bibinfo{volume}{13}, \bibinfo{pages}{e0206687}.
\newblock \URLprefix \url{https://dx.plos.org/10.1371/journal.pone.0206687}, \DOIprefix\doi{10.1371/journal.pone.0206687}.
\bibitem[{Mart{\'{i}}n-Baos(2023)}]{MartinBaos23}
\bibinfo{author}{Mart{\'{i}}n-Baos, J.{\'{A}}.}, \bibinfo{year}{2023}.
\newblock \bibinfo{title}{{Machine learning methods applied to transport demand modelling}}.
\newblock Ph.D. thesis. Universidad de Castilla-La Mancha. \bibinfo{address}{Ciudad Real}.
\newblock \URLprefix \url{http://hdl.handle.net/10578/31733}.
\bibitem[{Mart{\'{i}}n-Baos et~al.(2020)Mart{\'{i}}n-Baos, Garc{\'{i}}a-R{\'{o}}denas, L{\'{o}}pez-Garc{\'{i}}a and Rodriguez-Benitez}]{MGL20}
\bibinfo{author}{Mart{\'{i}}n-Baos, J.{\'{A}}.}, \bibinfo{author}{Garc{\'{i}}a-R{\'{o}}denas, R.}, \bibinfo{author}{L{\'{o}}pez-Garc{\'{i}}a, M.L.}, \bibinfo{author}{Rodriguez-Benitez, L.}, \bibinfo{year}{2020}.
\newblock \bibinfo{title}{{Discrete choice modeling using Kernel Logistic Regression}}, in: \bibinfo{booktitle}{Transportation Research Procedia}, pp. \bibinfo{pages}{457--464}.
\newblock \DOIprefix\doi{10.1016/j.trpro.2020.03.121}.
\bibitem[{Mart{\'{i}}n-Baos et~al.(2021)Mart{\'{i}}n-Baos, Garc{\'{i}}a-R{\'{o}}denas and Rodriguez-Benitez}]{MGR21}
\bibinfo{author}{Mart{\'{i}}n-Baos, J.{\'{A}}.}, \bibinfo{author}{Garc{\'{i}}a-R{\'{o}}denas, R.}, \bibinfo{author}{Rodriguez-Benitez, L.}, \bibinfo{year}{2021}.
\newblock \bibinfo{title}{{Revisiting kernel logistic regression under the random utility models perspective. An interpretable machine-learning approach}}.
\newblock \bibinfo{journal}{Transportation Letters} \bibinfo{volume}{13}, \bibinfo{pages}{151--162}.
\newblock \DOIprefix\doi{10.1080/19427867.2020.1861504}.
\bibitem[{Mart{\'{i}}n-Baos et~al.(2023)Mart{\'{i}}n-Baos, L{\'{o}}pez-G{\'{o}}mez, Rodriguez-Benitez, Hillel and Garc{\'{i}}a-R{\'{o}}denas}]{MLR23}
\bibinfo{author}{Mart{\'{i}}n-Baos, J.{\'{A}}.}, \bibinfo{author}{L{\'{o}}pez-G{\'{o}}mez, J.A.}, \bibinfo{author}{Rodriguez-Benitez, L.}, \bibinfo{author}{Hillel, T.}, \bibinfo{author}{Garc{\'{i}}a-R{\'{o}}denas, R.}, \bibinfo{year}{2023}.
\newblock \bibinfo{title}{{A prediction and behavioural analysis of machine learning methods for modelling travel mode choice}}.
\newblock \bibinfo{journal}{Transportation Research Part C: Emerging Technologies} \bibinfo{volume}{156}, \bibinfo{pages}{104318}.
\newblock \URLprefix \url{https://www.sciencedirect.com/science/article/pii/S0968090X23003078}, \DOIprefix\doi{https://doi.org/10.1016/j.trc.2023.104318}.
\bibitem[{McFadden(1973)}]{McFa73}
\bibinfo{author}{McFadden, D.}, \bibinfo{year}{1973}.
\newblock \bibinfo{title}{{Conditional Logit Analysis of Qualitative Choice Behaviour}}, in: \bibinfo{editor}{Zarembka, P.} (Ed.), \bibinfo{booktitle}{Frontiers in Econometrics}. \bibinfo{publisher}{Academic Press New York}, \bibinfo{address}{New York, NY, USA}, pp. \bibinfo{pages}{105--142}.
\bibitem[{Musco and Musco(2017)}]{MuM17}
\bibinfo{author}{Musco, C.}, \bibinfo{author}{Musco, C.}, \bibinfo{year}{2017}.
\newblock \bibinfo{title}{{Recursive Sampling for the Nystr{\"{o}}m Method}}, in: \bibinfo{booktitle}{31st Conference on Neural Information Processing Systems (NIPS 2017)}, \bibinfo{address}{Long Beach}.
\bibitem[{Oglic and G{\"{a}}rtner(2017)}]{OgG17}
\bibinfo{author}{Oglic, D.}, \bibinfo{author}{G{\"{a}}rtner, T.}, \bibinfo{year}{2017}.
\newblock \bibinfo{title}{{Nystr{\"{o}}m Method with Kernel K-means++ Samples as Landmarks}}, in: \bibinfo{editor}{Precup, D.}, \bibinfo{editor}{Teh, Y.W.} (Eds.), \bibinfo{booktitle}{Proceedings of the 34th International Conference on Machine Learning}, \bibinfo{publisher}{PMLR}. pp. \bibinfo{pages}{2652--2660}.
\newblock \URLprefix \url{https://proceedings.mlr.press/v70/oglic17a.html}.
\bibitem[{Omrani(2015)}]{Omr15}
\bibinfo{author}{Omrani, H.}, \bibinfo{year}{2015}.
\newblock \bibinfo{title}{{Predicting travel mode of individuals by machine learning}}, in: \bibinfo{booktitle}{Transportation Research Procedia}, \bibinfo{publisher}{Elsevier}. pp. \bibinfo{pages}{840--849}.
\bibitem[{Ouyed and Allili(2018)}]{OuA18}
\bibinfo{author}{Ouyed, O.}, \bibinfo{author}{Allili, M.S.}, \bibinfo{year}{2018}.
\newblock \bibinfo{title}{{Feature weighting for multinomial kernel logistic regression and application to action recognition}}.
\newblock \bibinfo{journal}{Neurocomputing} \bibinfo{volume}{275}, \bibinfo{pages}{1752--1768}.
\newblock \URLprefix \url{https://doi.org/10.1016/j.neucom.2017.10.024}, \DOIprefix\doi{10.1016/j.neucom.2017.10.024}.
\bibitem[{Rasmussen and Williams(2005)}]{RaW05}
\bibinfo{author}{Rasmussen, C.E.}, \bibinfo{author}{Williams, C.K.I.}, \bibinfo{year}{2005}.
\newblock \bibinfo{title}{{Gaussian Processes for Machine Learning (Adaptive Computation and Machine Learning)}}.
\newblock \bibinfo{publisher}{The MIT Press}.
\newblock \URLprefix \url{www.GaussianProcess.org/gpml}.
\bibitem[{Salas et~al.(2022)Salas, De~la Fuente, Astroza and Carrasco}]{SDA22}
\bibinfo{author}{Salas, P.}, \bibinfo{author}{De~la Fuente, R.}, \bibinfo{author}{Astroza, S.}, \bibinfo{author}{Carrasco, J.A.}, \bibinfo{year}{2022}.
\newblock \bibinfo{title}{{A systematic comparative evaluation of machine learning classifiers and discrete choice models for travel mode choice in the presence of response heterogeneity}}.
\newblock \bibinfo{journal}{Expert Systems with Applications} \bibinfo{volume}{193}, \bibinfo{pages}{116253}.
\newblock \DOIprefix\doi{10.1016/J.ESWA.2021.116253}.
\bibitem[{Sculley(2010)}]{Scu10}
\bibinfo{author}{Sculley, D.}, \bibinfo{year}{2010}.
\newblock \bibinfo{title}{{Web-scale k-means clustering}}, in: \bibinfo{booktitle}{Proceedings of the 19th International Conference on World Wide Web, WWW '10}, pp. \bibinfo{pages}{1177--1178}.
\newblock \DOIprefix\doi{10.1145/1772690.1772862}.
\bibitem[{Sekhar et~al.(2016)Sekhar, {Minal} and Madhu}]{SMM16}
\bibinfo{author}{Sekhar, C.}, \bibinfo{author}{{Minal}}, \bibinfo{author}{Madhu, E.}, \bibinfo{year}{2016}.
\newblock \bibinfo{title}{{Mode Choice Analysis Using Random Forrest Decision Trees}}, in: \bibinfo{booktitle}{Transportation Research Procedia}, pp. \bibinfo{pages}{644--652}.
\newblock \DOIprefix\doi{10.1016/j.trpro.2016.11.119}.
\bibitem[{Sun et~al.(2015)Sun, Zhao and Zhu}]{SZZ15}
\bibinfo{author}{Sun, S.}, \bibinfo{author}{Zhao, J.}, \bibinfo{author}{Zhu, J.}, \bibinfo{year}{2015}.
\newblock \bibinfo{title}{{A review of Nystr{\"{o}}m methods for large-scale machine learning}}.
\newblock \bibinfo{journal}{Information Fusion} \bibinfo{volume}{26}, \bibinfo{pages}{36--48}.
\newblock \URLprefix \url{https://www.sciencedirect.com/science/article/pii/S1566253515000317}, \DOIprefix\doi{https://doi.org/10.1016/j.inffus.2015.03.001}.
\bibitem[{Train(2009)}]{Tra09}
\bibinfo{author}{Train, K.}, \bibinfo{year}{2009}.
\newblock \bibinfo{title}{{Discrete Choice Methods with Simulation}}.
\newblock \bibinfo{publisher}{Cambridge University Press}, \bibinfo{address}{Cambridge}.
\bibitem[{Wang et~al.(2019)Wang, Gittens and Mahoney}]{WGM19}
\bibinfo{author}{Wang, S.}, \bibinfo{author}{Gittens, A.}, \bibinfo{author}{Mahoney, M.W.}, \bibinfo{year}{2019}.
\newblock \bibinfo{title}{{Scalable Kernel K-Means Clustering with Nystrom Approximation: Relative-Error Bounds}}.
\newblock \bibinfo{journal}{Journal of Machine Learning Research} \bibinfo{volume}{20}, \bibinfo{pages}{1--49}.
\newblock \URLprefix \url{http://jmlr.org/papers/v20/17-517.html}.
\bibitem[{Wang et~al.(2021)Wang, Mo, Hess and Zhao}]{WMH21}
\bibinfo{author}{Wang, S.}, \bibinfo{author}{Mo, B.}, \bibinfo{author}{Hess, S.}, \bibinfo{author}{Zhao, J.}, \bibinfo{year}{2021}.
\newblock \bibinfo{title}{{Comparing hundreds of machine learning classifiers and discrete choice models in predicting travel behavior: an empirical benchmark}}.
\newblock \URLprefix \url{https://arxiv.org/abs/2102.01130}, \DOIprefix\doi{10.48550/arxiv.2102.01130}.
\bibitem[{Wang et~al.(2014)Wang, Zhang, Qian and Zhang}]{WZQ14}
\bibinfo{author}{Wang, S.}, \bibinfo{author}{Zhang, C.}, \bibinfo{author}{Qian, H.}, \bibinfo{author}{Zhang, Z.}, \bibinfo{year}{2014}.
\newblock \bibinfo{title}{{Improving the Modified Nystr{\"{o}}m Method Using Spectral Shifting}}, in: \bibinfo{booktitle}{Proceedings of the 20th ACM SIGKDD International Conference on Knowledge Discovery and Data Mining}, \bibinfo{publisher}{Association for Computing Machinery}, \bibinfo{address}{New York, NY, USA}. pp. \bibinfo{pages}{611--620}.
\newblock \URLprefix \url{https://doi.org/10.1145/2623330.2623614}, \DOIprefix\doi{10.1145/2623330.2623614}.
\bibitem[{Zaidi and Webb(2017)}]{ZaW17}
\bibinfo{author}{Zaidi, N.A.}, \bibinfo{author}{Webb, G.I.}, \bibinfo{year}{2017}.
\newblock \bibinfo{title}{{A Fast Trust-Region Newton Method for Softmax Logistic Regression}}, in: \bibinfo{booktitle}{Proceedings of the 2017 SIAM International Conference on Data Mining}. \bibinfo{publisher}{Society for Industrial and Applied Mathematics}, \bibinfo{address}{Philadelphia, PA}, pp. \bibinfo{pages}{705--713}.
\newblock \DOIprefix\doi{10.1137/1.9781611974973.79}.
\bibitem[{Zhang et~al.(2008)Zhang, Tsang and Kwok}]{ZTK08}
\bibinfo{author}{Zhang, K.}, \bibinfo{author}{Tsang, I.W.}, \bibinfo{author}{Kwok, J.T.}, \bibinfo{year}{2008}.
\newblock \bibinfo{title}{{Improved Nystr{\"{o}}m low-rank approximation and error analysis}}.
\newblock \bibinfo{journal}{Proceedings of the 25th International Conference on Machine Learning} , \bibinfo{pages}{1232--1239}\DOIprefix\doi{10.1145/1390156.1390311}.
\bibitem[{Zhang et~al.(2022)Zhang, Shi, Hoi and Xu}]{ZSH22}
\bibinfo{author}{Zhang, Q.}, \bibinfo{author}{Shi, W.}, \bibinfo{author}{Hoi, S.}, \bibinfo{author}{Xu, Z.}, \bibinfo{year}{2022}.
\newblock \bibinfo{title}{{Non-uniform Nystr{\"{o}}m approximation for sparse kernel regression: Theoretical analysis and experimental evaluation}}.
\newblock \bibinfo{journal}{Neurocomputing} \bibinfo{volume}{501}, \bibinfo{pages}{410--419}.
\newblock \DOIprefix\doi{10.1016/j.neucom.2022.05.112}.
\bibitem[{Zhang et~al.(2011)Zhang, Dai and Edu}]{ZDJ11}
\bibinfo{author}{Zhang, Z.}, \bibinfo{author}{Dai, G.}, \bibinfo{author}{Edu, J.B.}, \bibinfo{year}{2011}.
\newblock \bibinfo{title}{{Bayesian Generalized Kernel Mixed Models}}.
\newblock \bibinfo{type}{Technical Report}.
\bibitem[{Zhao et~al.(2018)Zhao, Yan, Yu and Van~Hentenryck}]{ZYY19}
\bibinfo{author}{Zhao, X.}, \bibinfo{author}{Yan, X.}, \bibinfo{author}{Yu, A.}, \bibinfo{author}{Van~Hentenryck, P.}, \bibinfo{year}{2018}.
\newblock \bibinfo{title}{{Modeling Stated Preference for Mobility-on-Demand Transit: A Comparison of Machine Learning and Logit Models}}.
\newblock \URLprefix \url{http://arxiv.org/abs/1811.01315}, \DOIprefix\doi{10.48550/arXiv.1811.01315}.
\bibitem[{Zhao et~al.(2020)Zhao, Yan, Yu and Van~Hentenryck}]{ZYY20}
\bibinfo{author}{Zhao, X.}, \bibinfo{author}{Yan, X.}, \bibinfo{author}{Yu, A.}, \bibinfo{author}{Van~Hentenryck, P.}, \bibinfo{year}{2020}.
\newblock \bibinfo{title}{{Prediction and Behavioral Analysis of Travel Mode Choice: A Comparison of Machine Learning and Logit Models}}.
\newblock \bibinfo{journal}{Travel Behaviour and Society} \bibinfo{volume}{20}, \bibinfo{pages}{22--35}.
\bibitem[{Zhu and Hastie(2005)}]{ZhH05}
\bibinfo{author}{Zhu, J.}, \bibinfo{author}{Hastie, T.}, \bibinfo{year}{2005}.
\newblock \bibinfo{title}{{Kernel logistic regression and the import vector machine}}.
\newblock \bibinfo{journal}{Journal of Computational and Graphical Statistics} \bibinfo{volume}{14}, \bibinfo{pages}{185--205}.
\newblock \DOIprefix\doi{10.1198/106186005X25619}.

\end{thebibliography}






\end{document}